\documentclass{article}

\PassOptionsToPackage{numbers,sort&compress}{natbib}

 \usepackage[main, final]{neurips_2026}
\usepackage[T1]{fontenc}    % use 8-bit T1 fonts
\usepackage{url}            % simple URL typesetting
\usepackage{booktabs}       % professional-quality tables
\usepackage{amsfonts}       % blackboard math symbols
\usepackage{nicefrac}       % compact symbols for 1/2, etc.
\usepackage{microtype}      % microtypography
\usepackage{xcolor}         % colors
\usepackage{amsmath}
\usepackage{pifont}
\usepackage{graphicx}
\usepackage{enumitem}
\usepackage{amsthm}
\usepackage{amsfonts}
\usepackage{algorithm}
\usepackage{algpseudocode}
\newtheorem{theorem}{Theorem}
\newtheorem{proposition}{Proposition}
\newtheorem{corollary}{Corollary}
\newtheorem{remark}{Remark}
\newtheorem{lemma}{Lemma}
\usepackage{newtxtext,newtxmath}
\usepackage{xcolor}
\usepackage[colorlinks]{hyperref}
\usepackage{multirow}
\usepackage{colortbl}
\usepackage[most]{tcolorbox}
\usepackage{makecell}
\usepackage{wrapfig}
\usepackage{subcaption}

\usepackage{tabularx}
\usepackage{array}
\usepackage{fontawesome5}

\definecolor{rowblue}{HTML}{E0F7FA}
\definecolor{oursyellow}{HTML}{FFF8E8}
\definecolor{MyCustomBlue}{HTML}{3389DA}
\definecolor{MyCustomPink}{HTML}{EE4F89}
\definecolor{CustomPink}{HTML}{0021C2}

\hypersetup{
    citecolor=MyCustomBlue,
    linkcolor=MyCustomPink,
    urlcolor=CustomPink
}

\usepackage{tikz}
\newcommand{\circnum}[1]{%
  \tikz[baseline=(char.base)]{
    \node[shape=circle, fill=black, text=white,
          inner sep=0.9pt, font=\scriptsize\sffamily\bfseries] (char) {#1};
  }%
}

\definecolor{venuebg}{HTML}{D4ECF8}
\newtcbox{\venuebox}{on line,
  boxrule=0pt,
  colback=venuebg,
  colframe=gray!15,
  arc=1.5pt,
  left=2pt,right=2pt,top=1pt,bottom=1pt,
  boxsep=0pt}
  
\newcommand{\venuetag}[1]{\hspace{0.0em}{\scriptsize\venuebox{\textcolor{black!80}{#1}}}}

\definecolor{propbg}{RGB}{242,249,255}
\definecolor{propborder}{RGB}{90,150,220}
\newtcolorbox{theorembox}{
  colback=propbg,
  colframe=propborder,
  boxrule=1.2pt,
  arc=2.5mm,
  left=0.5mm,
  right=0.5mm,
  top=0.3mm,
  bottom=0.3mm,
  before skip=5pt,
  after skip=5pt
}

\definecolor{propositionpropbg}{RGB}{244,252,247}
\definecolor{propositionpropborder}{RGB}{72,155,105}
\newtcolorbox{propositionbox}{
  colback=propositionpropbg,
  colframe=propositionpropborder,
  boxrule=1.2pt,
  arc=2.5mm,
  left=0.5mm,
  right=0.5mm,
  top=0.3mm,
  bottom=0.3mm,
  before skip=5pt,
  after skip=5pt
}

\title{\raisebox{-0.15em}{\includegraphics[height=1em]{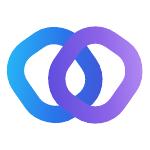}}One Algorithm, Two Goals: Dual Scoring for Parameter and Data Selection in LLM Fine-Tuning}

\author{
Xinrui Chen$^{1}$ \quad Liu Yang$^{2}$ \quad Ou Wu$^{1}$\thanks{Corresponding author.} \\
$^{1}$Hangzhou Institute for Advanced Study, \\University of Chinese Academy of Sciences,
Hangzhou, China\\
$^{2}$School of Artificial Intelligence, Tianjin University, Tianjin, China\\
\texttt{chenxinrui25@mails.ucas.ac.cn} \quad
\texttt{yangliuyl@tju.edu.cn} \quad
\texttt{wuou@ucas.ac.cn}
}

\begin{document}

\maketitle

\begin{abstract}
In Large Language Model (LLM) fine-tuning, parameter and data selection are common strategies for reducing fine-tuning cost, yet they are typically driven by separate scoring mechanisms. When a parameter mask and data subset jointly determine restricted fine-tuning, this separation incurs redundant overhead and makes coordinated selection difficult. We cast parameter and data selection as two bilevel selection problems under a common validation objective and derive a shared local response-surrogate scoring rule. Under first- and second-order validation-improvement approximations, parameter importance and data utility emerge as column-wise and row-wise aggregations of a single gradient interaction matrix, yielding a closed-form row--column correspondence for co-extracting both signals. Building on this structure, we propose \textbf{DualSFT} (\textbf{Dual}-\textbf{S}election \textbf{F}ine-\textbf{T}uning), a one-shot dual-scoring algorithm that produces a parameter mask and data subset from shared gradient statistics. On 3B--9B LLMs, single-axis DualSFT variants strengthen target-task performance and stability--plasticity trade-offs within their comparison groups, while full DualSFT yields a more favorable joint-constrained trade-off than sequential hybrid baselines under matched budgets. \href{https://github.com/ChnIRuI/DualSFT}{\faGithub}
\end{abstract}

\section{Introduction}
Full fine-tuning adapts Large Language Models (LLMs) to downstream tasks~\citep{roziere2023code, azerbayev2024llemma}, but is computationally prohibitive at modern model and data scales~\citep{hu2022lora, dettmers2023qlora, xu2026parameter}. This cost has motivated two complementary lines of research, targeting the parameters being updated and the data used for training. Parameter-efficient fine-tuning (PEFT) controls adaptation costs through approaches such as additive modules, low-rank updates, and parameter selection~\citep{yang2023parameter, hu2022lora, zhao2024galore, liu2024dora, lingam2024svft, xu2026parameter}. Data-efficient fine-tuning (DEFT) optimizes data use through data selection, weighting, and scheduling~\citep{xia2024less, sanyal2025upweighting, xu2026skrull, jiang2026difficulty, agarwal2025delift}.

Within PEFT and DEFT, parameter selection~\citep{panigrahi2023task, song2024sparse, yang2024s2ft, pan2024lisa, he2025smt, liu2025lift, zhou2025pay, yao2026gast} and data selection~\citep{xia2024less, deb2025fishersft, liang2025boosting, min2026gist, chang2026spice} both rank candidates under fixed budgets, but have largely developed independently. They target complementary resources: data selection reduces optimization workload, while parameter selection reduces trainable-state memory. Recent pipelines combine these methods to save resources~\citep{zhao2024tuning, khaki2025sparselora, luo2025tr, mccoy2025ai}, yet separate scoring stages can duplicate computation and evaluate the two selections using different criteria. The parameter mask and data subset jointly determine each restricted update during fine-tuning, motivating the question: \textit{\textbf{can one algorithm serve both goals through a shared scoring rule?}}

We formulate parameter and data selection as separate bilevel problems with a shared validation objective balancing target-task adaptation with preservation of prior capabilities, and approximate their training responses at a common checkpoint. This reveals a row--column duality in scoring: under first- and second-order approximations of local validation improvement, parameter importance and data utility are column and row aggregations of the same gradient interaction matrix. \textit{\textbf{This shared matrix connects the two scoring problems, allowing both sets of scores to be computed in closed form from shared gradient statistics.}} The scores admit a localized Shapley interpretation, attributing surrogate validation improvement to training examples and parameter-coordinate updates.

Building on this duality, we propose \textbf{DualSFT}, realizing \textit{\textbf{one algorithm for two goals}}: a lightweight warmup establishes a stable scoring checkpoint, one-shot dual scoring reads parameter importance and data utility from a shared projection vector, and restricted fine-tuning updates selected parameters on selected data. By scoring both axes before the final update, DualSFT coordinates them during scoring, avoiding independent scoring rules and repeated conditional passes of sequential hybrids.

DualSFT preserves pre-trained capabilities via corpus-free \textbf{CWSD}~(\textbf{C}onfidence-\textbf{W}eighted \textbf{S}elf-\textbf{D}istillation), mitigating catastrophic forgetting during downstream fine-tuning~\citep{wu2025mitigating, wang2025on, sanyal2025upweighting, huang2025mitigating, harmon2026mapping, lin2026sft, zhang2026nads}. Experiments on 3B to 9B LLMs for code generation~\citep{wei2024magicoder} and mathematical reasoning~\citep{yu2024metamath} show DualSFT reduces fine-tuning cost while improving target-task performance. Furthermore, DualSFT strengthens stability-plasticity trade-offs: one-sided variants achieve best mean trade-offs within their groups, and full DualSFT yields best mean joint-constrained trade-off against sequential pipelines under matched budgets. Ablations and sensitivity analyses support its scoring and preservation components. 

Our main contributions are summarized as follows:
\begin{itemize}[leftmargin=*]
    \item We derive a shared response-surrogate scoring rule for parameter and data selection from separate bilevel formulations with a common validation objective. Within this tractable scoring regime, parameter importance and data utility are the column-wise and row-wise aggregations of one gradient interaction matrix, with a localized Shapley interpretation of the resulting closed-form scores.
    \item We propose DualSFT, a one-shot dual scoring algorithm that co-extracts both signals from shared gradient statistics, reducing redundant scoring and repeated conditional passes in sequential combinations while coordinating the two restricted axes at scoring time.
    \item We construct a practical preservation direction using CWSD without pre-training data.
    \item On 3B--9B LLMs, DualSFT reduces tunable parameters and training data, improving target performance. Single-axis variants strengthen intra-group performance and stability-plasticity trade-offs, while full DualSFT outperforms sequential hybrid baselines in joint-constrained trade-offs.
\end{itemize}

\section{Related Work}

\subsection{Parameter and Data Efficient Fine-tuning}
PEFT efficiently restricts updates via additive modules~\citep{yang2023parameter}, low-rank updates~\citep{hu2022lora, zhao2024galore, liu2024dora}, and parameter selection~\citep{panigrahi2023task, song2024sparse, xu2026parameter}. Recent works combine structured or dynamic sparsity with criteria like Fisher guidance, regularization, and importance tracking~\citep{yang2024s2ft, he2025smt, pan2024lisa, lingam2024svft, liu2025lift, miao2025taso, yao2026gast, wang2025floe, zhou2025pay, huang2025mitigating}. Although conceptually related, game-theoretic evaluations~\citep{chu2025model} and Shapley-based methods primarily target post-training compression, including rank allocation~\citep{zhao2026shaplora} and layer pruning~\citep{ding2026pruning}.

DEFT improves efficiency via reweighting~\citep{sanyal2025upweighting, lin2026sft}, scheduling~\citep{jiang2026difficulty, xu2026skrull}, token filtering~\citep{pang2025token, qin2026sstoken}, and sample selection~\citep{xia2024less, deb2025fishersft, chang2026spice, jain2026train}, reallocating effort, reordering signals, dropping low-value tokens, and choosing subsets. Sample selection evolved from static gradient/subspace matching~\citep{xia2024less, min2026gist} to trajectory-aware, validation-guided scoring~\citep{yu2024mates, liang2025boosting, liu2026learn, shen2025seal, jain2026train}, refined by Shapley attribution~\citep{wang2025data}, submodular conflict resolution~\citep{chang2026spice}, and quality-diversity tradeoffs~\citep{yan2025coido, ling2025diversity, das2026matched}. Influence-based attribution links examples to target losses~\citep{koh2017understanding} and informs layer-wise resource allocation~\citep{askari2025layerif}. DualSFT derives example and parameter-coordinate scores from a shared local update surrogate.

Recent cost-efficient fine-tuning methods increasingly couple efficiency and retention. They restrict adaptation via low-rank or subspace adapters~\citep{biderman2024lora,yang2024corda,luo2025sclora}, partial updates, optimizer filtering~\citep{zhou2025pay,chen2025mofo,wang2026loki}, and data/token selection~\citep{wu2025mitigating,liu2026learn}. While preserving pre-trained capabilities downstream, they mostly operate on a single axis, lacking a shared scoring rule for parameter and data selection.

Recent cross-axis studies show that jointly considering parameter- and data-side resources can further reduce fine-tuning cost~\citep{mccoy2025ai,luo2025tr,khaki2025sparselora,zhao2024tuning}. Yet they lack explicitly coordinated data-subset and parameter-mask scoring under a shared validation objective. DualSFT derives both scores from a shared validation-improvement surrogate, yielding a row--column scoring rule for coordination.

\subsection{Bilevel Optimization}
Bilevel optimization (BLO) is widely used for meta-learning, hyperparameter optimization, and architecture search~\citep{finn2017model, franceschi2018bilevel, liu2018darts}. It supports validation-guided data selection before and after LLMs~\citep{killamsetty2021glister, shen2025seal, yu2025llm}, and efficient fine-tuning and compression~\citep{liao2025compress, shirkavand2025bilevel, jiang2025beyond}. Unlike prior BLO optimizing data, prompts, parameters, or compression targets separately, DualSFT links separate validation-guided bilevel formulations through a shared local surrogate, enabling coordinated data--parameter selection.

\section{Bilevel Formulations and Surrogates for Parameter and Data Selection}
\label{sec:blo_framework}

\begin{figure}[t]
    \centering
    \includegraphics[width=\linewidth]{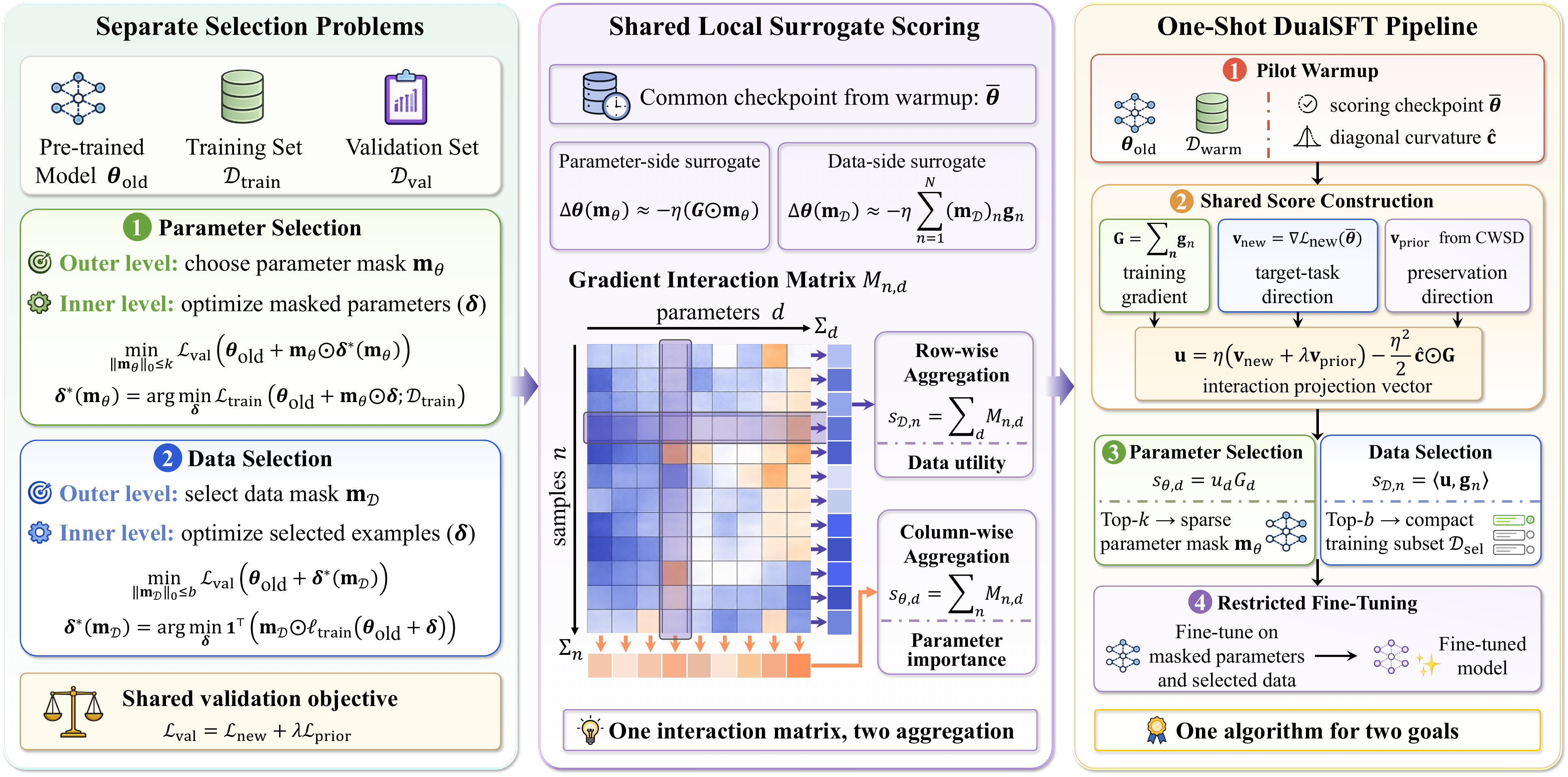}
    \vspace{-0.2in}
    \caption{DualSFT overview. Separate parameter and data bilevel problems share a local surrogate: scores aggregate columns and rows of the same conceptual gradient interaction matrix. DualSFT uses one-shot streaming dual scoring with CWSD before restricted fine-tuning.}
    \label{fig:sparsesft_framework}
    \vspace{-0.1in}
\end{figure}

Figure~\ref{fig:sparsesft_framework} summarizes our framework for parameter and data selection. We formulate them as separate budget-constrained bilevel problems sharing an outer validation objective that balances target-task adaptation and prior-capability preservation. Approximating their training responses at a common checkpoint yields tractable local surrogates, which underpin the shared scoring rule in Section~\ref{sec:shared_scores}.

\subsection{Problem Setting}
Consider a pre-trained LLM with $\boldsymbol{\theta}_{\text{old}} \in \mathbb{R}^D$. Let $\boldsymbol{\delta}\in\mathbb{R}^D$ denote the update applied to $\boldsymbol{\theta}_{\text{old}}$, and let $\boldsymbol{\theta}=\boldsymbol{\theta}_{\text{old}}+\boldsymbol{\delta}$ be the adapted parameter vector. Let $\mathcal{D}_{\text{train}}=\{(x_n,y_n)\}_{n=1}^N$ be the training set and $\mathcal{D}_{\text{val}}$ a validation set. We index updatable coordinates by a binary mask $\mathbf{m}_{\theta}\in\{0,1\}^D$ with $\|\mathbf{m}_{\theta}\|_0\le k$, so that $\delta_d=0$ whenever $(\mathbf{m}_{\theta})_d=0$, and we index training examples by a binary mask $\mathbf{m}_{\mathcal D}\in\{0,1\}^N$ with $\|\mathbf{m}_{\mathcal D}\|_0\le b$, inducing the selected subset $\mathcal{D}_{\mathrm{sel}}(\mathbf{m}_{\mathcal D})\triangleq\{(x_n,y_n)\in\mathcal{D}_{\text{train}}\mid (\mathbf{m}_{\mathcal D})_n=1\}$. The two masks act on different objects but share the same outer criterion. To balance target-task adaptation and prior-capability preservation, we define the validation objective as $\mathcal{L}_{\text{val}}(\boldsymbol{\theta})\triangleq \mathcal{L}_{\text{new}}(\boldsymbol{\theta})+\lambda \mathcal{L}_{\text{prior}}(\boldsymbol{\theta})$, where $\mathcal{L}_{\text{new}}$ is the target-task validation loss, $\mathcal{L}_{\text{prior}}$ the prior-capability preservation loss, and $\lambda>0$ controls the stability--plasticity trade-off.

\subsection{Bilevel Optimization Problems for Parameter and Data Selection}

Under the shared outer criterion $\mathcal{L}_{\text{val}}$, parameter selection and data selection can be formulated as two exact bilevel optimization problems. In both cases, the outer level selects a budget-constrained mask, while the inner level determines the training-induced parameter update under the corresponding restriction. Accordingly, the quality of a selection is determined by the induced update and its resulting validation performance. Here, $\ell(x_n,y_n;\boldsymbol{\theta})$ denotes the per-example training loss, $\boldsymbol{\ell}_{\text{train}}(\boldsymbol{\theta})\triangleq[\ell(x_1,y_1;\boldsymbol{\theta}),\dots,\ell(x_N,y_N;\boldsymbol{\theta})]^\top\in\mathbb{R}^N$ is the vector of per-sample training losses, $\mathbf{1}\in\mathbb{R}^N$ is the all-ones vector, and $\mathcal{L}_{\text{train}}(\boldsymbol{\theta};\mathcal{D}_{\text{train}})\triangleq \mathbf{1}^\top\boldsymbol{\ell}_{\text{train}}(\boldsymbol{\theta})$ is the full training objective.

For parameter selection during LLM fine-tuning, the outer mask determines which parameters are allowed to update, and the inner problem computes the optimal update under this restriction:
\begin{equation}
\min_{\|\mathbf{m}_{\theta}\|_0\le k}\ \mathcal{L}_{\text{val}}\!\left(\boldsymbol{\theta}_{\text{old}}+\mathbf{m}_{\theta}\odot\boldsymbol{\delta}^*(\mathbf{m}_{\theta})\right)
\quad
\text{s.t.}
\quad
\boldsymbol{\delta}^*(\mathbf{m}_{\theta})=\arg\min_{\boldsymbol{\delta}}\ \mathcal{L}_{\text{train}}\!\left(\boldsymbol{\theta}_{\text{old}}+\mathbf{m}_{\theta}\odot\boldsymbol{\delta};\mathcal{D}_{\text{train}}\right).
\label{eq:blo_param_new}
\end{equation}
For data selection during LLM fine-tuning, the outer mask determines which training examples are used for learning, and the inner problem computes the update induced by the selected subset:
\begin{equation}
\min_{\|\mathbf{m}_{\mathcal D}\|_0\le b}\ \mathcal{L}_{\text{val}}\!\left(\boldsymbol{\theta}_{\text{old}}+\boldsymbol{\delta}^*(\mathbf{m}_{\mathcal D})\right)
\quad
\text{s.t.}
\quad
\boldsymbol{\delta}^*(\mathbf{m}_{\mathcal D})=\arg\min_{\boldsymbol{\delta}}\ \mathbf{1}^\top\!\Big(\mathbf{m}_{\mathcal D}\odot \boldsymbol{\ell}_{\text{train}}(\boldsymbol{\theta}_{\text{old}}+\boldsymbol{\delta})\Big).
\label{eq:blo_data_new}
\end{equation}
Here, $\boldsymbol{\delta}^*(\mathbf{m}_{\theta})$ and $\boldsymbol{\delta}^*(\mathbf{m}_{\mathcal D})$ denote the optimal updates induced by the corresponding masks.

\subsection{Local Surrogate Approximation of the Exact Bilevel Problems}
Solving Eqs.~\eqref{eq:blo_param_new} and \eqref{eq:blo_data_new} is intractable at LLM scale, as each candidate mask requires constrained fine-tuning. Following In-Run Data Shapley~\citep{wang2025data} and scalable proxy-based scoring methods~\citep{wang2024greats,shen2025seal,jain2026train,chen2026influencepreserving,nikdan2025efficient,liang2025boosting}, we approximate both updates via one gradient step of the inner objective at a common checkpoint $\bar{\boldsymbol{\theta}}\in\mathbb{R}^D$, yielding local response surrogates sharing this reference point.

Let $\mathbf{g}_n\triangleq\nabla_{\boldsymbol{\theta}} \ell(x_n,y_n;\bar{\boldsymbol{\theta}})$ be the $n$-th example's gradient at $\bar{\boldsymbol{\theta}}$ and $\mathbf{G}\triangleq\sum_{n=1}^N\mathbf{g}_n$ the aggregated training gradient. The masked parameter-side step then descends along $\mathbf{G}\odot \mathbf{m}_{\theta}$, while the data-side step descends along $\sum_{n=1}^{N}(\mathbf{m}_{\mathcal D})_n\mathbf{g}_{n}$, giving the one-step inner-level updates:
\begin{equation}
\Delta\boldsymbol{\theta}(\mathbf{m}_{\theta})\approx -\eta(\mathbf{G}\odot\mathbf{m}_{\theta}),
\qquad
\Delta\boldsymbol{\theta}(\mathbf{m}_{\mathcal D})\approx -\eta\sum\nolimits_{n=1}^N (\mathbf{m}_{\mathcal D})_n\mathbf{g}_n,
\label{eq:single_step_new}
\end{equation}
with $\eta$ the effective local step size. Appendix~\ref{app:one_step_derivation} gives the derivation. The parameter-side surrogate masks the aggregate gradient, whereas the data-side surrogate sums selected-example gradients.

For notational convenience, let $[D]\triangleq\{1,\dots,D\}$, $[N]\triangleq\{1,\dots,N\}$, and identify each mask with its selected index set $S_{\theta}\triangleq\{d\in[D]\mid (\mathbf{m}_{\theta})_d=1\}$, $S_{\mathcal D}\triangleq\{n\in[N]\mid (\mathbf{m}_{\mathcal D})_n=1\}$. Substituting Eq.~\eqref{eq:single_step_new} into the validation gain at $\bar{\boldsymbol{\theta}}$ yields the localized utilities:
\begin{equation}
U_{\theta}(S_{\theta})\triangleq \mathcal{L}_{\text{val}}(\bar{\boldsymbol{\theta}})-\mathcal{L}_{\text{val}}\!\left(\bar{\boldsymbol{\theta}}+\Delta\boldsymbol{\theta}_{\theta}(S_{\theta})\right),
\qquad
U_{\mathcal D}(S_{\mathcal D})\triangleq \mathcal{L}_{\text{val}}(\bar{\boldsymbol{\theta}})-\mathcal{L}_{\text{val}}\!\left(\bar{\boldsymbol{\theta}}+\Delta\boldsymbol{\theta}_{\mathcal D}(S_{\mathcal D})\right),
\label{eq:surrogate_set_functions}
\end{equation}
where $G_d$ denotes the $d$-th entry of $\mathbf{G}$, $\mathbf{e}_d\in\mathbb{R}^D$ the $d$-th standard basis vector, $\Delta\boldsymbol{\theta}_{\theta}(S_{\theta})\triangleq-\eta\sum_{d\in S_{\theta}}G_d\mathbf{e}_d$, and $\Delta\boldsymbol{\theta}_{\mathcal D}(S_{\mathcal D})\triangleq-\eta\sum_{n\in S_{\mathcal D}}\mathbf{g}_n$. These utilities measure validation-objective reduction at $\bar{\boldsymbol{\theta}}$ from local updates, providing a tractable basis for shared scoring analysis.

\section{Shared Interaction Matrix and Dual Scores}
\label{sec:shared_scores}
Parameter and data scores arise as column and row aggregations of a shared interaction matrix built from a local validation-improvement expansion around $\bar{\boldsymbol{\theta}}$. The score identities and Shapley decompositions are exact within specified surrogate games; approximation and selection-stability bounds hold under stated assumptions. These scores guide one-shot selection for multi-step AdamW fine-tuning.

\subsection{Shared Taylor Approximation, Interaction Matrix, and Closed-Form Scores}
Let $\mathbf{v}_{\text{new}}\triangleq \nabla_{\boldsymbol{\theta}} \mathcal{L}_{\text{new}}(\bar{\boldsymbol{\theta}})$, $\mathbf{v}_{\text{prior}}\triangleq \nabla_{\boldsymbol{\theta}} \mathcal{L}_{\text{prior}}(\bar{\boldsymbol{\theta}})$, $\mathbf{v}_{\text{val}}\triangleq \nabla_{\boldsymbol{\theta}} \mathcal{L}_{\text{val}}(\bar{\boldsymbol{\theta}})=\mathbf{v}_{\text{new}}+\lambda \mathbf{v}_{\text{prior}}$, and $\mathbf{H}\triangleq \nabla_{\boldsymbol{\theta}}^2 \mathcal{L}_{\text{val}}(\bar{\boldsymbol{\theta}})$. Although the localized utilities in Eq.~\eqref{eq:surrogate_set_functions} act on different selection spaces, the surrogate updates in Eq.~\eqref{eq:single_step_new} place them around the same checkpoint $\bar{\boldsymbol{\theta}}$ under one outer objective, so both admit a common local second-order expansion. For either $a\in\{\theta,\mathcal D\}$ and its selected set $S_a$:
\begin{equation}
U_a(S_a)
=
-\mathbf{v}_{\text{val}}^\top \Delta\boldsymbol{\theta}_a(S_a)
-\frac{1}{2}\Delta\boldsymbol{\theta}_a(S_a)^\top \mathbf{H}\Delta\boldsymbol{\theta}_a(S_a)
+ R_a(S_a),
\label{eq:taylor_local_utility}
\end{equation}
where $R_a(S_a)=\mathcal{O}(\|\Delta\boldsymbol{\theta}_a(S_a)\|^3)$ is the Taylor remainder. We use the first- or second-order truncation of Eq.~\eqref{eq:taylor_local_utility}; Appendix~\ref{app:local_error} shows utility errors are $O(\eta^2)$ and $O(\eta^3)$, respectively.

Following tractable diagonal curvature approximations for scalable and stable LLM fine-tuning~\citep{zhao2025secondorder,zhao2025pazo,sliwa2025mitigating}, we adopt the diagonal approximation $\mathbf{H}\approx \operatorname{Diag}(c_1,\dots,c_D)$ for this instantiation. With $G_d$ and $g_{n,d}$ denoting the $d$-th entries of $\mathbf{G}$ and $\mathbf{g}_n$, we define three interaction matrices:
\begin{equation}
M_{n,d}^{(1)}\triangleq \eta (\mathbf{v}_{\text{val}})_d g_{n,d},\quad
M_{n,d}^{(\mathrm{diag})}\triangleq M_{n,d}^{(1)}-\frac{\eta^2}{2}c_d G_d g_{n,d},\quad
M_{n,d}^{(\mathrm{full})}\triangleq M_{n,d}^{(1)}-\frac{\eta^2}{2}(\mathbf{H}\mathbf{G})_d g_{n,d},
\label{eq:interaction_matrices}
\end{equation}
and let $\mathbf{M}^{(\alpha)}\in\mathbb{R}^{N\times D}$ collect the entries $M_{n,d}^{(\alpha)}$. For $\alpha\in\{1,\mathrm{diag},\mathrm{full}\}$, the corresponding scores are
\begin{equation}
s_{\theta,d}^{(\alpha)}\triangleq \sum\nolimits_{n=1}^N M_{n,d}^{(\alpha)},\qquad
s_{\mathcal D,n}^{(\alpha)}\triangleq \sum\nolimits_{d=1}^D M_{n,d}^{(\alpha)}.
\label{eq:closed_form_scores}
\end{equation}
Within the local validation-improvement surrogate, the scoring relation takes an explicit form: parameter importance and data utility are column and row aggregations of one shared interaction matrix.

\subsection{Shapley Interpretation of the Closed-Form Scores}
The closed-form scores admit a Shapley interpretation. For each $\alpha\in\{1,\mathrm{diag},\mathrm{full}\}$, localized utilities $U_{\theta}^{(\alpha)}$ and $U_{\mathcal D}^{(\alpha)}$ from Eq.~\eqref{eq:taylor_local_utility} define cooperative games $\mathcal{G}_{\theta}^{(\alpha)}=([D],U_{\theta}^{(\alpha)})$ and $\mathcal{G}_{\mathcal D}^{(\alpha)}=([N],U_{\mathcal D}^{(\alpha)})$.
\begin{theorembox}
\begin{theorem}
\label{thm:dual_shapley}
For any $\alpha\in\{1,\mathrm{diag},\mathrm{full}\}$, the Shapley values of the two localized surrogate games are given by the closed-form scores in Eq.~\eqref{eq:closed_form_scores}:
\begin{equation}
\phi_{\theta,d}^{(\alpha)} = s_{\theta,d}^{(\alpha)}=\sum\nolimits_{n=1}^N M_{n,d}^{(\alpha)},\qquad
\phi_{\mathcal D,n}^{(\alpha)} = s_{\mathcal D,n}^{(\alpha)}=\sum\nolimits_{d=1}^D M_{n,d}^{(\alpha)}.
\label{eq:dual_shapley}
\end{equation}
\end{theorem}
\end{theorembox}
The Shapley view interprets closed-form scores as attributions: first-order localized utilities decompose additively; in diagonal and full second-order cases, permutation averaging redistributes pairwise interactions into row and column aggregations of Eq.~\eqref{eq:closed_form_scores}. Within surrogate games, the interaction matrix yields practical scores and Shapley-consistent attributions. Appendix~\ref{app:shapley} details derivations.

\subsection{Coordination Gap under Isolated Scoring}
\label{sec:coord_gap}
The shared interaction matrix exposes the cost of isolated scoring under joint restricted updates. Let $\mathbf{G}=\sum_{n=1}^N\mathbf{g}_n$ and $\mathbf{G}_{S_\mathcal{D}}=\sum_{n\in S_\mathcal{D}}\mathbf{g}_n$. Under both restrictions, $\Delta\boldsymbol{\theta}(S_\mathcal{D},\mathbf{m}_\theta)=-\eta(\mathbf{m}_\theta\odot \mathbf{G}_{S_\mathcal{D}})$, inducing mask-aware first-order scores $\bar s_{\mathcal{D},n}(\mathbf{m}_\theta)=\eta\langle\mathbf{v}_{\text{val}}\odot \mathbf{m}_\theta,\mathbf{g}_n\rangle$ and $\bar s_{\theta,d}(S_\mathcal{D})=\eta(\mathbf{v}_{\text{val}})_d(\mathbf{G}_{S_\mathcal{D}})_d$.
\begin{propositionbox}
\begin{proposition}
\label{prop:coord_gap}
Let
\(s^{\mathrm{iso}}_{\mathcal{D},n}
=\eta\langle\mathbf{v}_{\text{val}},\mathbf{g}_n\rangle\)
and
\(s^{\mathrm{iso}}_{\theta,d}
=\eta(\mathbf{v}_{\text{val}})_dG_d\)
be the isolated first-order data and parameter scores. Let
\(\bar S_\theta=\{d\in[D]:(\mathbf{m}_\theta)_d=0\}\) and
\(\bar S_\mathcal{D}=[N]\setminus S_\mathcal{D}\). Then:
\begin{equation}
\begin{aligned}
s^{\mathrm{iso}}_{\mathcal{D},n}
-\bar s_{\mathcal{D},n}(\mathbf{m}_\theta)
&=
\eta{\textstyle\sum\nolimits_{d\in\bar S_\theta}}
(\mathbf{v}_{\text{val}})_d g_{n,d},
\quad
s^{\mathrm{iso}}_{\theta,d}
-\bar s_{\theta,d}(S_\mathcal{D})
=
\eta(\mathbf{v}_{\text{val}})_d
{\textstyle\sum\nolimits_{n\in\bar S_\mathcal{D}}} g_{n,d}.
\end{aligned}
\label{eq:coord_gap}
\end{equation}
\end{proposition}
\end{propositionbox}
Eq.~\eqref{eq:coord_gap} identifies contributions removed by the opposite restriction: frozen coordinates contribute to data scores, and excluded samples contribute to parameter scores. DualSFT coordinates the scoring criterion by extracting both marginal scores from shared gradient statistics under the same validation objective and checkpoint, followed by separate Top-$b$/Top-$k$ selection. Dependencies on the final opposite mask remain unmodeled in these one-shot scores. Section~\ref{sec:diagnostic_analysis} evaluates the resulting utility loss and the cost--benefit of conditional re-selection. Appendix~\ref{app:coord_gap} proves Proposition~\ref{prop:coord_gap}.

\section{The DualSFT Algorithm}
\label{sec:algorithm}

\subsection{Pilot Warmup and Global Dual Scoring}
\label{subsec:surrogates}

DualSFT instantiates the local scoring framework in Sections~\ref{sec:blo_framework} and~\ref{sec:shared_scores} using warmup and shared gradient statistics. Warmup on $\mathcal{D}_{\text{warm}}$ produces $\bar{\boldsymbol{\theta}}$ to stabilize scoring and reduce gradient noise~\citep{xia2024less,min2026gist,liang2025boosting}. It provides a diagonal curvature proxy~\citep{zhang2025adammini,kalra2025when,robert2025ldadam,vyas2025soap}: after $t_w$ AdamW updates, we bias-correct the second-moment buffer: $\hat{\mathbf{r}}_{t_w}=\mathbf{r}_{t_w}/(1-\beta_2^{t_w})$ and set $\hat{\mathbf{c}}=\sqrt{\hat{\mathbf{r}}_{t_w}}+\epsilon_{\mathrm{adam}}$ elementwise. This RMS-gradient scale is fixed during scoring; Appendix~\ref{app:adam_compatibility} explains its surrogate role.

At $\bar{\boldsymbol{\theta}}$, we accumulate $\mathbf{G}=\sum_{n=1}^{N_{\mathrm{sc}}}\mathbf{g}_n$ over $\mathcal{D}_{\mathrm{pool}}=\mathcal{D}_{\text{train}}\setminus\mathcal{D}_{\text{warm}}$, where $N_{\mathrm{sc}}=|\mathcal{D}_{\mathrm{pool}}|$. Setting $\eta_{\mathrm{sc}}=\eta_{\mathrm{ft}}/N_{\mathrm{sc}}$ makes $\eta_{\mathrm{sc}}\mathbf{G}$ a mean-gradient step with learning-rate scale $\eta_{\mathrm{ft}}$. Combining the target-validation gradient $\mathbf{v}_{\text{new}}$ with the CWSD preservation gradient $\mathbf{v}_{\text{prior}}$ gives the shared scoring projection:
\begin{equation}
\mathbf{u}\triangleq \eta_{\mathrm{sc}}\mathbf{v}_{\text{val}}-\frac{\eta_{\mathrm{sc}}^2}{2}\hat{\mathbf{c}}\odot\mathbf{G}
=\eta_{\mathrm{sc}}(\mathbf{v}_{\text{new}}+\lambda\mathbf{v}_{\text{prior}})-\frac{\eta_{\mathrm{sc}}^2}{2}\hat{\mathbf{c}}\odot\mathbf{G}.
\label{eq:u_vector}
\end{equation}
This projection combines the first-order validation gradient and optimizer-aware diagonal correction. Aggregating interaction scores column-wise and row-wise yields parameter and data scores:
\begin{equation}
M^{(\mathrm{diag})}_{n,d}=u_dg_{n,d},\qquad
s_{\theta,d}=\textstyle\sum_{n=1}^{N_{\mathrm{sc}}}M^{(\mathrm{diag})}_{n,d}=u_dG_d,\qquad
s_{\mathcal D,n}=\textstyle\sum_{d=1}^{D}M^{(\mathrm{diag})}_{n,d}=\langle\mathbf{u},\mathbf{g}_n\rangle.
\label{eq:practical_scores}
\end{equation}

We select the largest signed scores using Top-$k$ for parameters and Top-$b$ for data. Appendix~\ref{app:error_propagation} bounds deviations from the local full-Hessian scores in terms of $\eta_{\mathrm{sc}}^2$ and curvature-proxy error, and establishes selection stability when boundary gaps exceed the perturbation bounds. After accumulating $\mathbf{G}$ and constructing $\mathbf{u}$, a second scoring-pool pass computes data scores with Ghost Dot Product~\citep{wang2024greats,wang2025data}, avoiding per-sample-gradient materialization. For an affine weight matrix, its contribution is $\sum_t\boldsymbol{\epsilon}_{n,t}^{\top}\mathbf{U}_{\mathrm{layer}}\mathbf{a}_{n,t}$, where $t$ indexes token positions, $\mathbf{a}_{n,t}$ is the input activation, $\boldsymbol{\epsilon}_{n,t}$ is the pre-activation gradient of the example loss, and $\mathbf{U}_{\mathrm{layer}}$ reshapes the corresponding segment of $\mathbf{u}$.

CWSD supplies $\mathbf{v}_{\text{prior}}$ without pretraining data, using the frozen pre-trained model as teacher and the warmup model as student to measure predictive changes introduced during pilot warmup. Let $\mathcal{X}_{\text{anchor}}$ contain inputs from a held-out anchor set with $\mathcal{D}_{\text{anchor}}\cap\mathcal{D}_{\text{train}}=\emptyset$. We compute:
\begin{equation}
\mathbf{v}_{\text{prior}}\approx\nabla_{\bar{\boldsymbol{\theta}}}\frac{1}{|\mathcal{X}_{\text{anchor}}|}\textstyle\sum_{x\in\mathcal{X}_{\text{anchor}}}\omega(x)\tau^2\mathrm{KL}\!\Big(P_\tau(\cdot\mid\boldsymbol{\theta}_{\text{old}},x)\,\big\|\,P_\tau(\cdot\mid\bar{\boldsymbol{\theta}},x)\Big),
\label{eq:cwsd}
\end{equation}
where $\tau$ is the distillation temperature and $\omega(x)=1-H_{\boldsymbol{\theta}_{\text{old}}}(x)/\log|\mathcal{V}|$ emphasizes teacher-confident inputs; $H_{\boldsymbol{\theta}_{\text{old}}}(x)$ is the mean per-token predictive entropy. Teacher predictions and confidence weights are fixed, and gradients flow through the student. Evaluating at the warmup checkpoint captures departures from the pre-trained teacher. We average token-wise forward KL over non-padding response tokens and then over anchor examples. For each weighted token loss, the student-logit gradient before averaging is $\omega(x)\tau(\mathbf{p}-\mathbf{q})$, with $\mathbf{p}$ and $\mathbf{q}$ denoting student and teacher distributions. Appendix~\ref{app:cwsd_derivation} derives the resulting parameter gradient through the student Jacobian. CWSD is used only for selection; final restricted fine-tuning optimizes the target-task loss.

\subsection{Restricted Fine-Tuning and Complexity}
\label{subsec:unified_algorithm}
In Algorithm~\ref{alg:dualsft}, pilot warmup yields $\bar{\boldsymbol{\theta}}$, $\hat{\mathbf{c}}$; one-shot dual scoring via $\mathbf{u}$ selects $\mathcal{D}_{\text{sel}}$, $\mathbf{m}_{\theta}$; restricted fine-tuning updates selected parameters on $\mathcal{D}_{\text{sel}}$. Restarting from $\boldsymbol{\theta}_{\text{old}}$ decouples scoring from adaptation, following warmup-then-select-from-scratch pipelines~\citep{xia2024less,min2026gist,liang2025boosting}. Final AdamW fine-tuning, $\textsc{MaskedFT}_{\eta_{\mathrm{ft}}}(\boldsymbol{\theta}_{\text{old}},\mathcal{D}_{\text{sel}},\mathbf{m}_{\theta})$, masks gradients, adaptive moments, and decoupled weight decay.

\begin{wrapfigure}{R}{0.5\columnwidth}
\vspace{-0.35in}
\begin{minipage}{\linewidth}
\begin{algorithm}[H]
\caption{DualSFT}
\label{alg:dualsft}
\begingroup
\small
\linespread{1}\selectfont
\begin{algorithmic}[1]
\Require $\boldsymbol{\theta}_{\text{old}}, \mathcal{D}_{\text{train}}, \mathcal{D}_{\text{val}}, \mathcal{D}_{\text{warm}}, \mathcal{D}_{\text{anchor}}, \eta_{\mathrm{ft}}, \lambda, \tau, b, k$
\State $(\bar{\boldsymbol{\theta}}, \hat{\mathbf{c}}) \leftarrow \textsc{Warmup}(\boldsymbol{\theta}_{\text{old}}, \mathcal{D}_{\text{warm}})$
\State $\mathcal{D}_{\text{pool}}\leftarrow \mathcal{D}_{\text{train}}\setminus \mathcal{D}_{\text{warm}}$
\State $\mathbf{G} \leftarrow \textsc{GradSum}(\bar{\boldsymbol{\theta}},\mathcal{D}_{\text{pool}})$
\State $N_{\mathrm{sc}}\leftarrow |\mathcal{D}_{\text{pool}}|,\quad \eta_{\mathrm{sc}}\leftarrow \eta_{\mathrm{ft}}/N_{\mathrm{sc}}$
\State $\mathbf{v}_{\text{new}} \leftarrow \nabla_{\bar{\boldsymbol{\theta}}}\mathcal{L}_{\text{new}}(\bar{\boldsymbol{\theta}})$
\State $\mathbf{v}_{\text{prior}} \leftarrow \textsc{CWSD}(\bar{\boldsymbol{\theta}},\boldsymbol{\theta}_{\text{old}},\mathcal{D}_{\text{anchor}},\tau)$
\State $\mathbf{u} \leftarrow \eta_{\mathrm{sc}}(\mathbf{v}_{\text{new}}+\lambda\mathbf{v}_{\text{prior}})-(\eta_{\mathrm{sc}}^2/2)\hat{\mathbf{c}}\odot\mathbf{G}$
\State $\mathbf{s}_{\theta}\leftarrow \mathbf{u}\odot\mathbf{G},\quad \mathbf{s}_{\mathcal D}\leftarrow \textsc{GhostDot}(\mathbf{u},\bar{\boldsymbol{\theta}},\mathcal{D}_{\text{pool}})$
\State $\mathbf{m}_{\theta}\leftarrow \textsc{TopK}_{\mathrm{signed}}(\mathbf{s}_{\theta},k)$
\State $\mathbf{m}_{\mathcal D}\leftarrow \textsc{TopB}_{\mathrm{signed}}(\mathbf{s}_{\mathcal D},b)$
\State $\mathcal{D}_{\text{sel}}\leftarrow\{(x_n,y_n)\in\mathcal{D}_{\text{pool}}:(\mathbf{m}_{\mathcal D})_n=1\}$
\State $\boldsymbol{\theta}^{\star}\leftarrow\textsc{MaskedFT}_{\eta_{\mathrm{ft}}}(\boldsymbol{\theta}_{\text{old}},\mathcal{D}_{\text{sel}},\mathbf{m}_{\theta})$
\State \textbf{return} $\mathbf{m}_{\mathcal D},\mathbf{m}_{\theta},\boldsymbol{\theta}^{\star}$
\end{algorithmic}
\endgroup
\end{algorithm}
\end{minipage}
\vspace{-15pt}
\end{wrapfigure}
DualSFT avoids materializing the $\mathcal{O}(N_{\mathrm{sc}}\times D)$ interaction matrix, confining selection to lightweight warmup and two streaming passes. First accumulates $\mathbf{G}$, constructing $\mathbf{u}$; second computes sample scores, maintaining the Top-$b$ subset. Parameter scoring uses full-model vectors ($\mathbf{G}$, $\mathbf{u}$, $\hat{\mathbf{c}}$), incurring $\mathcal{O}(D)$ auxiliary memory, linear in model size, implementable block-wise or via sharding. For data scoring, Ghost Dot Product avoids materializing per-sample gradients $\mathbf{g}_n\in\mathbb{R}^D$. Restricted fine-tuning proceeds on $\mathcal{D}_{\text{sel}}$ masked by $\mathbf{m}_{\theta}$. Thus, DualSFT reduces scoring redundancy and avoids interaction-matrix materialization, with parameter scoring requiring linear auxiliary state.

\section{Experiments}
\label{sec:experiment}

% We evaluate DualSFT from four aspects. \circnum{1} Does DualSFT outperform standard, parameter-only, data-only, and sequential-hybrid baselines under matched budgets? \circnum{2} Which components contribute most, and how sensitive is performance to key hyperparameters? \circnum{3} Do DualSFT's algorithmic claims hold under diagnostic analysis of one-shot sufficiency, selection structure, and score faithfulness? \circnum{4} How does DualSFT compare against recent forgetting-specific baselines?

We evaluate DualSFT through matched-budget performance and efficiency comparisons, component ablations, sensitivity analyses, and diagnostics of score reliability and re-selection costs. Comparisons with forgetting-specific baselines further assess capability preservation.

\subsection{Experimental Setup}

\textbf{Baselines.} We compare against \circnum{1} standard methods (Standard SFT, LoRA~\citep{hu2022lora}); \circnum{2} parameter-only sparse methods (S$^2$FT~\citep{yang2024s2ft}, SMT~\citep{he2025smt}, LIFT~\citep{liu2025lift}, NanoAdam~\citep{zhou2025pay}); \circnum{3} data-only methods (Random, LESS~\citep{xia2024less}, FisherSFT~\citep{deb2025fishersft}, SPICE~\citep{chang2026spice}); \circnum{4} sequential hybrids (LESS + S$^2$FT, SPICE + NanoAdam); and \circnum{5} forgetting-aware baselines (DavIR~\citep{zhou2025davir}, FLOW~\citep{sanyal2025upweighting}, STM~\citep{wu2025mitigating}, FAPM~\citep{huang2025mitigating}, TALR~\citep{lin2026sft}).

\textbf{Datasets \& Metrics.} We evaluate on Magicoder~\citep{wei2024magicoder} and MetaMathQA~\citep{yu2024metamath}, each reserving 1{,}024 validation samples and 2{,}048 disjoint anchor samples for CWSD. Following prior studies~\citep{sanyal2025upweighting, wang2025on, yuan2026differential}, we assess both plasticity and stability. For code, plasticity is measured by HumanEval~\citep{chen2021evaluating} and stability by the four out-of-domain benchmarks PIQA~\citep{bisk2020piqa}, ARC-C~\citep{clark2018think}, MMLU~\citep{hendrycks2021measuring}, and GSM8K~\citep{cobbe2021training}; for math, GSM8K and HumanEval swap roles. We report Plasticity, Stability (four-benchmark mean), and Overall (their arithmetic mean); see Appendix~\ref{app:eval_details} for partitioning.

\textbf{Implementation.} Base models are Llama-3.2-3B~\citep{grattafiori2024llama}, Gemma-3-4B-PT~\citep{gemma2025}, and Qwen-3.5-9B-Base~\citep{qwen35blog}. Training uses LlamaFactory~\citep{zheng2024llamafactory} with AdamW and cosine decay (Magicoder: 3 epochs, batch 64; MetaMathQA: 2 epochs, batch 128). DualSFT uses $\lambda=0.8$, $\tau=1.0$, a 1-epoch warmup with fraction $r_{\mathrm{warm}}=5\%$, and budgets $b=10\%$ for data and $k=5\%$ for parameters. For scoring, $\eta_{\mathrm{sc}}=\eta_{\mathrm{ft}}/N_{\mathrm{sc}}$, with $\eta_{\mathrm{ft}}$ reported in Appendix~\ref{app:impl_details}; budget-constrained baselines use matched budgets within their comparison groups. We report 3-seed means and stds. Appendix~\ref{app:impl_details} details hyperparameters, separating adaptation budgets from scoring-only validation/anchors.

\subsection{Main Results}
\label{sec:main_results}

\begin{table*}[t]
\setlength{\aboverulesep}{1.5pt}
\setlength{\belowrulesep}{1.5pt}
\centering
\caption{Fine-tuning results on Magicoder (\%). Pre-trained, Standard SFT, and LoRA (gray) are excluded from group ranking. In each group, \colorbox{red!10}{\textbf{best}} and \colorbox{orange!15}{second-best} results are highlighted.}
\renewcommand{\arraystretch}{0.92}
\label{tab:main_results_magicoder}
\vspace{-0.1in}
\footnotesize
\setlength{\fboxsep}{0.5pt} % 严格压缩高亮框内边距
\resizebox{\textwidth}{!}{
\begin{tabular}{c | c | ccccc | c | c}
\toprule
\multirow{2}{*}{\textbf{Model}} & \multirow{2}{*}{\textbf{Method}} & \multicolumn{5}{c|}{\textbf{Stability} $\uparrow$} & \textbf{Plasticity} $\uparrow$ & \textbf{Overall} $\uparrow$ \\
\cmidrule(lr){3-7} \cmidrule(lr){8-8} \cmidrule(lr){9-9}
& & \textbf{ARC-C} & \textbf{PIQA} & \textbf{MMLU} & \textbf{GSM8K} & \textbf{Avg} & \textbf{HE} & \textbf{Score} \\
\midrule
\multirow{16}{*}{\rotatebox{90}{\raisebox{-0.15\height}{\includegraphics[height=1em]{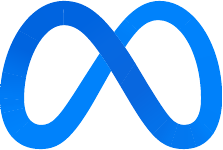}}\hspace{0.3em}\textbf{Llama-3.2-3B\qquad}}}
& \textcolor{gray}{Pre-trained} & \textcolor{gray}{43.00$_{\pm 0.00}$} & \textcolor{gray}{76.61$_{\pm 0.00}$} & \textcolor{gray}{56.52$_{\pm 0.00}$} & \textcolor{gray}{27.45$_{\pm 0.00}$} & \textcolor{gray}{50.89$_{\pm 0.00}$} & \textcolor{gray}{28.66$_{\pm 0.00}$} & \textcolor{gray}{39.78$_{\pm 0.00}$} \\
& \textcolor{gray}{Standard SFT} & \textcolor{gray}{39.99$_{\pm 1.46}$} & \textcolor{gray}{72.09$_{\pm 0.98}$} & \textcolor{gray}{48.18$_{\pm 0.39}$} & \textcolor{gray}{20.70$_{\pm 0.85}$} & \textcolor{gray}{45.24$_{\pm 0.50}$} & \textcolor{gray}{44.72$_{\pm 0.57}$} & \textcolor{gray}{44.98$_{\pm 0.37}$} \\
& \textcolor{gray}{LoRA}\venuetag{ICLR'22} & \textcolor{gray}{42.61$_{\pm 1.45}$} & \textcolor{gray}{74.43$_{\pm 0.99}$} & \textcolor{gray}{51.87$_{\pm 0.41}$} & \textcolor{gray}{23.30$_{\pm 0.83}$} & \textcolor{gray}{48.05$_{\pm 0.49}$} & \textcolor{gray}{41.46$_{\pm 0.50}$} & \textcolor{gray}{44.76$_{\pm 0.37}$} \\
\cmidrule(lr){2-9}
& S$^2$FT\venuetag{NIPS'24} & 41.81$_{\pm 1.45}$ & \colorbox{red!10}{\textbf{76.26$_{\pm 0.98}$}} & 52.10$_{\pm 0.40}$ & 22.57$_{\pm 0.72}$ & 48.18$_{\pm 0.49}$ & 42.68$_{\pm 0.86}$ & 45.43$_{\pm 0.44}$ \\
& SMT\venuetag{ICLR'25} & 40.98$_{\pm 1.45}$ & 74.16$_{\pm 0.98}$ & 52.59$_{\pm 0.40}$ & 23.88$_{\pm 0.75}$ & 47.90$_{\pm 0.48}$ & 43.09$_{\pm 0.57}$ & 45.50$_{\pm 0.39}$ \\
& LIFT\venuetag{ICML'25} & 41.21$_{\pm 1.45}$ & 72.51$_{\pm 0.99}$ & 52.14$_{\pm 0.40}$ & 23.22$_{\pm 0.73}$ & 47.27$_{\pm 0.48}$ & 42.28$_{\pm 0.57}$ & 44.77$_{\pm 0.39}$ \\
& NanoAdam\venuetag{NIPS'25} & \colorbox{red!10}{\textbf{42.66$_{\pm 1.44}$}} & 74.68$_{\pm 0.98}$ & \colorbox{orange!15}{52.71$_{\pm 0.40}$} & \colorbox{orange!15}{24.39$_{\pm 0.73}$} & \colorbox{orange!15}{48.61$_{\pm 0.48}$} & \colorbox{orange!15}{43.50$_{\pm 0.76}$} & \colorbox{orange!15}{46.05$_{\pm 0.50}$} \\
& \textbf{DualSFT-Param (Ours)} & \colorbox{orange!15}{42.55$_{\pm 1.46}$} & \colorbox{orange!15}{76.19$_{\pm 0.99}$} & \colorbox{red!10}{\textbf{54.01$_{\pm 0.41}$}} & \colorbox{red!10}{\textbf{26.28$_{\pm 0.75}$}} & \colorbox{red!10}{\textbf{49.76$_{\pm 0.50}$}} & \colorbox{red!10}{\textbf{44.51$_{\pm 0.86}$}} & \colorbox{red!10}{\textbf{47.13$_{\pm 0.42}$}} \\
\cmidrule(lr){2-9}
& Random & \colorbox{red!10}{\textbf{43.00$_{\pm 1.46}$}} & \colorbox{orange!15}{76.10$_{\pm 0.98}$} & \colorbox{red!10}{\textbf{55.31$_{\pm 0.40}$}} & \colorbox{red!10}{\textbf{26.99$_{\pm 0.78}$}} & \colorbox{red!10}{\textbf{50.35$_{\pm 0.49}$}} & 38.62$_{\pm 0.76}$ & 44.48$_{\pm 0.45}$ \\
& LESS\venuetag{ICML'24} & 42.18$_{\pm 1.45}$ & 75.66$_{\pm 0.99}$ & 53.97$_{\pm 0.40}$ & 23.45$_{\pm 0.63}$ & 48.82$_{\pm 0.50}$ & \colorbox{orange!15}{42.68$_{\pm 0.86}$} & \colorbox{orange!15}{45.75$_{\pm 0.46}$} \\
& FisherSFT\venuetag{ICML'25} & 41.55$_{\pm 1.44}$ & 74.17$_{\pm 0.99}$ & 52.83$_{\pm 0.40}$ & 22.92$_{\pm 0.62}$ & 47.87$_{\pm 0.47}$ & 41.46$_{\pm 0.86}$ & 44.67$_{\pm 0.43}$ \\
& SPICE\venuetag{ICLR'26} & 40.53$_{\pm 1.45}$ & \colorbox{orange!15}{76.10$_{\pm 0.98}$} & 53.13$_{\pm 0.40}$ & 24.51$_{\pm 0.73}$ & 48.57$_{\pm 0.44}$ & 42.28$_{\pm 0.57}$ & 45.42$_{\pm 0.39}$ \\
& \textbf{DualSFT-Data (Ours)} & \colorbox{orange!15}{42.66$_{\pm 1.44}$} & \colorbox{red!10}{\textbf{76.44$_{\pm 0.99}$}} & \colorbox{orange!15}{55.28$_{\pm 0.40}$} & \colorbox{orange!15}{26.54$_{\pm 0.73}$} & \colorbox{orange!15}{50.23$_{\pm 0.50}$} & \colorbox{red!10}{\textbf{44.11$_{\pm 0.76}$}} & \colorbox{red!10}{\textbf{47.17$_{\pm 0.43}$}} \\
\cmidrule(lr){2-9}
& LESS + S$^2$FT & \colorbox{orange!15}{42.58$_{\pm 1.44}$} & \colorbox{red!10}{\textbf{76.68$_{\pm 0.96}$}} & 53.55$_{\pm 0.39}$ & 22.87$_{\pm 0.71}$ & 48.92$_{\pm 0.48}$ & 38.82$_{\pm 0.57}$ & 43.87$_{\pm 0.40}$ \\
& SPICE + NanoAdam & 42.15$_{\pm 1.46}$ & \colorbox{orange!15}{76.61$_{\pm 0.98}$} & \colorbox{orange!15}{54.87$_{\pm 0.40}$} & \colorbox{orange!15}{25.93$_{\pm 0.71}$} & \colorbox{orange!15}{49.89$_{\pm 0.49}$} & \colorbox{orange!15}{39.43$_{\pm 0.76}$} & \colorbox{orange!15}{44.66$_{\pm 0.42}$} \\
& \textbf{DualSFT (Ours)} & \colorbox{red!10}{\textbf{42.95$_{\pm 1.45}$}} & 76.22$_{\pm 0.98}$ & \colorbox{red!10}{\textbf{56.09$_{\pm 0.40}$}} & \colorbox{red!10}{\textbf{26.89$_{\pm 0.70}$}} & \colorbox{red!10}{\textbf{50.54$_{\pm 0.47}$}} & \colorbox{red!10}{\textbf{43.70$_{\pm 0.76}$}} & \colorbox{red!10}{\textbf{47.12$_{\pm 0.43}$}} \\
\midrule
\multirow{16}{*}{\rotatebox{90}{\raisebox{-0.25\height}{\includegraphics[height=1.2em]{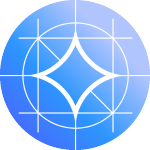}}\hspace{0.3em}\textbf{Gemma-3-4B-PT\qquad}}}
& \textcolor{gray}{Pre-trained} & \textcolor{gray}{51.45$_{\pm 0.00}$} & \textcolor{gray}{79.16$_{\pm 0.00}$} & \textcolor{gray}{59.61$_{\pm 0.00}$} & \textcolor{gray}{37.00$_{\pm 0.00}$} & \textcolor{gray}{56.80$_{\pm 0.00}$} & \textcolor{gray}{35.37$_{\pm 0.00}$} & \textcolor{gray}{46.09$_{\pm 0.00}$} \\
& \textcolor{gray}{Standard SFT} & \textcolor{gray}{47.16$_{\pm 1.46}$} & \textcolor{gray}{76.39$_{\pm 0.99}$} & \textcolor{gray}{53.87$_{\pm 0.40}$} & \textcolor{gray}{24.03$_{\pm 0.81}$} & \textcolor{gray}{50.36$_{\pm 0.50}$} & \textcolor{gray}{58.13$_{\pm 0.76}$} & \textcolor{gray}{54.25$_{\pm 0.51}$} \\
& \textcolor{gray}{LoRA}\venuetag{ICLR'22} & \textcolor{gray}{49.15$_{\pm 1.46}$} & \textcolor{gray}{75.57$_{\pm 1.00}$} & \textcolor{gray}{56.71$_{\pm 0.39}$} & \textcolor{gray}{31.64$_{\pm 0.82}$} & \textcolor{gray}{53.27$_{\pm 0.49}$} & \textcolor{gray}{54.67$_{\pm 1.04}$} & \textcolor{gray}{53.97$_{\pm 0.52}$} \\
\cmidrule(lr){2-9}
& S$^2$FT\venuetag{NIPS'24} & 49.74$_{\pm 1.46}$ & \colorbox{orange!15}{78.29$_{\pm 0.96}$} & 55.75$_{\pm 0.40}$ & 31.08$_{\pm 0.60}$ & 53.72$_{\pm 0.47}$ & 55.28$_{\pm 0.76}$ & 54.50$_{\pm 0.50}$ \\
& SMT\venuetag{ICLR'25} & 49.91$_{\pm 1.46}$ & 77.98$_{\pm 0.93}$ & 56.44$_{\pm 0.39}$ & 32.75$_{\pm 0.84}$ & 54.27$_{\pm 0.51}$ & \colorbox{orange!15}{56.50$_{\pm 0.76}$} & 55.39$_{\pm 0.50}$ \\
& LIFT\venuetag{ICML'25} & 50.03$_{\pm 1.46}$ & 77.17$_{\pm 0.95}$ & 54.56$_{\pm 0.40}$ & 30.60$_{\pm 0.63}$ & 53.09$_{\pm 0.48}$ & 55.69$_{\pm 0.57}$ & 54.39$_{\pm 0.39}$ \\
& NanoAdam\venuetag{NIPS'25} & \colorbox{orange!15}{50.77$_{\pm 1.45}$} & 78.20$_{\pm 0.94}$ & \colorbox{red!10}{\textbf{58.33$_{\pm 0.39}$}} & \colorbox{orange!15}{33.13$_{\pm 0.84}$} & \colorbox{orange!15}{55.11$_{\pm 0.49}$} & 56.30$_{\pm 1.04}$ & \colorbox{orange!15}{55.70$_{\pm 0.51}$} \\
& \textbf{DualSFT-Param (Ours)} & \colorbox{red!10}{\textbf{51.14$_{\pm 1.45}$}} & \colorbox{red!10}{\textbf{78.55$_{\pm 0.98}$}} & \colorbox{orange!15}{58.14$_{\pm 0.39}$} & \colorbox{red!10}{\textbf{36.34$_{\pm 0.83}$}} & \colorbox{red!10}{\textbf{56.04$_{\pm 0.49}$}} & \colorbox{red!10}{\textbf{57.52$_{\pm 0.76}$}} & \colorbox{red!10}{\textbf{56.78$_{\pm 0.48}$}} \\
\cmidrule(lr){2-9}
& Random & \colorbox{red!10}{\textbf{51.25$_{\pm 1.46}$}} & \colorbox{orange!15}{78.89$_{\pm 0.95}$} & \colorbox{orange!15}{56.91$_{\pm 0.40}$} & \colorbox{orange!15}{34.65$_{\pm 0.67}$} & \colorbox{orange!15}{55.42$_{\pm 0.48}$} & 49.19$_{\pm 0.76}$ & 52.31$_{\pm 0.48}$ \\
& LESS\venuetag{ICML'24} & 50.91$_{\pm 1.46}$ & \colorbox{red!10}{\textbf{78.93$_{\pm 0.93}$}} & 54.14$_{\pm 0.39}$ & 31.59$_{\pm 0.84}$ & 53.89$_{\pm 0.51}$ & \colorbox{red!10}{\textbf{57.11$_{\pm 1.04}$}} & \colorbox{orange!15}{55.50$_{\pm 0.53}$} \\
& FisherSFT\venuetag{ICML'25} & 50.26$_{\pm 1.45}$ & 74.50$_{\pm 0.94}$ & 53.71$_{\pm 0.38}$ & 32.17$_{\pm 0.77}$ & 52.66$_{\pm 0.48}$ & 53.46$_{\pm 0.76}$ & 53.06$_{\pm 0.45}$ \\
& SPICE\venuetag{ICLR'26} & 50.48$_{\pm 1.44}$ & 78.16$_{\pm 0.95}$ & 55.12$_{\pm 0.40}$ & 32.63$_{\pm 0.72}$ & 54.10$_{\pm 0.48}$ & \colorbox{orange!15}{56.91$_{\pm 0.76}$} & \colorbox{orange!15}{55.50$_{\pm 0.46}$} \\
& \textbf{DualSFT-Data (Ours)} & \colorbox{orange!15}{51.02$_{\pm 1.45}$} & 78.53$_{\pm 0.98}$ & \colorbox{red!10}{\textbf{58.33$_{\pm 0.40}$}} & \colorbox{red!10}{\textbf{36.54$_{\pm 0.79}$}} & \colorbox{red!10}{\textbf{56.11$_{\pm 0.49}$}} & \colorbox{orange!15}{56.91$_{\pm 0.76}$} & \colorbox{red!10}{\textbf{56.51$_{\pm 0.42}$}} \\
\cmidrule(lr){2-9}
& LESS + S$^2$FT & \colorbox{orange!15}{51.02$_{\pm 1.46}$} & \colorbox{orange!15}{78.65$_{\pm 0.93}$} & 55.68$_{\pm 0.39}$ & 31.13$_{\pm 0.78}$ & 54.12$_{\pm 0.49}$ & \colorbox{orange!15}{53.05$_{\pm 0.86}$} & 53.59$_{\pm 0.48}$ \\
& SPICE + NanoAdam & 50.77$_{\pm 1.45}$ & 78.55$_{\pm 0.94}$ & \colorbox{orange!15}{57.29$_{\pm 0.40}$} & \colorbox{orange!15}{33.89$_{\pm 0.82}$} & \colorbox{orange!15}{55.12$_{\pm 0.49}$} & 52.64$_{\pm 0.76}$ & \colorbox{orange!15}{53.88$_{\pm 0.47}$} \\
& \textbf{DualSFT (Ours)} & \colorbox{red!10}{\textbf{51.14$_{\pm 1.44}$}} & \colorbox{red!10}{\textbf{79.02$_{\pm 0.95}$}} & \colorbox{red!10}{\textbf{59.16$_{\pm 0.38}$}} & \colorbox{red!10}{\textbf{35.61$_{\pm 0.64}$}} & \colorbox{red!10}{\textbf{56.23$_{\pm 0.48}$}} & \colorbox{red!10}{\textbf{56.71$_{\pm 0.86}$}} & \colorbox{red!10}{\textbf{56.47$_{\pm 0.43}$}} \\
\midrule
\multirow{16}{*}{\rotatebox{90}{\raisebox{-0.25\height}{\includegraphics[height=1.3em]{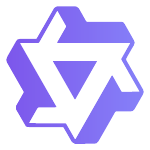}}\hspace{0.3em}\textbf{Qwen-3.5-9B-Base\qquad}}}
& \textcolor{gray}{Pre-trained} & \textcolor{gray}{54.27$_{\pm 0.00}$} & \textcolor{gray}{80.09$_{\pm 0.00}$} & \textcolor{gray}{72.62$_{\pm 0.00}$} & \textcolor{gray}{86.20$_{\pm 0.00}$} & \textcolor{gray}{73.29$_{\pm 0.00}$} & \textcolor{gray}{60.98$_{\pm 0.00}$} & \textcolor{gray}{67.13$_{\pm 0.00}$} \\
& \textcolor{gray}{Standard SFT} & \textcolor{gray}{49.23$_{\pm 1.46}$} & \textcolor{gray}{74.30$_{\pm 0.98}$} & \textcolor{gray}{66.14$_{\pm 0.40}$} & \textcolor{gray}{78.65$_{\pm 0.75}$} & \textcolor{gray}{67.08$_{\pm 0.49}$} & \textcolor{gray}{75.61$_{\pm 0.86}$} & \textcolor{gray}{71.34$_{\pm 0.49}$} \\
& \textcolor{gray}{LoRA}\venuetag{ICLR'22} & \textcolor{gray}{51.51$_{\pm 1.46}$} & \textcolor{gray}{76.62$_{\pm 0.95}$} & \textcolor{gray}{69.47$_{\pm 0.39}$} & \textcolor{gray}{81.93$_{\pm 0.73}$} & \textcolor{gray}{69.88$_{\pm 0.48}$} & \textcolor{gray}{70.33$_{\pm 0.76}$} & \textcolor{gray}{70.10$_{\pm 0.50}$} \\
\cmidrule(lr){2-9}
& S$^2$FT\venuetag{NIPS'24} & 51.73$_{\pm 1.45}$ & \colorbox{orange!15}{78.96$_{\pm 0.93}$} & \colorbox{red!10}{\textbf{70.35$_{\pm 0.37}$}} & 82.06$_{\pm 0.45}$ & 70.78$_{\pm 0.45}$ & 72.56$_{\pm 0.86}$ & 71.67$_{\pm 0.45}$ \\
& SMT\venuetag{ICLR'25} & 50.34$_{\pm 1.45}$ & 76.51$_{\pm 0.95}$ & 68.41$_{\pm 0.39}$ & \colorbox{orange!15}{82.64$_{\pm 0.74}$} & 69.48$_{\pm 0.49}$ & \colorbox{orange!15}{73.78$_{\pm 0.86}$} & 71.63$_{\pm 0.48}$ \\
& LIFT\venuetag{ICML'25} & 51.22$_{\pm 1.44}$ & 77.06$_{\pm 0.96}$ & 68.75$_{\pm 0.40}$ & 82.13$_{\pm 0.72}$ & 69.79$_{\pm 0.49}$ & 71.75$_{\pm 0.76}$ & 70.77$_{\pm 0.43}$ \\
& NanoAdam\venuetag{NIPS'25} & \colorbox{orange!15}{52.87$_{\pm 1.45}$} & \colorbox{red!10}{\textbf{79.04$_{\pm 0.95}$}} & 70.22$_{\pm 0.39}$ & \colorbox{red!10}{\textbf{84.23$_{\pm 0.80}$}} & \colorbox{orange!15}{71.59$_{\pm 0.49}$} & 73.37$_{\pm 0.76}$ & \colorbox{orange!15}{72.48$_{\pm 0.46}$} \\
& \textbf{DualSFT-Param (Ours)} & \colorbox{red!10}{\textbf{53.53$_{\pm 1.44}$}} & 78.73$_{\pm 0.95}$ & \colorbox{orange!15}{70.26$_{\pm 0.39}$} & \colorbox{red!10}{\textbf{84.23$_{\pm 0.70}$}} & \colorbox{red!10}{\textbf{71.69$_{\pm 0.47}$}} & \colorbox{red!10}{\textbf{75.20$_{\pm 0.76}$}} & \colorbox{red!10}{\textbf{73.44$_{\pm 0.44}$}} \\
\cmidrule(lr){2-9}
& Random & \colorbox{red!10}{\textbf{54.10$_{\pm 1.44}$}} & \colorbox{orange!15}{79.63$_{\pm 0.97}$} & \colorbox{red!10}{\textbf{71.55$_{\pm 0.40}$}} & \colorbox{red!10}{\textbf{85.47$_{\pm 0.77}$}} & \colorbox{red!10}{\textbf{72.69$_{\pm 0.48}$}} & 65.24$_{\pm 0.86}$ & 68.97$_{\pm 0.48}$ \\
& LESS\venuetag{ICML'24} & 50.26$_{\pm 1.45}$ & 79.02$_{\pm 0.97}$ & 69.65$_{\pm 0.41}$ & 81.27$_{\pm 0.80}$ & 70.05$_{\pm 0.47}$ & \colorbox{orange!15}{72.76$_{\pm 0.57}$} & \colorbox{orange!15}{71.41$_{\pm 0.40}$} \\
& FisherSFT\venuetag{ICML'25} & 50.94$_{\pm 1.46}$ & 76.15$_{\pm 0.93}$ & 68.49$_{\pm 0.37}$ & 81.35$_{\pm 0.71}$ & 69.23$_{\pm 0.50}$ & 71.54$_{\pm 0.76}$ & 70.39$_{\pm 0.43}$ \\
& SPICE\venuetag{ICLR'26} & 50.85$_{\pm 1.44}$ & 76.15$_{\pm 0.98}$ & 69.22$_{\pm 0.40}$ & 83.02$_{\pm 0.63}$ & 69.81$_{\pm 0.48}$ & 72.15$_{\pm 0.76}$ & 70.98$_{\pm 0.43}$ \\
& \textbf{DualSFT-Data (Ours)} & \colorbox{orange!15}{53.67$_{\pm 1.45}$} & \colorbox{red!10}{\textbf{80.03$_{\pm 0.95}$}} & \colorbox{orange!15}{70.75$_{\pm 0.39}$} & \colorbox{orange!15}{85.04$_{\pm 0.70}$} & \colorbox{orange!15}{72.37$_{\pm 0.48}$} & \colorbox{red!10}{\textbf{74.19$_{\pm 0.57}$}} & \colorbox{red!10}{\textbf{73.28$_{\pm 0.41}$}} \\
\cmidrule(lr){2-9}
& LESS + S$^2$FT & \colorbox{orange!15}{50.97$_{\pm 1.45}$} & \colorbox{orange!15}{79.54$_{\pm 0.96}$} & 69.14$_{\pm 0.40}$ & \colorbox{orange!15}{81.27$_{\pm 0.75}$} & \colorbox{orange!15}{70.23$_{\pm 0.48}$} & 67.89$_{\pm 0.76}$ & 69.06$_{\pm 0.44}$ \\
& SPICE + NanoAdam & 50.54$_{\pm 1.44}$ & 79.34$_{\pm 0.95}$ & \colorbox{orange!15}{70.09$_{\pm 0.39}$} & 80.74$_{\pm 0.71}$ & 70.18$_{\pm 0.48}$ & \colorbox{orange!15}{69.31$_{\pm 0.76}$} & \colorbox{orange!15}{69.74$_{\pm 0.42}$} \\
& \textbf{DualSFT (Ours)} & \colorbox{red!10}{\textbf{54.01$_{\pm 1.43}$}} & \colorbox{red!10}{\textbf{79.63$_{\pm 0.94}$}} & \colorbox{red!10}{\textbf{70.33$_{\pm 0.38}$}} & \colorbox{red!10}{\textbf{84.61$_{\pm 0.68}$}} & \colorbox{red!10}{\textbf{72.15$_{\pm 0.47}$}} & \colorbox{red!10}{\textbf{73.78$_{\pm 0.50}$}} & \colorbox{red!10}{\textbf{72.96$_{\pm 0.39}$}} \\
\bottomrule
\end{tabular}%
}
\vspace{-0.1in}
\end{table*}

\textbf{Matched-Budget Target-Task Performance.} Table~\ref{tab:main_results_magicoder} shows that DualSFT improves target-task performance within each variant's matched-budget comparison group. On Llama-3.2-3B with Magicoder, single-axis variants improve over their counterparts: DualSFT-Param raises HumanEval over NanoAdam from $43.50\%$ to $44.51\%$, and DualSFT-Data raises HumanEval over LESS from $42.68\%$ to $44.11\%$. In the joint-constrained regime, where both axes are restricted, full DualSFT reaches $43.70\%$ HumanEval, $4.27$ points above SPICE+NanoAdam under the same joint budget.

\textbf{Retention and Overall Performance.} Single-axis DualSFT-Param/Data lead their groups in Overall; full DualSFT is compared to joint-budget hybrids. It keeps Stability Avg.\ near pre-trained references, maintains strong HumanEval, and leads joint-budget methods in Stability Avg.\ and Overall, outperforming the strongest sequential hybrid by $2.46$, $2.59$, and $3.22$ points on Llama/Gemma/Qwen. Relative to the pre-trained baseline, DualSFT lowers Stability Avg.\ by $0.35$/$0.57$/$1.14$ points, versus $5.65$/$6.44$/$6.21$ for Standard SFT, indicating less forgetting. Appendix~\ref{app:sequential_grid} corroborates this advantage. Appendix~\ref{app:results_metamathqa} details MetaMathQA trends. Appendix~\ref{app:instruction_tuning} evaluates Alpaca-GPT4, where DualSFT-Param/Data and full DualSFT improve instruction-tuning Overall over matched-budget baselines by $1.72$/$1.69$/$2.13$ points, supporting applicability beyond code and mathematics. Sections~\ref{sec:coord_gap} and~\ref{sec:diagnostic_analysis} examine coordination gaps and one-shot selection.

\textbf{Efficiency and Trade-offs.} 
Fig.~\ref{fig:efficiency_main} compares Overall against time and peak memory. For parameters, DualSFT-Param achieves the highest Overall with the lowest peak fine-tuning memory ($19.3$\,GB) among parameter-selection methods, despite not being the fastest. On the data and joint side, DualSFT-Data attains the best Overall, and full DualSFT maintains $47.12\%$ with $28.5$\,GB peak selection memory and $7.4$ GPU-hours. Under the same joint budget, full DualSFT improves Overall over sequential hybrids, trading modest selection cost for stronger stability-plasticity outcomes.

\begin{figure*}[t]
    \centering
    \begin{subfigure}{0.49\linewidth}
        \centering
        \includegraphics[width=\linewidth]{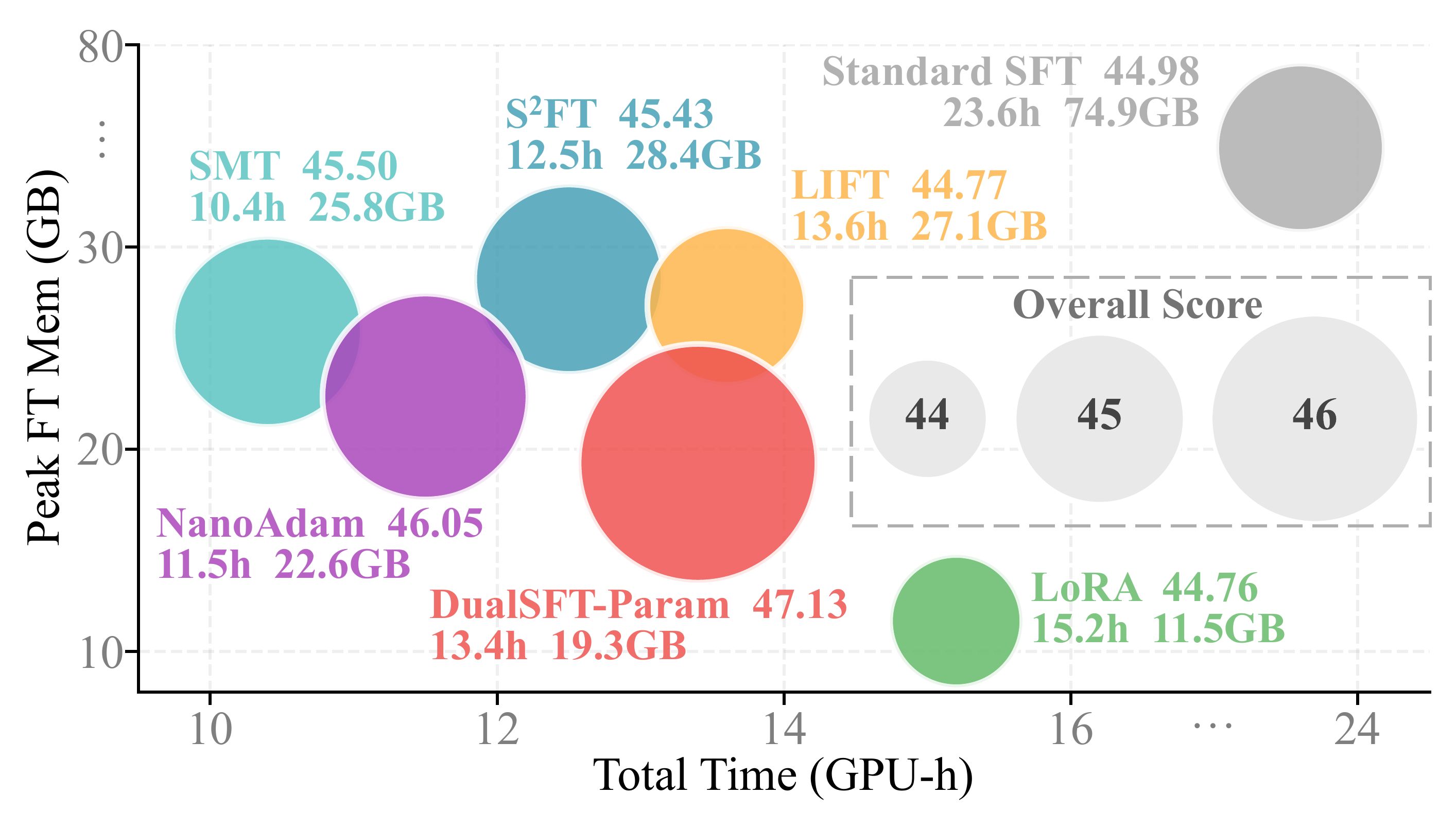}
        \vspace{-0.25in}
        \caption{Parameter selection efficiency frontier.}
        \label{fig:efficiency_param}
    \end{subfigure}\hfill
    \begin{subfigure}{0.49\linewidth}
        \centering
        \includegraphics[width=\linewidth]{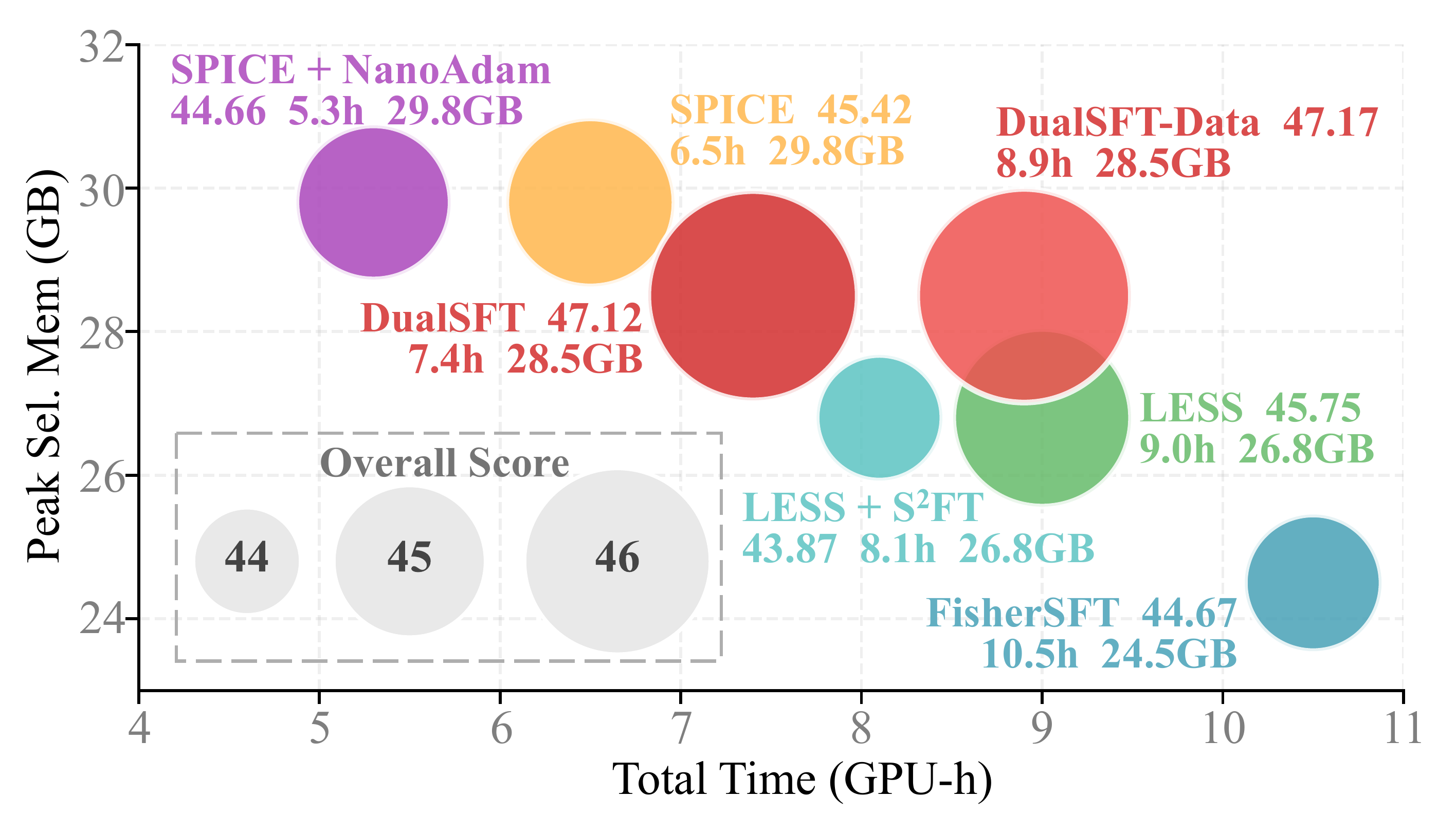} 
        \vspace{-0.25in}
        \caption{Data selection efficiency frontier.}
        \label{fig:efficiency_data}
    \end{subfigure}
    \vspace{-0.05in}
    \caption{Efficiency frontiers comparing against various baselines on Magicoder with Llama-3.2-3B.}
    \label{fig:efficiency_main}
    \vspace{-0.1in}
\end{figure*}

\subsection{Ablation \& Sensitivity Analysis}
\label{sec:ablation}

We examine component contributions, the source of parameter-selection gains relative to Fisher-based scoring, and sensitivity to hyperparameters and the validation set.

\begin{wrapfigure}{r}{0.5\linewidth}
    \vspace{-0.2in}
    \centering
    \includegraphics[width=\linewidth, trim=2mm 0mm 2mm 2mm, clip]{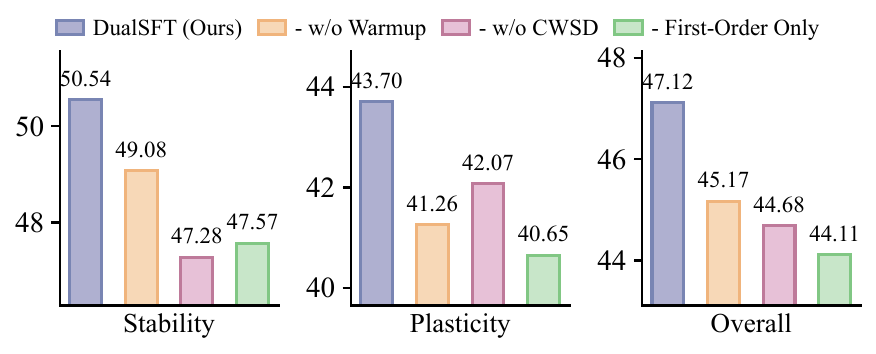}
    \vspace{-0.3in}
    \caption{Component ablation on Llama-3.2-3B.}
    \label{fig:ablation_components}
    \vspace{-10pt}
\end{wrapfigure}

\textbf{Component Ablation.} Fig.~\ref{fig:ablation_components} shows that removing any component lowers both Stability and Plasticity under joint data--parameter restrictions. Removing warmup reduces Plasticity from $43.70\%$ to $41.26\%$, with Overall decreasing by $1.95$ points to $45.17\%$. Omitting CWSD causes the largest Stability drop ($50.54\%\!\to\!47.28\%$) and lowers Plasticity to $42.07\%$, supporting preservation-guided selection on both axes. First-order-only scoring reduces Stability and Plasticity by $2.97$ and $3.05$ points, respectively, producing the largest Overall drop ($47.12\%\!\to\!44.11\%$) and highlighting the diagonal correction's contribution to both retention and adaptation. Appendix~\ref{app:cwsd_ablation} further compares preservation directions and anchor sizes.

\textbf{Comparison with Fisher-Based Parameter Selection.} Table~\ref{tab:fisher_controls_full} compares parameter-selection methods using full training data and a $5\%$ parameter budget. DualSFT-Param exceeds FISH Mask~\cite{Sung2021Training} by $1.79$ Overall points. With checkpoint, scoring pool, learning rate, and training protocol matched, replacing magnitude-only MB-Fisher-Full with DualSFT-Param-FO improves Overall by $1.04$ points; adding the diagonal correction yields $0.64$. Results support the contribution of validation-and-preservation-guided first-order scoring beyond Fisher-style saliency, with a complementary benefit from the diagonal correction. Appendix~\ref{app:fisher_controls} defines the controls and details the protocol.

\begin{figure}[t]
    \centering
    \includegraphics[width=\linewidth]{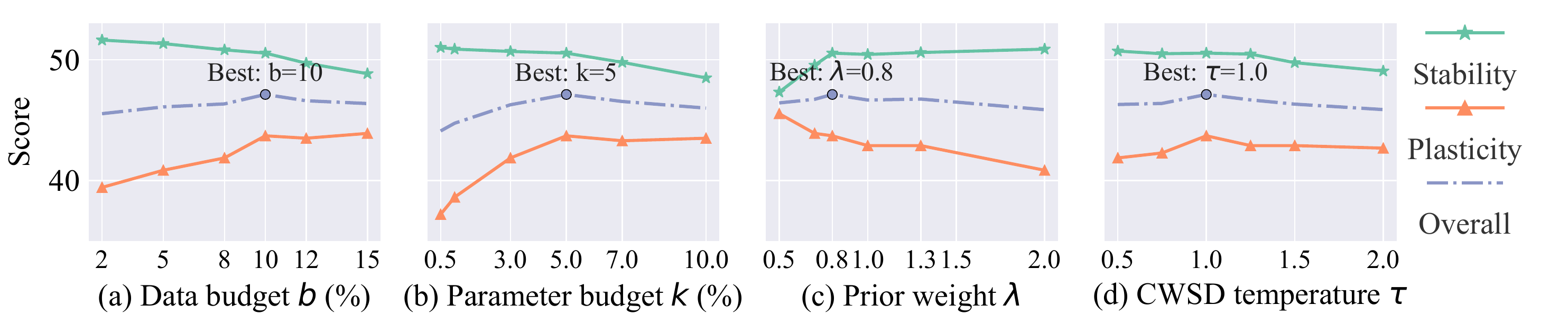}
    \vspace{-0.25in}
    \caption{Sensitivity to key hyperparameters on Llama-3.2-3B (Magicoder).}
    \label{fig:hyperparameter_sensitivity}
    % \vspace{-0.1in}
\end{figure}

\begin{figure}[t]
\centering
\begin{minipage}[t]{0.48\linewidth}
\vspace{0pt}
\centering
\captionof{table}{Parameter-Only Fisher Controls.}
\vspace{-0.1in}
\label{tab:fisher_controls_full}
\small
\setlength{\tabcolsep}{2pt}
\renewcommand{\arraystretch}{1}
\resizebox{\linewidth}{!}{%
\begin{tabular}{lccc}
\toprule
Method & Stability $\uparrow$ & Plasticity $\uparrow$ & Overall $\uparrow$ \\
\midrule
FISH Mask & $47.79_{\pm 0.44}$ & $42.89_{\pm 0.76}$ & $45.34_{\pm 0.60}$ \\
EF-W & $48.14_{\pm 0.43}$ & $43.09_{\pm 0.29}$ & $45.61_{\pm 0.35}$ \\
EF-W+LR & $47.64_{\pm 0.26}$ & $42.68_{\pm 0.50}$ & $45.16_{\pm 0.33}$ \\
MB-Fisher-1K & $47.72_{\pm 0.25}$ & $42.68_{\pm 0.50}$ & $45.20_{\pm 0.33}$ \\
MB-Fisher-Full & $47.61_{\pm 0.23}$ & $43.29_{\pm 0.50}$ & $45.45_{\pm 0.26}$ \\
DualSFT-Param-FO & $49.07_{\pm 0.39}$ & $43.90_{\pm 0.50}$ & $46.49_{\pm 0.38}$ \\
\rowcolor{rowblue}
DualSFT-Param & $\mathbf{49.76}_{\pm 0.50}$ & $\mathbf{44.51}_{\pm 0.86}$ & $\mathbf{47.13}_{\pm 0.42}$ \\
\bottomrule
\end{tabular}
}
\end{minipage}\hfill
\begin{minipage}[t]{0.51\linewidth}
\vspace{0pt}
\centering
\includegraphics[width=\linewidth, trim=0mm 1mm 2mm 2mm, clip]{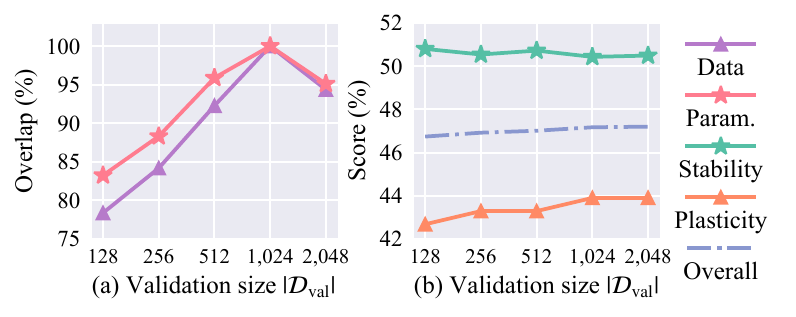}
\vspace{-0.25in}
\captionof{figure}{Validation-size sensitivity: overlap (left) and performance (right).}
\label{fig:validation_size}
\end{minipage}
\vspace{-0.1in}
\end{figure}

\textbf{Hyperparameter Sensitivity and Robustness.} Fig.~\ref{fig:hyperparameter_sensitivity} shows that Overall peaks at $b=10\%$ and $k=5\%$, reflecting the stability--plasticity trade-off as adaptation capacity grows. Performance remains near-optimal across $\lambda=0.7$--$1.3$, with the best settings at $\lambda=0.8$ and $\tau=1.0$. Fig.~\ref{fig:validation_size} examines validation-set size: across $512$--$2{,}048$ examples, data and parameter overlaps exceed $90\%$ relative to the $1{,}024$-example reference, while Overall ranges from $47.01$ to $47.20$. Doubling the reference size improves Overall by only $0.03$ points. Appendix~\ref{app:validation_sensitivity} details the protocol and additional composition and sampling tests, which yield Overall spreads of $0.19$ and $0.21$ points at size $1{,}024$. These results support this validation size as a reasonable cost--stability operating point under the tested conditions.

\subsection{Diagnostic Analysis}
\label{sec:diagnostic_analysis}
Unless specified, all diagnostics in this section are conducted on Llama-3.2-3B with Magicoder.

\begin{wraptable}{r}{0.48\columnwidth}
\vspace{-0.20in}
\centering
\caption{Same-surrogate local-regret diagnostic.}
\label{tab:local_regret}
\vspace{-0.10in}
\footnotesize
\setlength{\tabcolsep}{2.2pt}
\begin{tabular*}{\linewidth}{@{\extracolsep{\fill}}llccc@{}}
\toprule
\textbf{Side} & \textbf{Ckpt.} & \textbf{Jacc.} & \textbf{Ratio} & \textbf{Reg. (\%)} \\
\midrule
\multirow{2}{*}{Data ($b=10\%$)}
& Init   & 0.71 & 0.972 & 2.8 \\
& Warmup & 0.62 & 0.948 & 5.2 \\
\midrule
\multirow{2}{*}{Param ($k=5\%$)}
& Init   & 0.76 & 0.981 & 1.9 \\
& Warmup & 0.68 & 0.964 & 3.6 \\
\bottomrule
\end{tabular*}
\vspace{-0.2in}
\end{wraptable}
\textbf{Same-Surrogate Local-Regret Diagnostic.} Table~\ref{tab:local_regret} examines whether one-shot scores lose aware utility once the other-side restriction is imposed under the shared surrogate $\mathbf{u}$. Jacc., Ratio, and Reg.\ denote the aware-unaware overlap, the aware-utility ratio, and the relative drop. We compare $\langle\mathbf{u}\odot \mathbf{m}_\theta,\mathbf{g}_n\rangle$ vs.\ $\langle\mathbf{u},\mathbf{g}_n\rangle$ for data, and $u_dG_{S_\mathcal{D},d}$ vs.\ $u_dG_d$ for parameters. At warmup the regret is bounded ($5.2\%$ on data, $3.6\%$ on parameters) with Jaccard $0.62$--$0.68$; the smaller init regret is consistent with noisier early gradients. Together with the fixed-side re-selection results below, this is consistent with one-shot dual scoring acting as low-cost ex-ante coordination. Appendix~\ref{app:local_regret} gives definitions and protocol.

\begin{figure*}[h]
% \vspace{-0.08in}
\centering
% ======== Left column: two stacked tables =========
\begin{minipage}[t]{0.42\textwidth}
\vspace{0pt}
\centering
% ---------- Top-left table ----------
\setlength{\tabcolsep}{2pt}
\captionsetup{type=table}
\caption{Conditional re-selection. 
% Rel.~Sel.\ and Rel.~E2E report relative selection-stage and end-to-end costs.
}
\vspace{-0.08in}
\label{tab:conditional_reselection_main}
\resizebox{\linewidth}{!}{
\begin{tabular}{lccccc}
\toprule
\textbf{Method} & \textbf{Stab.} & \textbf{Plas.} & \textbf{Overall} & \textbf{Re. Sel.} & \textbf{Re. E2E} \\
\midrule
DualSFT & 50.54 & 43.70 & 47.12 & 1.00$\times$ & 1.00$\times$ \\
D. $\rightarrow$ P. & 50.48 & 43.90 & 47.19 & 1.06$\times$ & 1.05$\times$ \\
P. $\rightarrow$ D. & 50.61 & 43.90 & 47.26 & 1.45$\times$ & 1.41$\times$ \\
\bottomrule
\end{tabular}
}
\vspace{-0.05in}
% ---------- Bottom-left table ----------
\footnotesize
\footnotesize
\setlength{\tabcolsep}{2pt}
\captionsetup{type=table}
\caption{Selected data statistics.}
\label{tab:selected_data_main}
\resizebox{\linewidth}{!}{
\begin{tabular}{lcccc}
\toprule
\textbf{Data statistic} & \textbf{Sel.} & \textbf{Rand.} & \textbf{Top-loss} & \textbf{Pool} \\
\midrule
Val-grad alignment & 0.42 & 0.18 & 0.31 & 0.15 \\
Avg. training loss & 1.45 & 1.39 & 2.83 & 1.38 \\
Diversity score    & 0.82 & 0.84 & 0.71 & 0.86 \\
\bottomrule
\end{tabular}
}
\vspace{-0.15in}

\end{minipage}
\hfill
% ================== Right column: heatmap ==================
\begin{minipage}[t]{0.56\textwidth}
\vspace{0pt}
\centering
\includegraphics[width=\linewidth]{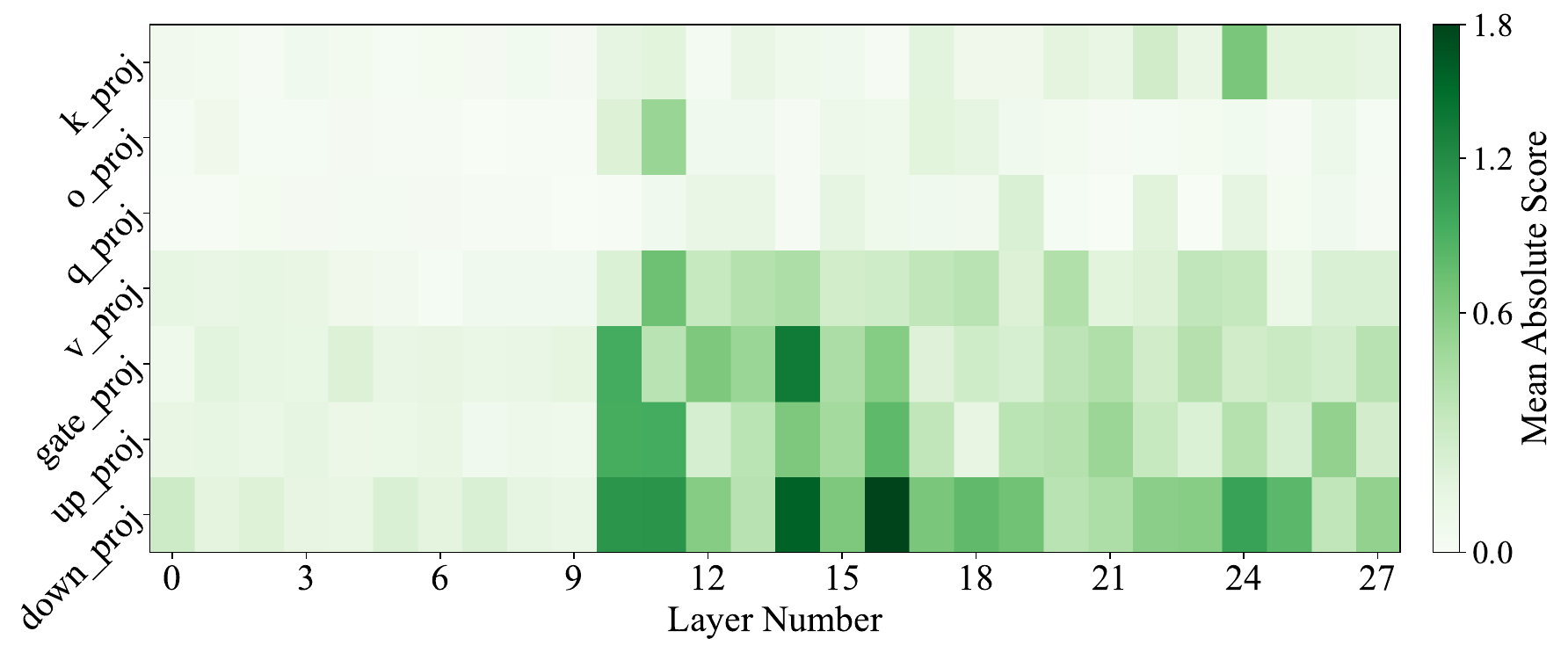}
\vspace{-0.25in}
\captionsetup{type=figure}
\caption{Layer-wise parameter selection on Llama-3.2-3B (Magicoder). Cells show mean absolute block scores. Appendix~\ref{app:selection_statistics} reports exact aggregates.}
\label{fig:layerwise_param_heatmap}
\end{minipage}
\end{figure*}

\textbf{Conditional Re-selection Analysis.} Table~\ref{tab:conditional_reselection_main} compares DualSFT with fixed-side re-selection (D./P.: data/parameter; Rel.~Sel./Rel.~E2E: relative selection/E2E costs). Re-selection improves Overall by at most $0.14$, while increasing selection and E2E costs by $6\%$--$45\%$ and $5\%$--$41\%$. Results support one-shot selection as a cost-effective choice under evaluated budgets. The cost asymmetry follows from granularity: D.\,$\to$\,P.\ aggregates selected-example gradients, whereas P.\,$\to$\,D.\ rescores candidates under a fixed mask. Appendix~\ref{app:conditional_reselection} reports the protocol and cost breakdowns.
% Appendix~\ref{app:dynamic_rescoring} further evaluates re-scoring during training: refreshing both selections improves Overall by $0.46$--$0.50$ points at $1.85$--$2.62\times$ end-to-end cost in a single-seed study.

\textbf{Analysis of Selected Data and Parameters.} Table~\ref{tab:selected_data_main} shows DualSFT-selected data have higher validation-gradient alignment than random, top-loss, and pool baselines, maintaining pool-like loss and diversity, proving selection avoids mere high loss or redundancy. Fig.~\ref{fig:layerwise_param_heatmap} reveals concentrated parameter scores: top five layers contain $42.4\%$ of selected parameters ($2.38\times$ enrichment), whereas early layers (L0 to L9) contain $8.2\%$ ($0.23\times$). Module shares closely match the pool ($0.91\times$ attention, $1.03\times$ FFN), except over-selected FFN down\_proj ($1.44\times$).

\begin{figure*}[t]
\centering
% ===== Left column: two stacked diagnostic tables ====
\begin{minipage}[t]{0.5\textwidth}
\vspace{0pt}
\centering
% ---------- Top-left table ----------
\scriptsize
\setlength{\tabcolsep}{2.4pt}
\captionsetup{type=table}
\caption{Score fidelity diagnostic.}
\label{tab:score_consistency_main}
\vspace{-0.08in}
\resizebox{\linewidth}{!}{
\begin{tabular}{l|cc|cc}
\toprule
\multirow{2}{*}{\textbf{Setting}}
& \multicolumn{2}{c|}{\textbf{Data}}
& \multicolumn{2}{c}{\textbf{Param}} \\
\cmidrule(lr){2-3}\cmidrule(lr){4-5}
& $\rho$ & Pair
& $\rho$ & Pair \\
\midrule
Init + First-Order        & 0.28 & 58.6 & 0.21 & 56.3 \\
Init + Diag-Second-Order  & 0.37 & 64.1 & 0.30 & 60.7 \\
Warmup + First-Order      & 0.56 & 73.1 & 0.49 & 69.1 \\
\rowcolor{rowblue}
\textbf{Warmup + Diag-Second-Order} & \textbf{0.82} & \textbf{79.8} & \textbf{0.79} & \textbf{75.3} \\
\bottomrule
\end{tabular}
}

% \vspace{0.07in}

% ---------- Bottom-left table ----------
\scriptsize
\setlength{\tabcolsep}{2.0pt}
\captionsetup{type=table}
\caption{Diagonal-vs-full rank agreement.}
\label{tab:diag_full_main}
\resizebox{\linewidth}{!}{
\begin{tabular}{l|ccc|ccc}
\toprule
\multirow{2}{*}{\textbf{Checkpoint}}
& \multicolumn{3}{c|}{\textbf{Data}}
& \multicolumn{3}{c}{\textbf{Param}} \\
\cmidrule(lr){2-4}\cmidrule(lr){5-7}
& $\rho$ & Pair & Top-5\%
& $\rho$ & Pair & Top-5\% \\
\midrule
Init
& 0.76 & 83.1 & 78.2
& 0.82 & 87.0 & 83.7 \\
\rowcolor{rowblue}
\textbf{Warmup}
& \textbf{0.86} & \textbf{89.8} & \textbf{86.5}
& \textbf{0.90} & \textbf{92.6} & \textbf{89.4} \\
\bottomrule
\end{tabular}
}

\end{minipage}
\hfill
% === Right column: forgetting-aware comparison figure ===
\begin{minipage}[t]{0.48\textwidth}
\vspace{0pt}
\centering
\includegraphics[width=\linewidth, trim=2mm 2mm 2mm 2mm, clip]{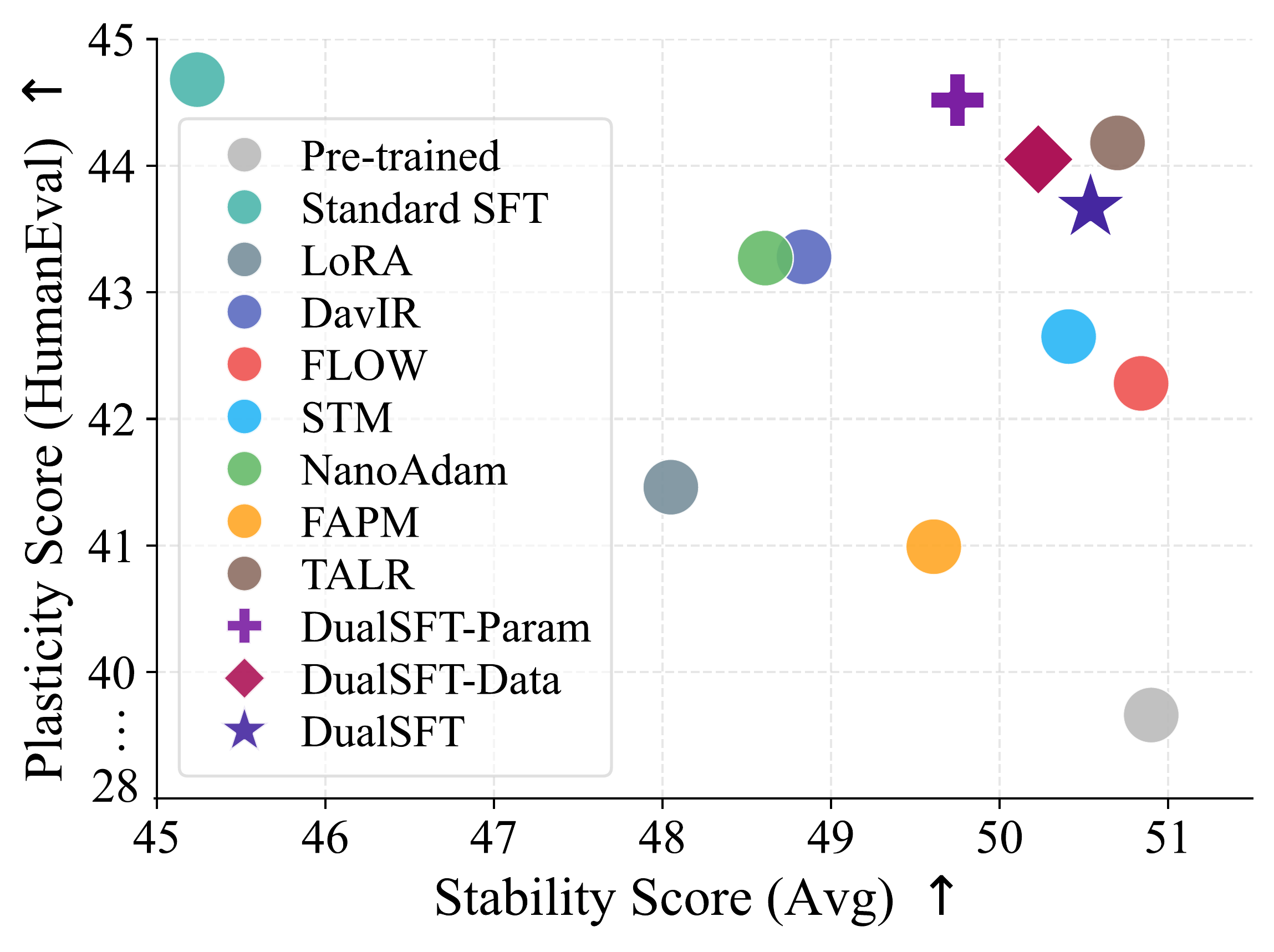}
\vspace{-0.20in}
\captionsetup{type=figure}
\caption{Forgetting-aware baseline comparison on Magicoder using Llama-3.2-3B.}
\label{fig:forgetting_comparison}
\end{minipage}
% \vspace{-0.1in}
\end{figure*}

\textbf{Score Fidelity and Diagonal Approximation.} We compare predicted and realized rankings of data subsets and budget-matched parameter masks under the same validation objective and selection budgets. Table~\ref{tab:score_consistency_main} reports Init/Warmup $\times$ First/Diag-Second-Order variants averaged over $K\in\{50,200,500\}$ AdamW steps, using Spearman $\rho$ and pairwise agreement; warmup and diagonal second-order scoring improve fidelity on both axes. Appendix~\ref{app:score_fidelity} reports declining score fidelity over longer optimization horizons. Table~\ref{tab:diag_full_main} compares diagonal scores with slice-level full-Hessian rankings, supporting the optimizer-aware diagonal proxy as a local ranking approximation.

\subsection{Comparison with Forgetting-Specific Baselines}
We compare DualSFT with some forgetting-specific baselines, including DavIR~\citep{zhou2025davir}, FLOW~\citep{sanyal2025upweighting}, STM~\citep{wu2025mitigating}, NanoAdam~\citep{zhou2025pay}, FAPM~\citep{huang2025mitigating}, and TALR~\citep{lin2026sft}. We follow the main evaluation protocol on Magicoder with three seeds, reporting means and standard deviations over these three runs. Baselines retain their respective adaptation settings; DualSFT-Param/Data use $5\%$ parameter/$10\%$ data budgets, respectively, and full DualSFT applies both. Fig.~\ref{fig:forgetting_comparison} shows that, on Llama-3.2-3B, DualSFT-Param leads the compared forgetting-aware methods in HumanEval ($44.51\%$), while full DualSFT achieves the highest Stability Avg.\ among its variants ($50.54\%$). Across backbones, the best DualSFT variant trails TALR in mean Overall by only $0.16$--$0.46$ points (Table~\ref{tab:forgetting_comparison}), indicating competitive stability--plasticity performance under restricted adaptation budgets.

\section{Conclusion}
DualSFT casts parameter and data selection as two views of a shared validation-improvement surrogate, co-extracting both signals. Under first-order and diagonal/full second-order approximations, they become row-wise and column-wise aggregations of a gradient interaction matrix, with a Shapley interpretation in the localized surrogate game. This turns two historically separate procedures into an efficient one-shot dual-scoring algorithm, reducing trainable parameters and training data in restricted fine-tuning. Experiments on 3B--9B LLMs show favorable stability--plasticity--efficiency trade-offs under matched budgets. DualSFT suggests treating parameters and data as coupled adaptation resources in future work, including longer horizons and post-training.

\newpage

\section*{Acknowledgments and Disclosures}
This work was supported by the National Natural Science Foundation of China under Grant Nos. 42550187 and 62476191, with computing resources provided by our research group. The authors declare no competing interests.

{
\small
\bibliographystyle{IEEEtran}
\bibliography{custom}
}

\newpage
\appendix

\newpage

\section{Detailed Theoretical Derivations and Shapley Proofs}
\label{app:derivations}

This appendix provides four technical complements to the main text: the derivation of the shared one-step surrogate from exact bilevel problems, local truncation guarantees for the Taylor surrogate, detailed proofs of Theorem~\ref{thm:dual_shapley}, and stability analysis of the practical diagonal scores used by DualSFT.

\subsection{From Exact Inner Objectives to One-Step Surrogate Updates}
\label{app:one_step_derivation}

We derive one-step surrogate updates in Eq.~\eqref{eq:single_step_new} from exact bilevel formulations in Eqs.~\eqref{eq:blo_param_new}--\eqref{eq:blo_data_new}. Following Section~\ref{sec:blo_framework}, we linearize the inner response at common checkpoint $\bar{\boldsymbol{\theta}}$ and use local gradients
\begin{equation}
\mathbf{g}_n \triangleq \nabla_{\boldsymbol{\theta}}\ell(x_n,y_n;\bar{\boldsymbol{\theta}}),
\qquad
\mathbf{G}\triangleq \sum_{n=1}^{N}\mathbf{g}_n.
\end{equation}
For the parameter-side problem in Eq.~\eqref{eq:blo_param_new}, a single gradient step on the inner objective under the coordinate mask $\mathbf{m}_{\theta}$ gives the local update
\begin{equation}
\Delta\boldsymbol{\theta}(\mathbf{m}_{\theta})
\approx
-\eta \,\nabla_{\boldsymbol{\theta}}\mathcal{L}_{\text{train}}(\bar{\boldsymbol{\theta}};\mathcal{D}_{\text{train}})
\odot \mathbf{m}_{\theta}
=
-\eta(\mathbf{G}\odot \mathbf{m}_{\theta}).
\end{equation}
For the data-side problem in Eq.~\eqref{eq:blo_data_new}, the selected subset induces the masked empirical objective
$\mathbf{1}^{\top}\!\big(\mathbf{m}_{\mathcal D}\odot \boldsymbol{\ell}_{\text{train}}(\boldsymbol{\theta})\big)$.
A single gradient step at $\bar{\boldsymbol{\theta}}$ gives
\begin{equation}
\Delta\boldsymbol{\theta}(\mathbf{m}_{\mathcal D})
\approx
-\eta \,\nabla_{\boldsymbol{\theta}}\!\left[
\mathbf{1}^{\top}\!\big(\mathbf{m}_{\mathcal D}\odot \boldsymbol{\ell}_{\text{train}}(\bar{\boldsymbol{\theta}})\big)
\right]
=
-\eta\sum_{n=1}^{N}(\mathbf{m}_{\mathcal D})_n\mathbf{g}_n.
\end{equation}
Identifying each binary mask with its selected index set gives the localized functions in Eq.~\eqref{eq:surrogate_set_functions}. These one-step surrogates underlie the Taylor analysis in Eq.~\eqref{eq:taylor_local_utility} and shared-score derivation in Section~\ref{sec:shared_scores}.

\subsection{Local Truncation Error}
\label{app:local_error}

We now justify the first- and second-order truncations of Eq.~\eqref{eq:taylor_local_utility}. For either the parameter-side or data-side utility, define
\begin{equation}
U_{\mathrm{step}}(\mathbf{S})
\triangleq
\mathcal{L}_{\text{val}}(\bar{\boldsymbol{\theta}})
-
\mathcal{L}_{\text{val}}\!\left(\bar{\boldsymbol{\theta}}+\Delta\boldsymbol{\theta}(\mathbf{S})\right),
\qquad
\Delta\boldsymbol{\theta}(\mathbf{S})=-\eta\,\mathbf{g}(\mathbf{S}).
\end{equation}

\begin{proposition}[Local truncation error]
\label{prop:local_truncation}
Assume $\mathcal{L}_{\text{val}}$ is $\beta$-smooth in a neighborhood of $\bar{\boldsymbol{\theta}}$. Then the first-order truncation
\begin{equation}
U_{\mathrm{1st}}(\mathbf{S})
=
-\mathbf{v}_{\text{val}}^{\top}\Delta\boldsymbol{\theta}(\mathbf{S})
\end{equation}
satisfies
\begin{equation}
\big|U_{\mathrm{step}}(\mathbf{S})-U_{\mathrm{1st}}(\mathbf{S})\big|
\le
\frac{\beta}{2}\eta^{2}\|\mathbf{g}(\mathbf{S})\|_{2}^{2}.
\end{equation}
If, in addition, $\nabla^{2}\mathcal{L}_{\text{val}}$ is $\rho$-Lipschitz on a neighborhood containing the segment between $\bar{\boldsymbol{\theta}}$ and $\bar{\boldsymbol{\theta}}+\Delta\boldsymbol{\theta}(\mathbf{S})$, then the second-order truncation
\begin{equation}
U_{\mathrm{2nd}}(\mathbf{S})
=
-\mathbf{v}_{\text{val}}^{\top}\Delta\boldsymbol{\theta}(\mathbf{S})
-\frac{1}{2}\Delta\boldsymbol{\theta}(\mathbf{S})^{\top}\mathbf{H}\Delta\boldsymbol{\theta}(\mathbf{S})
\end{equation}
satisfies
\begin{equation}
\big|U_{\mathrm{step}}(\mathbf{S})-U_{\mathrm{2nd}}(\mathbf{S})\big|
\le
\frac{\rho}{6}\eta^{3}\|\mathbf{g}(\mathbf{S})\|_{2}^{3}.
\end{equation}
\end{proposition}

\begin{proof}
By $\beta$-smoothness of $\mathcal{L}_{\text{val}}$, for any $\Delta$ in a neighborhood of $\bar{\boldsymbol{\theta}}$,
\begin{equation}
\left|
\mathcal{L}_{\text{val}}(\bar{\boldsymbol{\theta}}+\Delta)
-
\mathcal{L}_{\text{val}}(\bar{\boldsymbol{\theta}})
-
\nabla \mathcal{L}_{\text{val}}(\bar{\boldsymbol{\theta}})^{\top}\Delta
\right|
\le
\frac{\beta}{2}\|\Delta\|_{2}^{2}.
\end{equation}
Substituting $\Delta=\Delta\boldsymbol{\theta}(\mathbf{S})$ gives
\begin{equation}
\big|U_{\mathrm{step}}(\mathbf{S})-U_{\mathrm{1st}}(\mathbf{S})\big|
\le
\frac{\beta}{2}\eta^{2}\|\mathbf{g}(\mathbf{S})\|_{2}^{2}.
\end{equation}
If $\nabla^{2}\mathcal{L}_{\text{val}}$ is $\rho$-Lipschitz, then
\begin{equation}
\left|
\mathcal{L}_{\text{val}}(\bar{\boldsymbol{\theta}}+\Delta)
-
\mathcal{L}_{\text{val}}(\bar{\boldsymbol{\theta}})
-
\nabla \mathcal{L}_{\text{val}}(\bar{\boldsymbol{\theta}})^{\top}\Delta
-
\frac{1}{2}\Delta^{\top}\mathbf{H}\Delta
\right|
\le
\frac{\rho}{6}\|\Delta\|_{2}^{3}.
\end{equation}
Substituting $\Delta=\Delta\boldsymbol{\theta}(\mathbf{S})$ gives
\begin{equation}
\big|U_{\mathrm{step}}(\mathbf{S})-U_{\mathrm{2nd}}(\mathbf{S})\big|
\le
\frac{\rho}{6}\eta^{3}\|\mathbf{g}(\mathbf{S})\|_{2}^{3}.
\end{equation}
\end{proof}

\subsection{Analytical Shapley Values under Approximation Hierarchies}
\label{app:shapley}

We now prove Theorem~\ref{thm:dual_shapley}. For each $\alpha\in\{1,\mathrm{diag},\mathrm{full}\}$, let
\begin{equation}
\mathcal{G}_{\theta}^{(\alpha)}=([D],U_{\theta}^{(\alpha)}),
\qquad
\mathcal{G}_{\mathcal D}^{(\alpha)}=([N],U_{\mathcal D}^{(\alpha)})
\end{equation}
be the localized surrogate games defined in Section~\ref{sec:shared_scores}, and write
\begin{equation}
\phi_{\theta,d}^{(\alpha)}\triangleq \phi_d(\mathcal{G}_{\theta}^{(\alpha)}),
\qquad
\phi_{\mathcal D,n}^{(\alpha)}\triangleq \phi_n(\mathcal{G}_{\mathcal D}^{(\alpha)}).
\end{equation}
For a finite player set $\mathcal{P}$, utility function $v:2^{\mathcal{P}}\to\mathbb{R}$, and player $i\in\mathcal{P}$, the Shapley value is
\begin{equation}
\label{eq:shapley_def}
\phi_i(v)
=
\sum_{S\subseteq \mathcal{P}\setminus\{i\}}
\frac{|S|!(|\mathcal{P}|-|S|-1)!}{|\mathcal{P}|!}
\left[v(S\cup\{i\})-v(S)\right].
\end{equation}
Equivalently, if $\pi$ is a random permutation of $\mathcal{P}$ and $P_i^\pi$ denotes players preceding $i$ in $\pi$,
\begin{equation}
\label{eq:shapley_perm}
\phi_i(v)
=
\mathbb{E}_{\pi}\!\left[v(P_i^\pi\cup\{i\})-v(P_i^\pi)\right].
\end{equation}

\begin{lemma}[Half-Averaging under Permutations]
\label{lem:half_avg}
Fix $i\in\mathcal{P}$. For any coefficients $\{a_j\}_{j\in\mathcal{P}\setminus\{i\}}$,
\begin{equation}
\mathbb{E}_{\pi}\!\left[\sum_{j\in P_i^\pi} a_j\right]
=
\frac{1}{2}\sum_{j\neq i} a_j.
\end{equation}
The same identity holds componentwise for vector-valued $a_j$.
\end{lemma}

\begin{proof}
For each fixed $j\neq i$, under a uniformly random permutation, the probability that $j$ appears before $i$ is $1/2$. Applying linearity of expectation gives the result.
\end{proof}

\begin{proof}[Proof of Theorem~\ref{thm:dual_shapley}]
Let $M_{n,d}^{(1)}$, $M_{n,d}^{(\mathrm{diag})}$, and $M_{n,d}^{(\mathrm{full})}$ be defined as in Eq.~\eqref{eq:interaction_matrices}.

\textbf{First-order case.}
Using Eq.~\eqref{eq:taylor_local_utility}, the first-order parameter-side utility is
\begin{equation}
U_{\theta}^{(1)}(S)
=
-\mathbf{v}_{\text{val}}^\top \Delta\boldsymbol{\theta}_{\theta}(S)
=
\sum_{d\in S}\sum_{n=1}^{N} M_{n,d}^{(1)}.
\end{equation}
For any $S\subseteq [D]\setminus\{d\}$,
\begin{equation}
U_{\theta}^{(1)}(S\cup\{d\})-U_{\theta}^{(1)}(S)
=
\sum_{n=1}^{N} M_{n,d}^{(1)}
=
\phi_{\theta,d}^{(1)}
=
s_{\theta,d}^{(1)}.
\end{equation}
\begin{equation}
U_{\mathcal D}^{(1)}(S)
=
-\mathbf{v}_{\text{val}}^\top \Delta\boldsymbol{\theta}_{\mathcal D}(S)
=
\sum_{n\in S}\sum_{d=1}^{D} M_{n,d}^{(1)}.
\end{equation}
\begin{equation}
U_{\mathcal D}^{(1)}(S\cup\{n\})-U_{\mathcal D}^{(1)}(S)
=
\sum_{d=1}^{D} M_{n,d}^{(1)}
=
\phi_{\mathcal D,n}^{(1)}
=
s_{\mathcal D,n}^{(1)}.
\end{equation}

\textbf{Diagonal-curvature case.}
For the parameter-side game,
\begin{equation}
U_{\theta}^{(\mathrm{diag})}(S)
=
\sum_{d\in S}\sum_{n=1}^{N} M_{n,d}^{(1)}
-\frac{\eta^2}{2}\sum_{d\in S} c_d G_d^2,
\end{equation}
\begin{equation}
U_{\theta}^{(\mathrm{diag})}(S\cup\{d\})-U_{\theta}^{(\mathrm{diag})}(S)
=
\sum_{n=1}^{N} M_{n,d}^{(\mathrm{diag})}
=
\phi_{\theta,d}^{(\mathrm{diag})}
=
s_{\theta,d}^{(\mathrm{diag})}.
\end{equation}
For the data-side game,
\begin{equation}
U_{\mathcal D}^{(\mathrm{diag})}(S)
=
\sum_{n\in S}\sum_{d=1}^{D} M_{n,d}^{(1)}
-\frac{\eta^2}{2}\sum_{d=1}^{D} c_d\left(\sum_{m\in S} g_{m,d}\right)^2.
\end{equation}
\begin{equation}
U_{\mathcal D}^{(\mathrm{diag})}(S\cup\{n\})-U_{\mathcal D}^{(\mathrm{diag})}(S)
=
\sum_{d=1}^{D} M_{n,d}^{(1)}
-\frac{\eta^2}{2}\sum_{d=1}^{D} c_d\left[g_{n,d}^{2}+2g_{n,d}\sum_{m\in S}g_{m,d}\right].
\end{equation}
Applying Eq.~\eqref{eq:shapley_perm} and Lemma~\ref{lem:half_avg} coordinatewise replaces $\sum_{m\in S} g_{m,d}$ by $\frac{1}{2}\sum_{m\neq n}g_{m,d}$ under Shapley averaging, yielding
\begin{equation}
\phi_{\mathcal D,n}^{(\mathrm{diag})}
=
\sum_{d=1}^{D}\left(M_{n,d}^{(1)}-\frac{\eta^2}{2}c_d G_d g_{n,d}\right)
=
\sum_{d=1}^{D} M_{n,d}^{(\mathrm{diag})}
=
s_{\mathcal D,n}^{(\mathrm{diag})}.
\end{equation}

\textbf{Full-Hessian case.}
For the parameter-side game,
\begin{equation}
U_{\theta}^{(\mathrm{full})}(S)
=
\sum_{d\in S}\sum_{n=1}^{N} M_{n,d}^{(1)}
-\frac{\eta^2}{2}\sum_{i\in S}\sum_{j\in S} G_i H_{ij} G_j.
\end{equation}
For any $S\subseteq [D]\setminus\{d\}$, by symmetry of $\mathbf{H}$,
\begin{equation}
U_{\theta}^{(\mathrm{full})}(S\cup\{d\})-U_{\theta}^{(\mathrm{full})}(S)
=
\sum_{n=1}^{N} M_{n,d}^{(1)}
-\frac{\eta^2}{2}\left[G_d H_{dd} G_d + 2G_d\sum_{l\in S}H_{dl}G_l\right].
\end{equation}
Applying Eq.~\eqref{eq:shapley_perm} and Lemma~\ref{lem:half_avg} with $a_l=H_{dl}G_l$ gives
\begin{equation}
\phi_{\theta,d}^{(\mathrm{full})}
=
\sum_{n=1}^{N} M_{n,d}^{(1)}-\frac{\eta^2}{2}(\mathbf{H}\mathbf{G})_d G_d
=
\sum_{n=1}^{N} M_{n,d}^{(\mathrm{full})}
=
s_{\theta,d}^{(\mathrm{full})}.
\end{equation}
For the data-side game,
\begin{equation}
U_{\mathcal D}^{(\mathrm{full})}(S)
=
\sum_{n\in S}\sum_{d=1}^{D} M_{n,d}^{(1)}
-\frac{\eta^2}{2}\left(\sum_{m\in S}\mathbf{g}_m\right)^\top \mathbf{H}\left(\sum_{m\in S}\mathbf{g}_m\right).
\end{equation}
\begin{equation}
U_{\mathcal D}^{(\mathrm{full})}(S\cup\{n\})-U_{\mathcal D}^{(\mathrm{full})}(S)
=
\sum_{d=1}^{D} M_{n,d}^{(1)}
-\frac{\eta^2}{2}\left[\mathbf{g}_n^\top \mathbf{H}\mathbf{g}_n + 2\mathbf{g}_n^\top \mathbf{H}\sum_{m\in S}\mathbf{g}_m\right].
\end{equation}
Applying Eq.~\eqref{eq:shapley_perm} and Lemma~\ref{lem:half_avg} in vector form yields
\begin{equation}
\phi_{\mathcal D,n}^{(\mathrm{full})}
=
\sum_{d=1}^{D}\left(M_{n,d}^{(1)}-\frac{\eta^2}{2}(\mathbf{H}\mathbf{G})_d g_{n,d}\right)
=
\sum_{d=1}^{D} M_{n,d}^{(\mathrm{full})}
=
s_{\mathcal D,n}^{(\mathrm{full})}.
\end{equation}
This completes the proof.
\end{proof}

\subsection{Proof of Proposition~\ref{prop:coord_gap}}
\label{app:coord_gap}

Under the joint restrictions, the first-order update is \(\Delta\boldsymbol{\theta}(S_\mathcal{D},\mathbf{m}_\theta)=-\eta(\mathbf{m}_\theta\odot \mathbf{G}_{S_\mathcal{D}})\), where \(\mathbf{G}_{S_\mathcal{D}}=\sum\nolimits_{n\in S_\mathcal{D}}\mathbf{g}_n\) and \(S_\theta=\{d\in[D]:(\mathbf{m}_\theta)_d=1\}\). The corresponding first-order utility decomposes as
\begin{equation}
-\mathbf{v}_{\text{val}}^\top \Delta\boldsymbol{\theta}(S_\mathcal{D},\mathbf{m}_\theta)
=
\eta\sum\nolimits_{n\in S_\mathcal{D}}\langle\mathbf{v}_{\text{val}}\odot \mathbf{m}_\theta,\mathbf{g}_n\rangle
=
\eta\sum\nolimits_{d\in S_\theta}(\mathbf{v}_{\text{val}})_dG_{S_\mathcal{D},d}.
\end{equation}
Thus, the mask-aware marginal scores compatible with the joint restricted update are \(\bar s_{\mathcal{D},n}(\mathbf{m}_\theta)=\eta\langle\mathbf{v}_{\text{val}}\odot \mathbf{m}_\theta,\mathbf{g}_n\rangle\) and \(\bar s_{\theta,d}(S_\mathcal{D})=\eta(\mathbf{v}_{\text{val}})_dG_{S_\mathcal{D},d}\).

The isolated first-order scores are \(s^{\mathrm{iso}}_{\mathcal{D},n}=\eta\langle\mathbf{v}_{\text{val}},\mathbf{g}_n\rangle\) and \(s^{\mathrm{iso}}_{\theta,d}=\eta(\mathbf{v}_{\text{val}})_dG_d\), where \(\mathbf{G}=\sum\nolimits_{n=1}^{N}\mathbf{g}_n\). Subtracting the mask-aware data score from the isolated data score gives
\begin{equation}
s^{\mathrm{iso}}_{\mathcal{D},n}
-\bar s_{\mathcal{D},n}(\mathbf{m}_\theta)
=
\eta\langle\mathbf{v}_{\text{val}}\odot(\mathbf{1}_D-\mathbf{m}_\theta),\mathbf{g}_n\rangle
=
\eta\sum\nolimits_{d:(\mathbf{m}_\theta)_d=0}(\mathbf{v}_{\text{val}})_d g_{n,d}.
\end{equation}
Similarly, subtracting the mask-aware parameter score from the isolated parameter score gives
\begin{equation}
s^{\mathrm{iso}}_{\theta,d}
-\bar s_{\theta,d}(S_\mathcal{D})
=
\eta(\mathbf{v}_{\text{val}})_d(G_d-G_{S_\mathcal{D},d})
=
\eta(\mathbf{v}_{\text{val}})_d\sum\nolimits_{n\in[N]\setminus S_\mathcal{D}} g_{n,d}.
\end{equation}
These two identities prove Proposition~\ref{prop:coord_gap}. They show that the data-side gap is supported on frozen coordinates, whereas the parameter-side gap is supported on samples excluded by the data selection. This characterizes the isolated-scoring gap, but it does not mean that DualSFT performs fully conditional re-scoring. DualSFT instead extracts both marginal scores once from the same local surrogate, leaving fixed-side re-selection as an optional refinement analyzed in Section~\ref{sec:diagnostic_analysis}.

\vspace{-0.05in}

\subsection{Practical Diagonal Scores and Selection Stability}
\label{app:error_propagation}

We analyze scores in Eqs.~\eqref{eq:u_vector}--\eqref{eq:practical_scores}. Let $\widehat{D}\triangleq \operatorname{Diag}(\hat{\mathbf{c}})$, and write $\widehat{s}_{\theta,d}$ and $\widehat{s}_{\mathcal D,n}$ for scores in Eq.~\eqref{eq:practical_scores}. For comparison, write $s_{\theta,d}^{(\mathrm{full})}$ and $s_{\mathcal D,n}^{(\mathrm{full})}$ for the local full-Hessian counterparts obtained from Eq.~\eqref{eq:closed_form_scores} with the step size instantiated as $\eta_{\mathrm{sc}}$.

\begin{proposition}[Diagonal-Proxy Score Perturbation]
\label{prop:direct_score_gap}
For parameter $d\in[D]$ and sample $n\in[N]$,
\begin{equation}
\widehat{s}_{\theta,d}
-
s_{\theta,d}^{(\mathrm{full})}
=
\frac{\eta_{\mathrm{sc}}^2}{2}\big((\mathbf{H}-\widehat{D})\mathbf{G}\big)_d G_d,
\qquad
\widehat{s}_{\mathcal D,n}
-
s_{\mathcal D,n}^{(\mathrm{full})}
=
\frac{\eta_{\mathrm{sc}}^2}{2}\mathbf{g}_n^\top (\mathbf{H}-\widehat{D})\mathbf{G}.
\end{equation}
Consequently,
\begin{equation}
\|\widehat{\mathbf{s}}_{\theta}-\mathbf{s}_{\theta}^{(\mathrm{full})}\|_{\infty}
\le
\frac{\eta_{\mathrm{sc}}^2}{2}\|\mathbf{H}-\widehat{D}\|_{\mathrm{op}}\|\mathbf{G}\|_2^2,
\qquad
\|\widehat{\mathbf{s}}_{\mathcal D}-\mathbf{s}_{\mathcal D}^{(\mathrm{full})}\|_{\infty}
\le
\frac{\eta_{\mathrm{sc}}^2}{2}B_g\|\mathbf{H}-\widehat{D}\|_{\mathrm{op}}\|\mathbf{G}\|_2
\end{equation}
if $\|\mathbf{g}_n\|_2\le B_g$ for all $n$.
\end{proposition}

\begin{proof}
Subtracting the definitions of the diagonal scores and the full-Hessian scores gives the two identities directly. The $\ell_{\infty}$ bounds follow from Cauchy--Schwarz and the operator-norm inequality.
\end{proof}

\begin{corollary}[Top-$k$ parameter selection stability]
\label{cor:topk_param_stability}
Let $\widehat{s}_{\theta,(1)}\ge\cdots\ge\widehat{s}_{\theta,(D)}$ denote the sorted practical parameter scores, and define
\begin{equation}
\Delta_k^{\theta}
\triangleq
\widehat{s}_{\theta,(k)}-\widehat{s}_{\theta,(k+1)}.
\end{equation}
If
\begin{equation}
\Delta_k^{\theta}
>
\eta_{\mathrm{sc}}^2\|\mathbf{H}-\widehat{D}\|_{\mathrm{op}}\|\mathbf{G}\|_2^2,
\end{equation}
then the Top-$k$ parameter set induced by $\widehat{\mathbf{s}}_{\theta}$ is identical to that induced by $\mathbf{s}_{\theta}^{(\mathrm{full})}$.
\end{corollary}

\begin{proof}
By Proposition~\ref{prop:direct_score_gap},
\begin{equation}
\|\widehat{\mathbf{s}}_{\theta}-\mathbf{s}_{\theta}^{(\mathrm{full})}\|_{\infty}
\le
\varepsilon_{\theta}
\triangleq
\frac{\eta_{\mathrm{sc}}^2}{2}\|\mathbf{H}-\widehat{D}\|_{\mathrm{op}}\|\mathbf{G}\|_2^2.
\end{equation}
$\Delta_k^{\theta}>2\varepsilon_{\theta}$ implies Top-$k$ stability by the standard boundary-gap argument.
\end{proof}

\begin{corollary}[Top-$b$ data selection stability]
\label{cor:topb_data_stability}
Let $\widehat{s}_{\mathcal D,(1)}\ge\cdots\ge\widehat{s}_{\mathcal D,(N)}$ denote the sorted practical data scores, and define
\begin{equation}
\Delta_b^{\mathcal D}
\triangleq
\widehat{s}_{\mathcal D,(b)}-\widehat{s}_{\mathcal D,(b+1)}.
\end{equation}
If
\begin{equation}
\Delta_b^{\mathcal D}
>
\eta_{\mathrm{sc}}^2 B_g\,\|\mathbf{H}-\widehat{D}\|_{\mathrm{op}}\,\|\mathbf{G}\|_2,
\end{equation}
then the Top-$b$ data set induced by $\widehat{\mathbf{s}}_{\mathcal D}$ is identical to that induced by $\mathbf{s}_{\mathcal D}^{(\mathrm{full})}$.
\end{corollary}

\begin{proof}
By Proposition~\ref{prop:direct_score_gap},
\begin{equation}
\|\widehat{\mathbf{s}}_{\mathcal D}-\mathbf{s}_{\mathcal D}^{(\mathrm{full})}\|_{\infty}
\le
\varepsilon_{\mathcal D}
\triangleq
\frac{\eta_{\mathrm{sc}}^2}{2}B_g\,\|\mathbf{H}-\widehat{D}\|_{\mathrm{op}}\,\|\mathbf{G}\|_2.
\end{equation}
$\Delta_b^{\mathcal D}>2\varepsilon_{\mathcal D}$ implies Top-$b$ stability by the same argument.
\end{proof}

\begin{remark}[Scope of the comparison]
\label{rem:reference_utility}
This subsection compares practical diagonal scores with local full-Hessian counterparts. It does not claim to recover exact bilevel optimum but shows diagonal approximation is controlled and, under a boundary-gap condition, preserves Top-$k$/Top-$b$ selection.
\end{remark}

\section{Optimizer Compatibility: SGD Surrogate and Practical Curvature Proxy}
\vspace{-0.05in}
\label{app:adam_compatibility}

Sections~\ref{sec:blo_framework} to \ref{sec:shared_scores} analyze a localized single-step surrogate, for which Theorem~\ref{thm:dual_shapley} gives closed-form Shapley decompositions of its surrogate games. In practice, DualSFT retains this surrogate utility and instantiates its diagonal second-order coefficient with an optimizer-derived proxy. Let $\mathbf{r}_{t_w}$ denote AdamW's second-moment buffer after $t_w$ warmup updates, with decay $\beta_2$ and stabilizer $\epsilon_{\mathrm{adam}}$. Define $\hat{\mathbf{c}}=\sqrt{\mathbf{r}_{t_w}/(1-\beta_2^{t_w})}+\epsilon_{\mathrm{adam}}$ and keep it fixed during one-shot scoring. This is a positive damped RMS-gradient scale, not an exact Hessian diagonal. It aligns with recent analyses of Adam-family preconditioning and memory-reduced adaptive optimizers~\citep{zhang2025adammini,robert2025ldadam,vyas2025soap,kalra2025when}. Accordingly, the exact row-column Shapley decomposition remains that of the SGD-defined localized surrogate, while the practical second-order score should be interpreted as its optimizer-aware instantiation.

Accordingly, any implementation-level mismatch introduced by the practical diagonal proxy is absorbed into the curvature/proxy utility term analyzed in Appendix~\ref{app:error_propagation}, so that the resulting score perturbation and selection stability can be controlled within the same error-propagation framework.

\section{Detailed Derivation of Confidence-Weighted Self-Distillation}
\label{app:cwsd_derivation}

Section~\ref{sec:algorithm} introduces CWSD in Eq.~\eqref{eq:cwsd} as a tractable surrogate for prior-preserving direction $\mathbf{v}_{\text{prior}}$ without the pre-training corpus. This appendix proves CWSD is non-degenerate at warmup checkpoint $\bar{\boldsymbol{\theta}}$, and its gradient reduces to a vector-Jacobian product.

\subsection{Objective and Evaluation Point}
\label{app:cwsd_objective_eval_point}

Let $\mathcal{D}_{\text{anchor}}$ be the held-out anchor set with $\mathcal{D}_{\text{anchor}}\cap\mathcal{D}_{\text{train}}=\emptyset$, and let $\mathcal{X}_{\text{anchor}}$ denote its input set.
\begin{equation}
\label{eq:app_cwsd_objective}
\mathcal{L}_{\text{CWSD}}(\bar{\boldsymbol{\theta}})
=
\frac{1}{|\mathcal{X}_{\text{anchor}}|}
\textstyle \sum_{x\in \mathcal{X}_{\text{anchor}}}
\omega(x)\tau^2
\mathrm{KL}\!\Big(
P_\tau(\cdot \mid \boldsymbol{\theta}_{\text{old}}, x)
\,\big\|\,
P_\tau(\cdot \mid \bar{\boldsymbol{\theta}}, x)
\Big).
\end{equation}
The surrogate preservation direction is
\begin{equation}
\label{eq:app_cwsd_prior_direction}
\mathbf{v}_{\text{prior}}
\triangleq
\nabla_{\bar{\boldsymbol{\theta}}} \mathcal{L}_{\text{CWSD}}(\bar{\boldsymbol{\theta}}).
\end{equation}

CWSD is evaluated at the warmup checkpoint $\bar{\boldsymbol{\theta}}$ rather than $\boldsymbol{\theta}_{\text{old}}$. If evaluated at $\boldsymbol{\theta}_{\text{old}}$, teacher and student distributions would coincide, and the KL term and its gradient would both be zero. Using $\bar{\boldsymbol{\theta}}$ avoids this degeneracy and yields a non-trivial surrogate preservation direction for one-shot scoring.

\subsection{Gradient Derivation}
\label{app:cwsd_gradient_derivation}

For one anchor example $x$, let
\begin{equation}
\mathbf{q}=P_\tau(\cdot \mid \boldsymbol{\theta}_{\text{old}}, x),
\qquad
\mathbf{p}=P_\tau(\cdot \mid \bar{\boldsymbol{\theta}}, x),
\end{equation}
and let $\mathbf{z}$ denote the student logits at $\bar{\boldsymbol{\theta}}$. The per-example CWSD loss is
\begin{equation}
\ell_{\text{CWSD}}(x)
=
\omega(x)\tau^2 \mathrm{KL}(\mathbf{q}\,\|\,\mathbf{p})
=
\omega(x)\tau^2 \sum\nolimits_i q_i(\log q_i-\log p_i).
\end{equation}
Since $\mathbf{q}$ is produced by the frozen reference model $\boldsymbol{\theta}_{\text{old}}$, gradients flow only through $\mathbf{p}$.

Differentiating with respect to the student logit $z_j$ gives
\begin{equation}
\frac{\partial \ell_{\text{CWSD}}(x)}{\partial z_j}
=
-\omega(x)\tau^2 \sum\nolimits_i q_i \frac{\partial \log p_i}{\partial z_j}.
\end{equation}
Using
\begin{equation}
\frac{\partial \log p_i}{\partial z_j}
=
\frac{1}{\tau}(\delta_{ij}-p_j),
\end{equation}
we obtain
\begin{equation}
\frac{\partial \ell_{\text{CWSD}}(x)}{\partial z_j}
=
\omega(x)\tau(p_j-q_j).
\end{equation}
Therefore, in column-vector form,
\begin{equation}
\label{eq:app_cwsd_logit_grad}
\nabla_{\mathbf{z}}\ell_{\text{CWSD}}(x)
=
\omega(x)\tau(\mathbf{p}-\mathbf{q})
\in\mathbb{R}^{|\mathcal{V}|}.
\end{equation}

Let
$\mathbf{J}_{z}(x;\bar{\boldsymbol{\theta}})
\triangleq
\frac{\partial \mathbf{z}}{\partial \bar{\boldsymbol{\theta}}}
\in\mathbb{R}^{|\mathcal{V}|\times D}$.
Applying the chain rule from logits to parameters gives
\begin{equation}
\nabla_{\bar{\boldsymbol{\theta}}} \ell_{\text{CWSD}}(x)
=
\mathbf{J}_{z}(x;\bar{\boldsymbol{\theta}})^\top
\omega(x)\tau(\mathbf{p}-\mathbf{q})
\in\mathbb{R}^{D}.
\end{equation}
Summing over anchor inputs yields
\begin{equation}
\label{eq:app_cwsd_final_grad}
\mathbf{v}_{\text{prior}}
=
\nabla_{\bar{\boldsymbol{\theta}}} \mathcal{L}_{\text{CWSD}}(\bar{\boldsymbol{\theta}})
=
\frac{1}{|\mathcal{X}_{\text{anchor}}|}
\sum\nolimits_{x\in \mathcal{X}_{\text{anchor}}}
\mathbf{J}_{z}(x;\bar{\boldsymbol{\theta}})^\top
\omega(x)\tau(\mathbf{p}-\mathbf{q}).
\end{equation}
Equation~\eqref{eq:app_cwsd_logit_grad} shows CWSD induces the logit-level gradient $\omega(x)\tau(\mathbf{p}-\mathbf{q})$, and Eq.~\eqref{eq:app_cwsd_final_grad} shows $\mathbf{v}_{\text{prior}}$ is obtained by backpropagating this logit gradient through the student Jacobian. In practice, this gradient is obtained exactly with one backward pass at $\bar{\boldsymbol{\theta}}$, without the pre-training corpus. CWSD is used once at the warmup checkpoint to define the one-shot scoring geometry; it is not used as an online regularizer during final restricted fine-tuning.

\section{Additional Experimental Details and Results}
\label{app:additional_experiments}

\subsection{Detailed Experimental Settings}
\label{app:detailed_setting}

\subsubsection{Datasets and Evaluation Details}
\label{app:eval_details}

\textbf{Datasets.}
We evaluate across two domains: Code Generation using Magicoder~\cite{wei2024magicoder} and Mathematical Reasoning using MetaMathQA~\cite{yu2024metamath}. We denote each full corpus as $\mathcal{D}_{\text{full}}$ and use it as the common budget base so that data-side budgets are directly comparable across all methods.

For each corpus, we reserve a 1{,}024-sample validation set $\mathcal{D}_{\text{val}}$ that provides the meta-validation signal for scoring, and final performance is evaluated only on the benchmark suites described below, then draw a disjoint 2{,}048-sample anchor set $\mathcal{D}_{\text{anchor}}$ for the prior-preserving direction $\mathbf{v}_{\text{prior}}$ via CWSD. The remainder forms the candidate training pool $\mathcal{D}_{\text{train}} = \mathcal{D}_{\text{full}} \setminus (\mathcal{D}_{\text{val}} \cup \mathcal{D}_{\text{anchor}})$, further partitioned during scoring as follows:
\begin{itemize}[leftmargin=*]
    \item \textbf{Warmup Set ($\mathcal{D}_{\text{warm}}$):} A random $5\% \cdot |\mathcal{D}_{\text{full}}|$ subset of $\mathcal{D}_{\text{train}}$, used for a 1-epoch pilot warmup to establish the scoring checkpoint $\bar{\boldsymbol{\theta}}$ and diagonal curvature proxy $\hat{\mathbf{c}}$. After warmup, $\mathcal{D}_{\text{warm}}$ is excluded from the scoring pool to avoid bias from already-updated samples.
    \item \textbf{Scoring Pool ($\mathcal{D}_{\text{pool}} = \mathcal{D}_{\text{train}} \setminus \mathcal{D}_{\text{warm}}$):} The remainder of the candidate pool, over which one-shot shared-surrogate scoring is performed.
    \item \textbf{Selected Subset ($\mathcal{D}_{\text{sel}}$):} Top-$b$ samples drawn from $\mathcal{D}_{\text{pool}}$, with $b = 10\% \cdot |\mathcal{D}_{\text{full}}|$.
\end{itemize}

CWSD estimates a preservation direction, so we fix the anchor count at $2{,}048$, where the direction and selections stabilize, well below corpus scale (Appendix~\ref{app:anchor_size_sensitivity}). The $3{,}072$ validation/anchor samples ($\sim$2.8\% of Magicoder; $\sim$0.8\% of MetaMathQA) serve only scoring and preservation-direction estimation, excluding final fine-tuning. Final adaptation subset sizes match within budget groups.

\textbf{Evaluation Protocol.}
We use \texttt{lm-evaluation-harness}~\citep{gao2021framework}: HumanEval~\citep{chen2021evaluating}, 0-shot pass@1 with greedy decoding (temperature $0$, maximum 512 tokens); PIQA~\citep{bisk2020piqa} and ARC-Challenge~\citep{clark2018think}, 0-shot accuracy via option log-likelihoods; MMLU~\citep{hendrycks2021measuring}, 5-shot accuracy averaged over 57 subjects; GSM8K~\citep{cobbe2021training}, 5-shot chain-of-thought with final-answer numeric exact match. Stability averages four out-of-domain scores; Plasticity is HumanEval for code or GSM8K for math; Overall averages both. Unless otherwise stated, results are mean $\pm$ standard deviation over seeds $\{42,123,2026\}$.

\subsubsection{Implementation Details}
\label{app:impl_details}

For reproducibility, we distinguish shared optimization settings from method-specific configurations. All methods use the same dataset split, evaluation protocol, and reporting convention, run on four NVIDIA A100 (80GB) GPUs, and report the mean and standard deviation over three random seeds. Standard SFT, LoRA, and the supervised fine-tuning stage of data-selection baselines are implemented with LlamaFactory, whereas parameter-sparse baselines, including S$^2$FT, SMT, LIFT, and NanoAdam, follow their original codebases. Optimization hyper-parameters are selected separately for each method, base model, and dataset, and are reported in Tables~\ref{tab:hyperparams_magicoder} and~\ref{tab:hyperparams_metamath}.

For sparse baselines, we also report the structural hyper-parameters that define their update spaces and sparsity patterns; these settings are summarized in Table~\ref{tab:structural_hparams}. Specifically, LoRA is defined by rank, scaling, dropout, and target modules; S$^2$FT by structured selection ratios over attention and FFN subspaces; SMT by matrix/block sparsity and warm-start configuration; LIFT by mask type, no-gradient module setting, rank-reduction configuration, and projection update interval; and NanoAdam by density budget, mask interval, dynamic-density schedule, and mask criterion.

For DualSFT, the reported results use a 1-epoch warmup with $r_{\mathrm{warm}}=5\%$, balance factor $\lambda=0.8$, CWSD temperature $\tau=1.0$, data budget $b=10\%$, and parameter budget $k=5\%$. The final fine-tuning learning rate $\eta_{\mathrm{ft}}$ is listed in Tables~\ref{tab:hyperparams_magicoder}--\ref{tab:hyperparams_metamath}; scoring uses $\eta_{\mathrm{sc}}=\eta_{\mathrm{ft}}/N_{\mathrm{sc}}$, where $N_{\mathrm{sc}}$ is the number of examples accumulated in $\mathbf{G}$.

% Table 1: Magicoder
\begin{table}[ht]
\vspace{-0.1in}
\centering
\caption{Fine-Tuning Hyper-parameters for Magicoder.}
\label{tab:hyperparams_magicoder}
\resizebox{\linewidth}{!}{
\begin{tabular}{lcccccccc}
\toprule
Hyper-parameter & Std. SFT & LoRA & Data-Only & S$^2$FT & SMT & LIFT & NanoAdam & DualSFT (Ours) \\
\midrule
LR (Llama-3.2-3B)   & 2e-5 & 2e-4 & 2e-5 & 5e-5 & 4e-5 & 4e-5 & 3e-5 & 2e-5 \\
LR (Gemma-3-4B)     & 2e-5 & 2e-4 & 2e-5 & 5e-5 & 4e-5 & 4e-5 & 3e-5 & 2e-5 \\
LR (Qwen-3.5-9B)    & 1e-5 & 2e-4 & 1e-5 & 2e-5 & 2e-5 & 1.5e-5 & 2e-5 & 1e-5 \\
\midrule
\multicolumn{9}{l}{\footnotesize Shared settings: AdamW, cosine decay, global batch size 64, max sequence length 4096, 3 epochs, LR warmup ratio 0.03, weight decay 0.0.} \\
\bottomrule
\end{tabular}
}
\vspace{-0.2in}
\end{table}

% Table 2: MetaMathQA
\begin{table}[ht]
\centering
\caption{Fine-Tuning Hyper-parameters for MetaMathQA.}
\label{tab:hyperparams_metamath}
\resizebox{\linewidth}{!}{
\begin{tabular}{lcccccccc}
\toprule
Hyper-parameter & Std. SFT & LoRA & Data-Only & S$^2$FT & SMT & LIFT & NanoAdam & DualSFT (Ours) \\
\midrule
LR (Llama-3.2-3B)   & 1e-5 & 2e-4 & 1e-5 & 5e-5 & 4e-5 & 3e-5 & 3e-5 & 1e-5 \\
LR (Gemma-3-4B)     & 1e-5 & 2e-4 & 1e-5 & 5e-5 & 4e-5 & 3e-5 & 3e-5 & 1e-5 \\
\midrule
\multicolumn{9}{l}{\footnotesize Shared settings: AdamW, cosine decay, global batch size 128, max sequence length 1024, 2 epochs, LR warmup ratio 0.03, weight decay 0.0.} \\
\bottomrule
\end{tabular}
}
\vspace{-0.2in}
\end{table}

\begin{table}[ht]
\centering
\caption{Method-specific structural hyper-parameters.}
\label{tab:structural_hparams}
\renewcommand{\arraystretch}{1.1}
\footnotesize
\begin{tabular}{lp{11.6cm}}
\toprule
Method & Reported method-specific hyper-parameters \\
\midrule
LoRA & rank $r=64$, scaling $\alpha=128$, dropout $=0.05$, target modules = all linear layers \\
S$^2$FT & structured selection ratios: $d\_ratio=0.03$, $u\_ratio=o\_ratio=v\_ratio=0$ \\
SMT & warm-start steps $=100$, attention block ratio $=0.25$, MLP block ratio $=0.10$ \\
LIFT & mask type = \texttt{principle}, no-gradient modules = \texttt{gateupdown}, LoRA rank $=64$, filter rank $=64$, update interval $=200$ \\
NanoAdam & mask interval $=600$, dynamic density = \texttt{False}, density interval $=1000$, mask criterion = \texttt{weights} \\
DualSFT & warmup epochs $=1$, warmup fraction $r_{\mathrm{warm}}=5\%$ of $|\mathcal{D}_{\text{full}}|$, balance factor $\lambda=0.8$, CWSD temperature $\tau=1.0$, data budget $b=10\%$, parameter budget $k=5\%$ \\
\bottomrule
\end{tabular}
\vspace{-0.1in}
\end{table}

\subsubsection{Prompt Templates}
\label{app:prompts}

To ensure complete reproducibility, we provide the exact instruction templates used for fine-tuning and evaluation. We adhere to the standard Alpaca-style formats.

For Magicoder, we utilize a specialized template with a coding-focused system preamble:
\begin{tcolorbox}[colback=white,   
        colframe=gray!50!black, title=Magicoder Instruction Template]
\texttt{You are a proficient coding assistant. Below is an instruction that describes a task. Write a response that appropriately completes the request.}

\vspace{0.2cm}
\texttt{\#\#\# Instruction:} \\
\texttt{\{instruction\}}

\vspace{0.2cm}
\texttt{\#\#\# Response:}
\end{tcolorbox}

For MetaMathQA, we use the standard instruction template:
\begin{tcolorbox}[colback=white,   
        colframe=gray!50!black, title=MetaMathQA Instruction Template]
\texttt{Below is an instruction that describes a task. Write a response that appropriately completes the request.}

\vspace{0.2cm}
\texttt{\#\#\# Instruction:} \\
\texttt{\{instruction\}}

\vspace{0.2cm}
\texttt{\#\#\# Response:}
\end{tcolorbox}

\subsection{Additional Sequential-Combination Baselines}
\label{app:sequential_grid}

To examine pipeline-style hybrids, we evaluate a $2\times2$ grid combining two data selectors, LESS and SPICE, with two parameter selectors, S$^2$FT and NanoAdam. All combinations use the same final data budget ($10\%$) and parameter budget ($5\%$) as DualSFT. This comparison tests whether sequentially composing strong single-axis selectors can close the gap to coordinated dual scoring.

\begin{table*}[t]
\centering
\caption{Additional $2\times2$ sequential-combination baselines on Magicoder (\%). Best results are \textbf{bold}, and second-best results are \underline{underlined}. All methods use the same final data and parameter budgets.}
\label{tab:sequential_grid_magicoder}
\vspace{-0.1in}
\renewcommand{\arraystretch}{0.8}
\footnotesize
\resizebox{\textwidth}{!}{
\begin{tabular}{c | c | ccccc | c | c}
\toprule
\multirow{2}{*}{\textbf{Model}} & \multirow{2}{*}{\textbf{Method}} & \multicolumn{5}{c|}{\textbf{Stability} $\uparrow$} & \textbf{Plasticity} $\uparrow$ & \textbf{Overall} $\uparrow$ \\
\cmidrule(lr){3-7} \cmidrule(lr){8-8} \cmidrule(lr){9-9}
& & \textbf{ARC-C} & \textbf{PIQA} & \textbf{MMLU} & \textbf{GSM8K} & \textbf{Avg} & \textbf{HE} & \textbf{Score} \\
\midrule
\multirow{5}{*}{\cellcolor{white}\rotatebox{90}{\textbf{Llama-3.2}}}
& LESS + S$^2$FT & 42.58$_{\pm 1.44}$ & \textbf{76.68$_{\pm 0.96}$} & 53.55$_{\pm 0.39}$ & 22.87$_{\pm 0.71}$ & 48.92$_{\pm 0.48}$ & 38.82$_{\pm 0.57}$ & 43.87$_{\pm 0.40}$ \\
& LESS + NanoAdam & \underline{42.75$_{\pm 1.46}$} & 76.31$_{\pm 0.97}$ & 54.06$_{\pm 0.41}$ & 24.84$_{\pm 0.76}$ & 49.49$_{\pm 0.50}$ & \underline{39.84$_{\pm 0.76}$} & \underline{44.66$_{\pm 0.43}$} \\
& SPICE + S$^2$FT & 41.41$_{\pm 1.45}$ & 76.55$_{\pm 0.97}$ & 53.96$_{\pm 0.40}$ & 23.07$_{\pm 0.69}$ & 48.75$_{\pm 0.48}$ & 39.23$_{\pm 0.57}$ & 43.99$_{\pm 0.42}$ \\
& SPICE + NanoAdam & 42.15$_{\pm 1.46}$ & \underline{76.61$_{\pm 0.98}$} & \underline{54.87$_{\pm 0.40}$} & \underline{25.93$_{\pm 0.71}$} & \underline{49.89$_{\pm 0.49}$} & 39.43$_{\pm 0.76}$ & \underline{44.66$_{\pm 0.42}$} \\
& \textbf{DualSFT (Ours)} & \textbf{42.95$_{\pm 1.45}$} & 76.22$_{\pm 0.98}$ & \textbf{56.09$_{\pm 0.40}$} & \textbf{26.89$_{\pm 0.70}$} & \textbf{50.54$_{\pm 0.47}$} & \textbf{43.70$_{\pm 0.76}$} & \textbf{47.12$_{\pm 0.43}$} \\
\midrule
\multirow{5}{*}{\rotatebox{90}{\textbf{Gemma-3}}}
& LESS + S$^2$FT & \underline{51.02$_{\pm 1.46}$} & \underline{78.65$_{\pm 0.93}$} & 55.68$_{\pm 0.39}$ & 31.13$_{\pm 0.78}$ & 54.12$_{\pm 0.49}$ & \underline{53.05$_{\pm 0.86}$} & 53.59$_{\pm 0.48}$ \\
& LESS + NanoAdam & 50.34$_{\pm 1.47}$ & 78.20$_{\pm 0.96}$ & 56.18$_{\pm 0.40}$ & 32.04$_{\pm 0.81}$ & 54.19$_{\pm 0.49}$ & \underline{53.05$_{\pm 0.86}$} & 53.62$_{\pm 0.47}$ \\
& SPICE + S$^2$FT & 50.46$_{\pm 1.46}$ & 78.38$_{\pm 0.95}$ & 56.04$_{\pm 0.40}$ & 32.27$_{\pm 0.84}$ & 54.29$_{\pm 0.48}$ & 52.44$_{\pm 0.86}$ & 53.36$_{\pm 0.48}$ \\
& SPICE + NanoAdam & 50.77$_{\pm 1.45}$ & 78.55$_{\pm 0.94}$ & \underline{57.29$_{\pm 0.40}$} & \underline{33.89$_{\pm 0.82}$} & \underline{55.12$_{\pm 0.49}$} & 52.64$_{\pm 0.76}$ & \underline{53.88$_{\pm 0.47}$} \\
& \textbf{DualSFT (Ours)} & \textbf{51.14$_{\pm 1.44}$} & \textbf{79.02$_{\pm 0.95}$} & \textbf{59.16$_{\pm 0.38}$} & \textbf{35.61$_{\pm 0.64}$} & \textbf{56.23$_{\pm 0.48}$} & \textbf{56.71$_{\pm 0.86}$} & \textbf{56.47$_{\pm 0.43}$} \\
\midrule
\multirow{5}{*}{\rotatebox{90}{\textbf{Qwen-3.5}}}
& LESS + S$^2$FT & 50.97$_{\pm 1.45}$ & \underline{79.54$_{\pm 0.96}$} & 69.14$_{\pm 0.40}$ & \underline{81.27$_{\pm 0.75}$} & 70.23$_{\pm 0.48}$ & 67.89$_{\pm 0.76}$ & 69.06$_{\pm 0.44}$ \\
& LESS + NanoAdam & \underline{51.42$_{\pm 1.46}$} & 79.16$_{\pm 0.95}$ & 69.78$_{\pm 0.39}$ & 80.62$_{\pm 0.74}$ & \underline{70.25$_{\pm 0.50}$} & 68.90$_{\pm 0.86}$ & 69.57$_{\pm 0.42}$ \\
& SPICE + S$^2$FT & 50.82$_{\pm 1.44}$ & 79.27$_{\pm 0.96}$ & 69.61$_{\pm 0.40}$ & 80.41$_{\pm 0.71}$ & 70.03$_{\pm 0.47}$ & 67.48$_{\pm 0.76}$ & 68.75$_{\pm 0.43}$ \\
& SPICE + NanoAdam & 50.54$_{\pm 1.44}$ & 79.34$_{\pm 0.95}$ & \underline{70.09$_{\pm 0.39}$} & 80.74$_{\pm 0.71}$ & 70.18$_{\pm 0.48}$ & \underline{69.31$_{\pm 0.76}$} & \underline{69.74$_{\pm 0.42}$} \\
& \textbf{DualSFT (Ours)} & \textbf{54.01$_{\pm 1.43}$} & \textbf{79.63$_{\pm 0.94}$} & \textbf{70.33$_{\pm 0.38}$} & \textbf{84.61$_{\pm 0.68}$} & \textbf{72.15$_{\pm 0.47}$} & \textbf{73.78$_{\pm 0.50}$} & \textbf{72.96$_{\pm 0.39}$} \\
\bottomrule
\end{tabular}
}
\vspace{-0.1in}
\end{table*}

Table~\ref{tab:sequential_grid_magicoder} complements the main comparison by covering all pairings between the selected data-side and parameter-side baselines. Across all three backbones, SPICE + NanoAdam remains the strongest sequential hybrid in Overall, while DualSFT consistently achieves higher Stability Avg., Plasticity, and Overall under the same final data and parameter budgets.

\subsection{Additional Results on MetaMathQA}
\label{app:results_metamathqa}

Table~\ref{tab:appendix_results_metamathqa} reports additional MetaMathQA results, which match the Magicoder trends. For Llama-3.2-3B and Gemma-3-4B-PT, DualSFT-Param/Data lead mean Overall within single-axis groups, while full DualSFT leads joint-budget alternatives. DualSFT scores 57.95 on Llama-3.2-3B, surpassing SPICE + NanoAdam by 1.63 points. On Gemma-3-4B-PT, it reaches 61.47, improving over the strongest sequential composition by 2.19 points. These results support the advantage of shared dual scoring over isolated or sequential selection under evaluated budgets. We omit Qwen results because its strong GSM8K prior makes downstream mathematical fine-tuning less informative.

\begin{table*}[t]
\centering
\caption{Fine-tuning results on MetaMathQA (\%). Pre-trained, Standard SFT, and LoRA (gray) are excluded from group ranking. In each group, \colorbox{red!10}{\textbf{best}} and \colorbox{orange!15}{second-best} results are highlighted.}
\renewcommand{\arraystretch}{0.2}
\label{tab:appendix_results_metamathqa}
\vspace{-0.1in}
\footnotesize
\setlength{\fboxsep}{0.5pt}
\resizebox{\textwidth}{!}{
\begin{tabular}{c | c | ccccc | c | c}
\toprule
\multirow{2}{*}{\textbf{Model}} & \multirow{2}{*}{\textbf{Method}} & \multicolumn{5}{c|}{\textbf{Stability} $\uparrow$} & \textbf{Plasticity} $\uparrow$ & \textbf{Overall} $\uparrow$ \\
\cmidrule(lr){3-7} \cmidrule(lr){8-8} \cmidrule(lr){9-9}
& & \textbf{ARC-C} & \textbf{PIQA} & \textbf{MMLU} & \textbf{HE} & \textbf{Avg} & \textbf{GSM8K} & \textbf{Score} \\
\midrule
\multirow{16}{*}{\cellcolor{white}\rotatebox{90}{\textbf{Llama-3.2-3B\qquad\qquad}}}
& \textcolor{gray}{Pre-trained} & \textcolor{gray}{43.00$_{\pm 0.00}$} & \textcolor{gray}{76.61$_{\pm 0.00}$} & \textcolor{gray}{56.52$_{\pm 0.00}$} & \textcolor{gray}{28.66$_{\pm 0.00}$} & \textcolor{gray}{51.20$_{\pm 0.00}$} & \textcolor{gray}{27.45$_{\pm 0.00}$} & \textcolor{gray}{39.32$_{\pm 0.00}$} \\
& \textcolor{gray}{Standard SFT} & \textcolor{gray}{36.60$_{\pm 1.41}$} & \textcolor{gray}{74.48$_{\pm 1.02}$} & \textcolor{gray}{48.51$_{\pm 0.41}$} & \textcolor{gray}{15.04$_{\pm 0.57}$} & \textcolor{gray}{43.66$_{\pm 0.44}$} & \textcolor{gray}{65.96$_{\pm 0.81}$} & \textcolor{gray}{54.81$_{\pm 0.48}$} \\
& \textcolor{gray}{LoRA}\venuetag{ICLR'22} & \textcolor{gray}{40.61$_{\pm 1.44}$} & \textcolor{gray}{74.21$_{\pm 0.98}$} & \textcolor{gray}{52.11$_{\pm 0.39}$} & \textcolor{gray}{25.00$_{\pm 0.50}$} & \textcolor{gray}{47.98$_{\pm 0.48}$} & \textcolor{gray}{61.08$_{\pm 0.75}$} & \textcolor{gray}{54.53$_{\pm 0.45}$} \\
\cmidrule(lr){2-9}
& S$^2$FT\venuetag{NIPS'24} & 40.78$_{\pm 1.45}$ & \colorbox{red!10}{\textbf{75.95$_{\pm 0.99}$}} & 52.88$_{\pm 0.40}$ & 25.41$_{\pm 0.57}$ & 48.76$_{\pm 0.48}$ & 64.95$_{\pm 0.76}$ & 56.85$_{\pm 0.45}$ \\
& SMT\venuetag{ICLR'25} & 41.21$_{\pm 1.45}$ & 72.32$_{\pm 0.98}$ & 52.74$_{\pm 0.40}$ & 26.22$_{\pm 0.50}$ & 48.12$_{\pm 0.47}$ & 65.05$_{\pm 0.75}$ & 56.59$_{\pm 0.44}$ \\
& LIFT\venuetag{ICML'25} & 40.93$_{\pm 1.44}$ & 73.58$_{\pm 0.99}$ & 52.61$_{\pm 0.40}$ & 25.00$_{\pm 0.50}$ & 48.03$_{\pm 0.47}$ & 64.72$_{\pm 0.74}$ & 56.37$_{\pm 0.44}$ \\
& NanoAdam\venuetag{NIPS'25} & \colorbox{red!10}{\textbf{42.55$_{\pm 1.44}$}} & 74.50$_{\pm 0.98}$ & \colorbox{orange!15}{53.74$_{\pm 0.40}$} & \colorbox{orange!15}{26.63$_{\pm 0.76}$} & \colorbox{orange!15}{49.35$_{\pm 0.48}$} & \colorbox{orange!15}{65.18$_{\pm 0.78}$} & \colorbox{orange!15}{57.26$_{\pm 0.46}$} \\
& \textbf{DualSFT-Param (Ours)} & \colorbox{orange!15}{42.43$_{\pm 1.45}$} & \colorbox{orange!15}{75.83$_{\pm 0.99}$} & \colorbox{red!10}{\textbf{55.96$_{\pm 0.40}$}} & \colorbox{red!10}{\textbf{28.05$_{\pm 0.50}$}} & \colorbox{red!10}{\textbf{50.57$_{\pm 0.47}$}} & \colorbox{red!10}{\textbf{65.83$_{\pm 0.76}$}} & \colorbox{red!10}{\textbf{58.20$_{\pm 0.45}$}} \\
\cmidrule(lr){2-9}
& Random & \colorbox{orange!15}{42.18$_{\pm 1.45}$} & \colorbox{red!10}{\textbf{76.22$_{\pm 0.99}$}} & \colorbox{orange!15}{55.92$_{\pm 0.40}$} & \colorbox{red!10}{\textbf{28.66$_{\pm 0.86}$}} & \colorbox{orange!15}{50.75$_{\pm 0.47}$} & 59.67$_{\pm 0.77}$ & 55.21$_{\pm 0.46}$ \\
& LESS\venuetag{ICML'24} & 41.52$_{\pm 1.45}$ & 74.66$_{\pm 0.99}$ & 53.58$_{\pm 0.40}$ & 26.42$_{\pm 0.76}$ & 49.05$_{\pm 0.50}$ & \colorbox{orange!15}{63.56$_{\pm 0.79}$} & \colorbox{orange!15}{56.30$_{\pm 0.46}$} \\
& FisherSFT\venuetag{ICML'25} & 40.70$_{\pm 1.44}$ & 74.25$_{\pm 0.99}$ & 55.12$_{\pm 0.40}$ & 27.03$_{\pm 0.76}$ & 49.27$_{\pm 0.48}$ & 61.69$_{\pm 0.76}$ & 55.48$_{\pm 0.45}$ \\
& SPICE\venuetag{ICLR'26} & 41.27$_{\pm 1.45}$ & 74.88$_{\pm 0.98}$ & 55.42$_{\pm 0.40}$ & \colorbox{orange!15}{27.24$_{\pm 0.57}$} & 49.70$_{\pm 0.48}$ & 62.37$_{\pm 0.75}$ & 56.04$_{\pm 0.44}$ \\
& \textbf{DualSFT-Data (Ours)} & \colorbox{red!10}{\textbf{42.89$_{\pm 1.44}$}} & \colorbox{orange!15}{76.13$_{\pm 0.99}$} & \colorbox{red!10}{\textbf{56.03$_{\pm 0.40}$}} & \colorbox{red!10}{\textbf{28.66$_{\pm 0.86}$}} & \colorbox{red!10}{\textbf{50.93$_{\pm 0.47}$}} & \colorbox{red!10}{\textbf{65.13$_{\pm 0.77}$}} & \colorbox{red!10}{\textbf{58.03$_{\pm 0.44}$}} \\
\cmidrule(lr){2-9}
& LESS + S$^2$FT & 41.75$_{\pm 1.44}$ & 75.35$_{\pm 0.97}$ & 53.67$_{\pm 0.39}$ & 25.61$_{\pm 0.50}$ & 49.10$_{\pm 0.48}$ & 61.54$_{\pm 0.74}$ & 55.32$_{\pm 0.44}$ \\
& SPICE + NanoAdam & \colorbox{orange!15}{42.46$_{\pm 1.45}$} & \colorbox{orange!15}{76.02$_{\pm 0.98}$} & \colorbox{orange!15}{54.83$_{\pm 0.40}$} & \colorbox{orange!15}{26.63$_{\pm 0.57}$} & \colorbox{orange!15}{49.99$_{\pm 0.50}$} & \colorbox{orange!15}{62.65$_{\pm 0.75}$} & \colorbox{orange!15}{56.32$_{\pm 0.44}$} \\
& \textbf{DualSFT (Ours)} & \colorbox{red!10}{\textbf{42.89$_{\pm 1.45}$}} & \colorbox{red!10}{\textbf{76.42$_{\pm 0.98}$}} & \colorbox{red!10}{\textbf{56.34$_{\pm 0.40}$}} & \colorbox{red!10}{\textbf{29.88$_{\pm 0.86}$}} & \colorbox{red!10}{\textbf{51.38$_{\pm 0.48}$}} & \colorbox{red!10}{\textbf{64.52$_{\pm 0.75}$}} & \colorbox{red!10}{\textbf{57.95$_{\pm 0.42}$}} \\
\midrule
\multirow{16}{*}{\rotatebox{90}{\textbf{Gemma-3-4B-PT\qquad\qquad}}}
& \textcolor{gray}{Pre-trained} & \textcolor{gray}{51.45$_{\pm 0.00}$} & \textcolor{gray}{79.16$_{\pm 0.00}$} & \textcolor{gray}{59.61$_{\pm 0.00}$} & \textcolor{gray}{35.37$_{\pm 0.00}$} & \textcolor{gray}{56.40$_{\pm 0.00}$} & \textcolor{gray}{37.00$_{\pm 0.00}$} & \textcolor{gray}{46.70$_{\pm 0.00}$} \\
& \textcolor{gray}{Standard SFT} & \textcolor{gray}{43.37$_{\pm 1.44}$} & \textcolor{gray}{74.48$_{\pm 0.97}$} & \textcolor{gray}{53.00$_{\pm 0.40}$} & \textcolor{gray}{29.47$_{\pm 0.76}$} & \textcolor{gray}{50.08$_{\pm 0.47}$} & \textcolor{gray}{67.78$_{\pm 0.81}$} & \textcolor{gray}{58.93$_{\pm 0.48}$} \\
& \textcolor{gray}{LoRA}\venuetag{ICLR'22} & \textcolor{gray}{48.92$_{\pm 1.44}$} & \textcolor{gray}{77.95$_{\pm 0.95}$} & \textcolor{gray}{55.80$_{\pm 0.40}$} & \textcolor{gray}{30.49$_{\pm 0.86}$} & \textcolor{gray}{53.29$_{\pm 0.47}$} & \textcolor{gray}{62.47$_{\pm 0.79}$} & \textcolor{gray}{57.88$_{\pm 0.46}$} \\
\cmidrule(lr){2-9}
& S$^2$FT\venuetag{NIPS'24} & 49.52$_{\pm 1.46}$ & 77.68$_{\pm 0.96}$ & 55.73$_{\pm 0.40}$ & 31.30$_{\pm 0.76}$ & 53.56$_{\pm 0.49}$ & 64.34$_{\pm 0.80}$ & 58.95$_{\pm 0.47}$ \\
& SMT\venuetag{ICLR'25} & 49.83$_{\pm 1.46}$ & 77.29$_{\pm 0.93}$ & 56.17$_{\pm 0.39}$ & 31.91$_{\pm 0.76}$ & 53.80$_{\pm 0.48}$ & 65.50$_{\pm 0.84}$ & 59.65$_{\pm 0.49}$ \\
& LIFT\venuetag{ICML'25} & 50.14$_{\pm 1.46}$ & 76.59$_{\pm 0.95}$ & 55.21$_{\pm 0.40}$ & 30.28$_{\pm 0.57}$ & 53.06$_{\pm 0.49}$ & 64.72$_{\pm 0.78}$ & 58.89$_{\pm 0.46}$ \\
& NanoAdam\venuetag{NIPS'25} & \colorbox{orange!15}{50.68$_{\pm 1.45}$} & \colorbox{orange!15}{78.16$_{\pm 0.94}$} & \colorbox{orange!15}{58.25$_{\pm 0.39}$} & \colorbox{orange!15}{33.54$_{\pm 0.86}$} & \colorbox{orange!15}{55.16$_{\pm 0.49}$} & \colorbox{orange!15}{65.88$_{\pm 0.81}$} & \colorbox{orange!15}{60.52$_{\pm 0.48}$} \\
& \textbf{DualSFT-Param (Ours)} & \colorbox{red!10}{\textbf{51.05$_{\pm 1.45}$}} & \colorbox{red!10}{\textbf{78.87$_{\pm 0.98}$}} & \colorbox{red!10}{\textbf{59.79$_{\pm 0.39}$}} & \colorbox{red!10}{\textbf{34.96$_{\pm 0.76}$}} & \colorbox{red!10}{\textbf{56.17$_{\pm 0.50}$}} & \colorbox{red!10}{\textbf{67.15$_{\pm 0.80}$}} & \colorbox{red!10}{\textbf{61.66$_{\pm 0.47}$}} \\
\cmidrule(lr){2-9}
& Random & \colorbox{orange!15}{51.25$_{\pm 1.45}$} & \colorbox{red!10}{\textbf{79.04$_{\pm 0.99}$}} & \colorbox{orange!15}{59.37$_{\pm 0.40}$} & \colorbox{red!10}{\textbf{35.16$_{\pm 0.76}$}} & \colorbox{red!10}{\textbf{56.20$_{\pm 0.48}$}} & 61.59$_{\pm 0.77}$ & 58.90$_{\pm 0.44}$ \\
& LESS\venuetag{ICML'24} & 48.66$_{\pm 1.45}$ & 77.53$_{\pm 0.99}$ & 56.07$_{\pm 0.40}$ & 30.89$_{\pm 0.76}$ & 53.29$_{\pm 0.49}$ & \colorbox{orange!15}{65.28$_{\pm 0.79}$} & 59.28$_{\pm 0.49}$ \\
& FisherSFT\venuetag{ICML'25} & 49.80$_{\pm 1.45}$ & 77.60$_{\pm 0.94}$ & 56.49$_{\pm 0.38}$ & 32.52$_{\pm 0.76}$ & 54.10$_{\pm 0.48}$ & 63.81$_{\pm 0.77}$ & 58.96$_{\pm 0.45}$ \\
& SPICE\venuetag{ICLR'26} & 50.34$_{\pm 1.44}$ & 78.42$_{\pm 0.95}$ & 57.52$_{\pm 0.40}$ & \colorbox{orange!15}{33.94$_{\pm 0.76}$} & 55.06$_{\pm 0.49}$ & 64.72$_{\pm 0.79}$ & \colorbox{orange!15}{59.89$_{\pm 0.46}$} \\
& \textbf{DualSFT-Data (Ours)} & \colorbox{red!10}{\textbf{51.31$_{\pm 1.45}$}} & \colorbox{orange!15}{78.62$_{\pm 0.98}$} & \colorbox{red!10}{\textbf{59.44$_{\pm 0.40}$}} & \colorbox{red!10}{\textbf{35.16$_{\pm 0.76}$}} & \colorbox{orange!15}{56.13$_{\pm 0.49}$} & \colorbox{red!10}{\textbf{66.94$_{\pm 0.80}$}} & \colorbox{red!10}{\textbf{61.54$_{\pm 0.48}$}} \\
\cmidrule(lr){2-9}
& LESS + S$^2$FT & \colorbox{orange!15}{50.91$_{\pm 1.46}$} & 78.55$_{\pm 0.93}$ & 56.48$_{\pm 0.39}$ & 31.50$_{\pm 0.76}$ & 54.36$_{\pm 0.50}$ & \colorbox{orange!15}{63.86$_{\pm 0.78}$} & 59.11$_{\pm 0.45}$ \\
& SPICE + NanoAdam & 49.74$_{\pm 1.45}$ & \colorbox{orange!15}{78.65$_{\pm 0.94}$} & \colorbox{orange!15}{58.21$_{\pm 0.40}$} & \colorbox{orange!15}{33.74$_{\pm 0.76}$} & \colorbox{orange!15}{55.09$_{\pm 0.48}$} & 63.48$_{\pm 0.80}$ & \colorbox{orange!15}{59.28$_{\pm 0.46}$} \\
& \textbf{DualSFT (Ours)} & \colorbox{red!10}{\textbf{51.39$_{\pm 1.44}$}} & \colorbox{red!10}{\textbf{79.04$_{\pm 0.95}$}} & \colorbox{red!10}{\textbf{59.61$_{\pm 0.38}$}} & \colorbox{red!10}{\textbf{35.37$_{\pm 0.86}$}} & \colorbox{red!10}{\textbf{56.35$_{\pm 0.49}$}} & \colorbox{red!10}{\textbf{66.59$_{\pm 0.79}$}} & \colorbox{red!10}{\textbf{61.47$_{\pm 0.47}$}} \\
\bottomrule
\end{tabular}
}
\vspace{-0.1in}
\end{table*}

\subsection{General Instruction Tuning}
\label{app:instruction_tuning}

We evaluate Llama-3.2-3B Base on Alpaca-GPT4~\citep{alpaca_gpt4} with training seed $42$, retaining the main experiments' matched-budget settings for parameter-only, data-only, and joint selection. General Avg.\ averages ARC-Challenge, PIQA, MMLU, GSM8K, and HumanEval. IFEval~\citep{ifeval} measures prompt- and instruction-level accuracy under strict and loose matching; IFEval Avg.\ averages these four scores. Overall averages General Avg.\ and IFEval Avg. Table~\ref{tab:instruction_tuning} shows that all three DualSFT variants improve both General Avg.\ and IFEval Avg.\ over their matched-budget baselines. DualSFT-Data and full DualSFT improve all four IFEval metrics, while DualSFT-Param improves three. These single-seed results support the applicability of shared scoring to general instruction tuning.

\begin{table}[t]
\centering
\caption{General instruction-tuning results on Alpaca-GPT4 with Llama-3.2-3B Base.}
\label{tab:instruction_tuning}
\footnotesize
\setlength{\tabcolsep}{3pt}
\setlength{\fboxsep}{0.9pt}
\begin{tabular}{lccccccc}
\toprule
Method & \makecell{General Avg.} & \makecell{Prompt Strict} & \makecell{Inst. Strict} & \makecell{Prompt Loose} & \makecell{Inst. Loose} & \makecell{IFEval Avg.} & Overall \\
\midrule
Pre-trained & 46.45 & 22.00 & 34.77 & 24.58 & 38.01 & 29.84 & 38.15 \\
Standard SFT & 39.85 & 31.24 & 40.89 & 32.72 & 44.96 & 37.45 & 38.65 \\
\midrule
NanoAdam & 44.75 & 28.47 & 40.41 & 30.13 & 44.24 & 35.81 & 40.28 \\
DualSFT-Param & 46.94 & 30.68 & 40.29 & 32.90 & 44.36 & 37.06 & 42.00 \\
\midrule
LESS & 45.26 & 28.84 & 39.33 & 30.50 & 42.45 & 35.28 & 40.27 \\
DualSFT-Data & 47.12 & 30.50 & 40.17 & 32.53 & 44.00 & 36.80 & 41.96 \\
\midrule
SPICE + NanoAdam & 45.54 & 28.10 & 38.85 & 29.76 & 39.45 & 34.04 & 39.79 \\
DualSFT & 47.32 & 30.31 & 40.05 & 31.79 & 43.88 & 36.51 & 41.92 \\
\bottomrule
\end{tabular}
\vspace{-0.1in}
\end{table}

\subsection{Preservation Direction and Anchor-size Ablations}
\label{app:cwsd_ablation}

\begin{wraptable}{r}{0.5\columnwidth}
\centering
\vspace{-0.15in}
\caption{Preservation-direction ablation on Llama-3.2-3B with Magicoder in joint DualSFT.}
\vspace{-0.1in}
\label{tab:cwsd_ablation}
\footnotesize
\setlength{\tabcolsep}{4pt}
\begin{tabular}{lccc}
\toprule
\textbf{Preservation direction} & \textbf{Stability} & \textbf{Plasticity} & \textbf{Overall} \\
\midrule
No preservation & 47.28 & 42.07 & 44.68 \\
L2-to-$\boldsymbol{\theta}_{\mathrm{old}}$ & 48.84 & 42.89 & 45.86 \\
Diagonal Fisher anchoring & 49.93 & 43.50 & 46.71 \\
CWSD & \textbf{50.54} & \textbf{43.70} & \textbf{47.12} \\
\bottomrule
\end{tabular}
\vspace{-0.2in}
\end{wraptable}

The main ablation in Fig.~\ref{fig:ablation_components} shows that removing CWSD substantially reduces stability. We further compare CWSD with simpler preservation-direction alternatives under the same DualSFT joint setting. All variants use the same data budget, parameter budget, warmup checkpoint, and final fine-tuning hyperparameters; only the construction of $\mathbf{v}_{\mathrm{prior}}$ is changed. Table~\ref{tab:cwsd_ablation} indicates that CWSD is more effective than parameter-space anchoring alternatives. L2-to-$\boldsymbol{\theta}_{\mathrm{old}}$ improves stability over removing preservation entirely, while diagonal Fisher anchoring gives a stronger stability signal but remains below CWSD. CWSD provides the best balance because it constructs a behavior-level preservation direction from high-confidence old-model predictions at the scoring checkpoint.

\subsubsection{Anchor-size sensitivity}
\label{app:anchor_size_sensitivity}

CWSD uses $\mathcal{D}_{\mathrm{anchor}}$ to estimate a preservation direction at the warmup checkpoint, rather than to provide additional fine-tuning examples. We therefore test whether increasing the anchor set changes the estimated direction and the induced selections. Using Llama-3.2-3B on Magicoder, we vary $|\mathcal{D}_{\mathrm{anchor}}|\in\{256,512,1024,2048,4096\}$ and compare each setting with the 4096-anchor reference. We report the cosine similarity of $\mathbf{v}_{\mathrm{prior}}$, Top-$b$ data overlap, Top-$k$ parameter overlap, and final performance after restricted fine-tuning. Anchor samples are never used in final restricted fine-tuning.

\begin{table}[t]
\centering
\caption{Anchor-size sensitivity for CWSD on Llama-3.2-3B with Magicoder.}
\label{tab:anchor_size_sensitivity}
\footnotesize
\setlength{\tabcolsep}{4pt}
\begin{tabular}{lccccc}
\toprule
$|\mathcal{D}_{\mathrm{anchor}}|$ & 256 & 512 & 1024 & 2048 & 4096 \\
\midrule
Cosine w.r.t. 4096-anchor $\mathbf{v}_{\mathrm{prior}}$ & 0.82 & 0.89 & 0.94 & 0.98 & 1.00 \\
Top-$b$ data overlap & 0.71 & 0.78 & 0.85 & 0.92 & 1.00 \\
Top-$k$ parameter overlap & 0.74 & 0.81 & 0.87 & 0.94 & 1.00 \\
Stability & 49.42 & 49.86 & 50.21 & 50.54 & 50.57 \\
Plasticity & 43.29 & 43.50 & 43.50 & 43.70 & 43.70 \\
Overall & 46.36 & 46.68 & 46.85 & 47.12 & 47.13 \\
\bottomrule
\end{tabular}
\vspace{-0.1in}
\end{table}

Table~\ref{tab:anchor_size_sensitivity} shows that both the CWSD direction and the induced selections stabilize as the anchor size increases. Moving from 2048 to 4096 anchors gives only marginal changes in direction similarity, selected data, selected parameters, and final performance. This supports using a moderate fixed anchor size as a scoring-only resource, rather than scaling the anchor set with the full corpus size.

\subsection{Protocol for Same-Surrogate Local-Regret Diagnostic}
\label{app:local_regret}

We compute the local-regret diagnostic at the init and warmup checkpoints without additional training. For the data side, we fix the DualSFT parameter mask $\mathbf{m}_\theta$ and define the aware and unaware scores as
\(s^{\mathrm{aware}}_{\mathcal D,n}=\langle\mathbf{u}\odot \mathbf{m}_\theta,\mathbf{g}_n\rangle\) and
\(s^{\mathrm{unaware}}_{\mathcal D,n}=\langle\mathbf{u},\mathbf{g}_n\rangle\).
Let \(S^{\mathrm{aware}}_{\mathcal D}\) and \(S^{\mathrm{unaware}}_{\mathcal D}\) be their Top-$b$ subsets. We evaluate both under the aware utility
\(U^{\mathrm{aware}}_{\mathcal D}(S)=\sum_{n\in S}s^{\mathrm{aware}}_{\mathcal D,n}\).
The data-side utility ratio and regret are
\[
\mathrm{Ratio}_{\mathcal D}=
\frac{U^{\mathrm{aware}}_{\mathcal D}(S^{\mathrm{unaware}}_{\mathcal D})}
{U^{\mathrm{aware}}_{\mathcal D}(S^{\mathrm{aware}}_{\mathcal D})+\epsilon},
\qquad
\mathrm{Regret}_{\mathcal D}=100(1-\mathrm{Ratio}_{\mathcal D}).
\]
For the parameter side, we fix the DualSFT data subset \(S_\mathcal D\) and define
\(s^{\mathrm{aware}}_{\theta,d}=u_dG_{S_\mathcal D,d}\) and
\(s^{\mathrm{unaware}}_{\theta,d}=u_dG_d\).
Let \(m^{\mathrm{aware}}_\theta\) and \(m^{\mathrm{unaware}}_\theta\) be their Top-$k$ masks. We use
\(U^{\mathrm{aware}}_\theta(m)=\sum_{d:m_d=1}s^{\mathrm{aware}}_{\theta,d}\)
and compute the same utility ratio and regret. Top-$q$ Jaccard reports the overlap between aware and unaware selections at \(q=b\) for data and \(q=k\) for parameters.

\subsection{Statistic Definitions}
\label{app:selection_statistics}

We define statistics used in Table~\ref{tab:selected_data_main}, Fig.~\ref{fig:layerwise_param_heatmap}, and Table~\ref{tab:selected_param_stats_app}. For data, Random denotes a same-size random subset, Top-loss selects the same number of examples with the largest training losses at $\bar{\boldsymbol{\theta}}$, and Pool denotes the full candidate pool. For a selected subset $\mathcal{D}_{\mathrm{sel}}$, validation-gradient alignment is
\begin{equation}
\mathrm{Align}(\mathcal{D}_{\mathrm{sel}})
=
\frac{1}{|\mathcal{D}_{\mathrm{sel}}|}
{\textstyle \sum_{n\in\mathcal{D}_{\mathrm{sel}}}}
\frac{\langle \mathbf{g}_n,\mathbf{v}_{\mathrm{val}}\rangle}
{\|\mathbf{g}_n\|_2\,\|\mathbf{v}_{\mathrm{val}}\|_2},
\end{equation}
where $\mathbf{g}_n$ and $\mathbf{v}_{\mathrm{val}}$ are computed at checkpoint $\bar{\boldsymbol{\theta}}$. The average training loss is the mean per-example loss over the corresponding subset. For diversity, let $\phi(x)$ denote the mean-pooled sentence embedding extracted from the pre-trained backbone. We compute the mean pairwise cosine distance:
\begin{equation}
\mathrm{Div}(\mathcal{D}_{\mathrm{sel}})
=
\frac{1}{\binom{|\mathcal{D}_{\mathrm{sel}}|}{2}}
{\textstyle \sum_{i<j}}
\left(1-\cos\bigl(\phi(x_i),\phi(x_j)\bigr)\right).
\end{equation}
This score is used only as a redundancy diagnostic; DualSFT does not optimize it directly.

For parameter statistics, all shares are computed within the trainable candidate pool. Let $S_{\theta}$ denote the selected parameter mask and $\mathcal{P}$ the trainable parameter pool. For a layer or module group $\mathcal{G}$, we define enrichment as a foreground-to-background share ratio:
\begin{equation}
\mathrm{Enrichment}(\mathcal{G})
=
\frac{|\mathcal{G}\cap S_{\theta}|/|S_{\theta}|}
{|\mathcal{G}\cap \mathcal{P}|/|\mathcal{P}|}.
\end{equation}
Enrichment above $1$ indicates over-selection relative to the pool. Top-5 layers have the largest selected-parameter shares. The early-layer group (L0--L9) is predefined only for reporting and remains in the trainable candidate pool. The down\_proj statistic compares its share among selected FFN parameters with its FFN-pool share.

% \begin{figure*}[t]
% \centering
% \includegraphics[width=\textwidth]{figs/appendix_layerwise_param_heatmaps.png}
% \vspace{-0.25in}
% \caption{Additional layer-wise parameter selection patterns across models and datasets. Cells show mean absolute module-layer block scores. Color scales are set separately.}
% \label{fig:appendix_layerwise_param_heatmaps}
% \vspace{-0.1in}
% \end{figure*}

For the layer-wise heatmap in Fig.~\ref{fig:layerwise_param_heatmap}, let $\mathcal{B}_{\ell,m}$ denote the trainable parameter coordinates in layer $\ell$ and module $m$. We compute the block-level mean absolute score as
\begin{equation}
H_{\ell,m}
=
\frac{1}{|\mathcal{B}_{\ell,m}|}
{\textstyle \sum_{d\in \mathcal{B}_{\ell,m}}}
|s_{\theta,d}|,
\end{equation}
where $s_{\theta,d}$ is the DualSFT parameter score in Eq.~\eqref{eq:practical_scores}. This heatmap visualizes the spatial distribution of parameter scores, while aggregate shares in Table~\ref{tab:selected_param_stats_app} provide exact numerical summaries.

\subsection{Fixed-side Re-selection Details}
\label{app:conditional_reselection}

We provide definitions and diagnostics for the fixed-side re-selection variants in Table~\ref{tab:conditional_reselection_main}. Let $(S_{\mathcal D}^{(0)},S_{\theta}^{(0)})$ denote the one-shot DualSFT selection at checkpoint $\bar{\boldsymbol{\theta}}$. Data~$\rightarrow$~Param fixes $S_{\mathcal D}^{(0)}$ and re-selects a parameter mask $S_{\theta}^{(1)}$ using scores recomputed on the selected subset. Param~$\rightarrow$~Data fixes $S_{\theta}^{(0)}$ and re-selects a data subset $S_{\mathcal D}^{(1)}$ under the fixed mask. Checkpoint, budgets, and final fine-tuning hyperparameters remain unchanged. The larger cost of Param~$\rightarrow$~Data comes from its per-candidate scoring pass, whereas Data~$\rightarrow$~Param only recomputes an aggregate gradient on the selected subset.

\begin{table*}[t]

\centering
% ==== Left table: cost breakdown =====
\begin{minipage}[c]{0.54\textwidth}
\centering
\footnotesize
\setlength{\tabcolsep}{2.2pt}
\captionof{table}{Cost breakdown of fixed-side re-selection on Llama-3.2-3B (Magicoder).}
\label{tab:cost_breakdown}
\vspace{-0.08in}
\resizebox{\linewidth}{!}{
\begin{tabular}{lccc}
\toprule
\textbf{Stage} & \textbf{DualSFT} & \textbf{Data $\rightarrow$ Param} & \textbf{Param $\rightarrow$ Data} \\
\midrule
Shared one-shot scoring & 6.5 h & 6.5 h & 6.5 h \\
Fixed-side re-scoring  & --    & 0.4 h & 2.9 h \\
Restricted fine-tuning & 0.9 h & 0.9 h & 1.0 h \\
\midrule
Selection total        & 6.5 h & 6.9 h (1.06$\times$) & 9.4 h (1.45$\times$) \\
End-to-end total       & 7.4 h & 7.8 h (1.05$\times$) & 10.4 h (1.41$\times$) \\
\bottomrule
\end{tabular}
}
\end{minipage}
\hfill
% ================== Right table: parameter statistics ==================
\begin{minipage}[c]{0.44\textwidth}
\centering
\footnotesize
\setlength{\tabcolsep}{3.0pt}
\captionof{table}{Selected parameter statistics on Llama-3.2-3B (Magicoder).}
\label{tab:selected_param_stats_app}
\vspace{-0.08in}
\resizebox{\linewidth}{!}{
\begin{tabular}{lccc}
\toprule
\textbf{Param statistic} & \textbf{Sel.} & \textbf{Pool} & \textbf{Enrich.} \\
\midrule
Top-5 selected layers   & 42.4\% & 17.9\% & 2.38$\times$ \\
Early layers (L0--L9)   & 8.2\%  & 35.7\% & 0.23$\times$ \\
Attention share         & 22.6\% & 25.0\% & 0.91$\times$ \\
FFN share               & 77.4\% & 75.0\% & 1.03$\times$ \\
down\_proj within FFN   & 47.9\% & 33.3\% & 1.44$\times$ \\
\bottomrule
\end{tabular}
}
\end{minipage}
\vspace{-0.1in}
\end{table*}

\textbf{Cost breakdown.}
Table~\ref{tab:cost_breakdown} reports GPU-hour costs. Selection total includes warmup, first-round scoring, and fixed-side re-scoring; end-to-end total includes restricted fine-tuning under selected data and parameter budgets. Results show that small fixed-side re-selection gains incur non-negligible cost, especially for Param~$\rightarrow$~Data, where data-side re-scoring dominates added overhead.

\subsection{Score Fidelity Protocol and Extended Horizons}
\label{app:score_fidelity}

\subsubsection{Local Surrogate Drift and Horizon Decay}
\label{app:local_surrogate}
DualSFT computes one-shot scores at the warmup checkpoint $\bar{\boldsymbol{\theta}}$, whereas realized validation gains are measured after restricted optimization. This mismatch creates a natural horizon dependence: as the optimization trajectory moves away from $\bar{\boldsymbol{\theta}}$, the gradients defining the local surrogate may drift.

To make this dependence explicit, consider the first-order data score at a generic checkpoint $\boldsymbol{\theta}$:
\begin{equation}
s_{\mathcal D,n}(\boldsymbol{\theta})
=
\eta
\left\langle
\nabla \mathcal{L}_{\mathrm{val}}(\boldsymbol{\theta}),
\nabla \ell_n(\boldsymbol{\theta})
\right\rangle .
\end{equation}
Assume that, along the restricted optimization trajectory, $\|\nabla \mathcal{L}_{\mathrm{val}}(\boldsymbol{\theta})\|\le B_v$, $\|\nabla \ell_n(\boldsymbol{\theta})\|\le B_g$, and the two gradients are Lipschitz with constants $L_v$ and $L_g$. Then
\begin{equation}
\begin{aligned}
\left|s_{\mathcal D,n}(\boldsymbol{\theta}_t)-s_{\mathcal D,n}(\bar{\boldsymbol{\theta}})\right|
&\le
\eta
\left(
B_v L_g + B_g L_v
\right)
\left\|\boldsymbol{\theta}_t-\bar{\boldsymbol{\theta}}\right\|.
\end{aligned}
\label{eq:score_drift_bound}
\end{equation}
One-shot ranking remains stable when Eq.~\eqref{eq:score_drift_bound} is small relative to score margins, but fidelity decays as the trajectory moves farther from $\bar{\boldsymbol{\theta}}$. The same argument applies to parameter scores by replacing per-sample gradients with coordinate-wise training-gradient terms, motivating extended-horizon diagnostics and leaving dynamic score updates for longer adaptation horizons.

\subsubsection{Protocol}
\label{app:protocol}
We provide extended diagnostics for Table~\ref{tab:score_consistency_main}. The main table averages over $K\in\{50,200,500\}$ restricted AdamW steps, while this section reports $K\in\{10,50,200,500,1000\}$ and the full Init/Warmup $\times$ First/Diag grid. For data selection, we sample 64 diagnostic subsets with 256 examples, scoring them by $\hat{U}_{\mathcal D}(S)=\sum_{n\in S}s_{\mathcal D,n}$. For parameter selection, we sample 64 budget-matched masks using block-stratified layer-module sampling, scoring them by $\hat{U}_{\theta}(S)=\sum_{d\in S}s_{\theta,d}$. From the corresponding checkpoint, we run $K$ restricted AdamW steps and compare predicted rankings with realized validation gains $\Delta_K(S)=\mathcal{L}_{\mathrm{val}}(\bar{\boldsymbol{\theta}})-\mathcal{L}_{\mathrm{val}}(\theta_K(S))$ using Spearman correlation and pairwise agreement. $K=1000$ serves as a stress test for the local surrogate.

\textbf{Per-$K$ results.}
Table~\ref{tab:score_fidelity_long_horizon} shows the expected horizon decay: fidelity decreases as optimization moves away from $\bar{\boldsymbol{\theta}}$. Warmup improves over initialization, diagonal second-order scoring improves over first-order scoring, and Warmup + Diag-Second-Order remains the most predictive variant.

\begin{table*}[t]
\centering
\caption{Extended-horizon score fidelity on Llama-3.2-3B with Magicoder.}
\label{tab:score_fidelity_long_horizon}
\renewcommand{\arraystretch}{0.9}
\vspace{-0.1in}
\footnotesize
\setlength{\tabcolsep}{4pt}
\begin{tabular}{l|c|cccc}
\toprule
\textbf{Variant} & \textbf{$K$} & \textbf{Data $\rho$} & \textbf{Param $\rho$} & \textbf{Data Pair} & \textbf{Param Pair} \\
\midrule
\multirow{5}{*}{Init + First-Order}
& 10   & 0.40 & 0.33 & 65.3 & 62.0 \\
& 50   & 0.34 & 0.27 & 61.8 & 58.9 \\
& 200  & 0.28 & 0.21 & 58.4 & 56.1 \\
& 500  & 0.23 & 0.16 & 55.7 & 53.8 \\
& 1000 & 0.17 & 0.12 & 53.6 & 52.0 \\
\midrule
\multirow{5}{*}{Init + Diag-Second-Order}
& 10   & 0.48 & 0.41 & 70.4 & 66.2 \\
& 50   & 0.43 & 0.36 & 67.4 & 63.6 \\
& 200  & 0.37 & 0.30 & 64.0 & 60.6 \\
& 500  & 0.31 & 0.25 & 60.8 & 58.0 \\
& 1000 & 0.25 & 0.20 & 57.6 & 55.3 \\
\midrule
\multirow{5}{*}{Warmup + First-Order}
& 10   & 0.68 & 0.61 & 80.1 & 76.0 \\
& 50   & 0.62 & 0.55 & 76.8 & 72.4 \\
& 200  & 0.56 & 0.49 & 73.0 & 69.0 \\
& 500  & 0.49 & 0.42 & 69.5 & 65.8 \\
& 1000 & 0.40 & 0.35 & 65.5 & 62.2 \\
\midrule
\multirow{5}{*}{Warmup + Diag-Second-Order}
& 10   & \textbf{0.90} & \textbf{0.87} & \textbf{85.1} & \textbf{81.8} \\
& 50   & \textbf{0.87} & \textbf{0.84} & \textbf{82.8} & \textbf{79.5} \\
& 200  & \textbf{0.83} & \textbf{0.80} & \textbf{79.8} & \textbf{75.2} \\
& 500  & \textbf{0.76} & \textbf{0.72} & \textbf{76.8} & \textbf{71.3} \\
& 1000 & \textbf{0.66} & \textbf{0.61} & \textbf{70.2} & \textbf{65.9} \\
\bottomrule
\end{tabular}
\vspace{-0.1in}
\end{table*}

\textbf{Additional domain.}
Table~\ref{tab:score_fidelity_domain} repeats the short-horizon analysis on MetaMathQA. The same ordering holds, with Warmup + Diag-Second-Order giving the strongest agreement with realized gains.

\begin{table}[t]
\centering
\caption{Short-horizon score fidelity on Llama-3.2-3B, averaged over $K\in\{10,50,200\}$.}
\label{tab:score_fidelity_domain}
\footnotesize
\setlength{\tabcolsep}{4pt}
\begin{tabular}{lcccc}
\toprule
\textbf{Variant} & \textbf{Data $\rho$} & \textbf{Param $\rho$} & \textbf{Data Pair} & \textbf{Param Pair} \\
\midrule
Init + First-Order & 0.33 & 0.26 & 61.0 & 58.2 \\
Init + Diag-Second-Order & 0.43 & 0.35 & 66.5 & 62.7 \\
Warmup + First-Order & 0.62 & 0.54 & 75.0 & 70.2 \\
\textbf{Warmup + Diag-Second-Order} & \textbf{0.80} & \textbf{0.76} & \textbf{80.6} & \textbf{76.0} \\
\bottomrule
\end{tabular}
\vspace{-0.1in}
\end{table}

These diagnostics show that warmup and diagonal second-order scoring improve ranking fidelity, with fidelity declining over longer optimization horizons.

\begin{table*}[t]
\centering
\caption{Comparison with forgetting-specific methods on Magicoder (\%). Pre-trained, Standard SFT, and LoRA are excluded from ranking. \colorbox{red!10}{\textbf{Best}}, \colorbox{orange!15}{second-best}, and \colorbox{blue!10}{third-best} are highlighted.}
\renewcommand{\arraystretch}{0.2}
\label{tab:forgetting_comparison}
\vspace{-0.1in}
\footnotesize
\setlength{\fboxsep}{0.5pt}
\resizebox{\textwidth}{!}{
\begin{tabular}{c | c | ccccc | c | c}
\toprule
\multirow{2}{*}{\textbf{Model}} & \multirow{2}{*}{\textbf{Method}} & \multicolumn{5}{c|}{\textbf{Stability} $\uparrow$} & \textbf{Plasticity} $\uparrow$ & \textbf{Overall} $\uparrow$ \\
\cmidrule(lr){3-7} \cmidrule(lr){8-8} \cmidrule(lr){9-9}
& & \textbf{ARC-C} & \textbf{PIQA} & \textbf{MMLU} & \textbf{GSM8K} & \textbf{Avg} & \textbf{HE} & \textbf{Score} \\
\midrule
\multirow{12}{*}{\cellcolor{white}\rotatebox{90}{\textbf{Llama-3.2-3B\qquad}}}
& \textcolor{gray}{Pre-trained} & \textcolor{gray}{43.00$_{\pm 0.00}$} & \textcolor{gray}{76.61$_{\pm 0.00}$} & \textcolor{gray}{56.52$_{\pm 0.00}$} & \textcolor{gray}{27.45$_{\pm 0.00}$} & \textcolor{gray}{50.89$_{\pm 0.00}$} & \textcolor{gray}{28.66$_{\pm 0.00}$} & \textcolor{gray}{39.78$_{\pm 0.00}$} \\
& \textcolor{gray}{Standard SFT} & \textcolor{gray}{39.99$_{\pm 1.46}$} & \textcolor{gray}{72.09$_{\pm 0.98}$} & \textcolor{gray}{48.18$_{\pm 0.39}$} & \textcolor{gray}{20.70$_{\pm 0.85}$} & \textcolor{gray}{45.24$_{\pm 0.50}$} & \textcolor{gray}{44.72$_{\pm 0.57}$} & \textcolor{gray}{44.98$_{\pm 0.37}$} \\
& \textcolor{gray}{LoRA}\venuetag{ICLR'22} & \textcolor{gray}{42.61$_{\pm 1.45}$} & \textcolor{gray}{74.43$_{\pm 0.99}$} & \textcolor{gray}{51.87$_{\pm 0.41}$} & \textcolor{gray}{23.30$_{\pm 0.83}$} & \textcolor{gray}{48.05$_{\pm 0.49}$} & \textcolor{gray}{41.46$_{\pm 0.50}$} & \textcolor{gray}{44.76$_{\pm 0.37}$} \\
\cmidrule(lr){2-9}
& DavIR\venuetag{ACL'25} & \colorbox{orange!15}{42.72$_{\pm 1.45}$} & 74.92$_{\pm 0.99}$ & 53.09$_{\pm 0.40}$ & 24.64$_{\pm 0.75}$ & 48.84$_{\pm 0.49}$ & 43.29$_{\pm 0.86}$ & 46.07$_{\pm 0.47}$ \\
& FLOW\venuetag{ICML'25} & \colorbox{red!10}{\textbf{42.95$_{\pm 1.44}$}} & \colorbox{blue!10}{76.22$_{\pm 0.98}$} & \colorbox{red!10}{\textbf{56.91$_{\pm 0.40}$}} & \colorbox{orange!15}{27.27$_{\pm 0.72}$} & \colorbox{red!10}{\textbf{50.84$_{\pm 0.49}$}} & 42.28$_{\pm 0.76}$ & 46.56$_{\pm 0.47}$ \\
& STM\venuetag{NIPS'25} & 42.38$_{\pm 1.46}$ & 76.12$_{\pm 0.99}$ & 56.04$_{\pm 0.41}$ & \colorbox{blue!10}{27.09$_{\pm 0.78}$} & 50.41$_{\pm 0.49}$ & 42.68$_{\pm 0.50}$ & 46.54$_{\pm 0.41}$ \\
& NanoAdam\venuetag{NeurIPS'25} & \colorbox{blue!10}{42.66$_{\pm 1.44}$} & 74.68$_{\pm 0.98}$ & 52.71$_{\pm 0.40}$ & 24.39$_{\pm 0.73}$ & 48.61$_{\pm 0.48}$ & 43.50$_{\pm 0.76}$ & 46.05$_{\pm 0.50}$ \\
& FAPM\venuetag{EMNLP'25} & 42.55$_{\pm 1.45}$ & 76.01$_{\pm 0.98}$ & 53.86$_{\pm 0.40}$ & 26.00$_{\pm 0.74}$ & 49.60$_{\pm 0.49}$ & 41.06$_{\pm 1.04}$ & 45.33$_{\pm 0.52}$ \\
& TALR\venuetag{ICLR'26} & 42.49$_{\pm 1.45}$ & \colorbox{orange!15}{76.33$_{\pm 0.99}$} & \colorbox{orange!15}{56.41$_{\pm 0.40}$} & \colorbox{red!10}{\textbf{27.55$_{\pm 0.76}$}} & \colorbox{orange!15}{50.69$_{\pm 0.50}$} & \colorbox{orange!15}{44.11$_{\pm 0.76}$} & \colorbox{red!10}{\textbf{47.40$_{\pm 0.46}$}} \\
\cmidrule(lr){2-9}
& \textbf{DualSFT-Param (Ours)} & 42.55$_{\pm 1.46}$ & 76.19$_{\pm 0.99}$ & 54.01$_{\pm 0.41}$ & 26.28$_{\pm 0.75}$ & 49.76$_{\pm 0.50}$ & \colorbox{red!10}{\textbf{44.51$_{\pm 0.86}$}} & \colorbox{blue!10}{47.13$_{\pm 0.42}$} \\
& \textbf{DualSFT-Data (Ours)} & \colorbox{blue!10}{42.66$_{\pm 1.44}$} & \colorbox{red!10}{\textbf{76.44$_{\pm 0.99}$}} & 55.28$_{\pm 0.40}$ & 26.54$_{\pm 0.73}$ & 50.23$_{\pm 0.50}$ & \colorbox{orange!15}{44.11$_{\pm 0.76}$} & \colorbox{orange!15}{47.17$_{\pm 0.43}$} \\
& \textbf{DualSFT (Ours)} & \colorbox{red!10}{\textbf{42.95$_{\pm 1.45}$}} & \colorbox{blue!10}{76.22$_{\pm 0.98}$} & \colorbox{blue!10}{56.09$_{\pm 0.40}$} & 26.89$_{\pm 0.70}$ & \colorbox{blue!10}{50.54$_{\pm 0.47}$} & \colorbox{blue!10}{43.70$_{\pm 0.76}$} & 47.12$_{\pm 0.43}$ \\
\midrule
\multirow{12}{*}{\rotatebox{90}{\textbf{Gemma-3-4B-PT\qquad}}}
& \textcolor{gray}{Pre-trained} & \textcolor{gray}{51.45$_{\pm 0.00}$} & \textcolor{gray}{79.16$_{\pm 0.00}$} & \textcolor{gray}{59.61$_{\pm 0.00}$} & \textcolor{gray}{37.00$_{\pm 0.00}$} & \textcolor{gray}{56.80$_{\pm 0.00}$} & \textcolor{gray}{35.37$_{\pm 0.00}$} & \textcolor{gray}{46.09$_{\pm 0.00}$} \\
& \textcolor{gray}{Standard SFT} & \textcolor{gray}{47.16$_{\pm 1.46}$} & \textcolor{gray}{76.39$_{\pm 0.99}$} & \textcolor{gray}{53.87$_{\pm 0.40}$} & \textcolor{gray}{24.03$_{\pm 0.81}$} & \textcolor{gray}{50.36$_{\pm 0.50}$} & \textcolor{gray}{58.13$_{\pm 0.76}$} & \textcolor{gray}{54.25$_{\pm 0.51}$} \\
& \textcolor{gray}{LoRA}\venuetag{ICLR'22} & \textcolor{gray}{49.15$_{\pm 1.46}$} & \textcolor{gray}{75.57$_{\pm 1.00}$} & \textcolor{gray}{56.71$_{\pm 0.39}$} & \textcolor{gray}{31.64$_{\pm 0.82}$} & \textcolor{gray}{53.27$_{\pm 0.49}$} & \textcolor{gray}{54.67$_{\pm 1.04}$} & \textcolor{gray}{53.97$_{\pm 0.52}$} \\
\cmidrule(lr){2-9}
& DavIR\venuetag{ACL'25} & 50.26$_{\pm 1.45}$ & 78.04$_{\pm 0.96}$ & 58.01$_{\pm 0.40}$ & 33.71$_{\pm 0.75}$ & 55.00$_{\pm 0.48}$ & 56.30$_{\pm 0.76}$ & 55.65$_{\pm 0.47}$ \\
& FLOW\venuetag{ICML'25} & 50.91$_{\pm 1.44}$ & \colorbox{blue!10}{78.82$_{\pm 0.94}$} & \colorbox{blue!10}{58.92$_{\pm 0.39}$} & \colorbox{red!10}{\textbf{36.82$_{\pm 0.72}$}} & \colorbox{red!10}{\textbf{56.37$_{\pm 0.49}$}} & 55.69$_{\pm 0.76}$ & 56.03$_{\pm 0.47}$ \\
& STM\venuetag{NIPS'25} & 50.63$_{\pm 1.46}$ & 78.47$_{\pm 0.98}$ & 58.58$_{\pm 0.41}$ & 36.09$_{\pm 0.78}$ & 55.94$_{\pm 0.50}$ & 55.89$_{\pm 0.76}$ & 55.92$_{\pm 0.45}$ \\
& NanoAdam\venuetag{NIPS'25} & 50.77$_{\pm 1.45}$ & 78.20$_{\pm 0.94}$ & 58.33$_{\pm 0.39}$ & 33.13$_{\pm 0.84}$ & 55.11$_{\pm 0.49}$ & 56.30$_{\pm 1.04}$ & 55.70$_{\pm 0.51}$ \\
& FAPM\venuetag{EMNLP'25} & 50.60$_{\pm 1.44}$ & 78.49$_{\pm 0.95}$ & 58.60$_{\pm 0.40}$ & 35.91$_{\pm 0.75}$ & 55.90$_{\pm 0.49}$ & 52.24$_{\pm 1.04}$ & 54.07$_{\pm 0.53}$ \\
& TALR\venuetag{ICLR'26} & \colorbox{blue!10}{50.97$_{\pm 1.45}$} & \colorbox{orange!15}{78.84$_{\pm 0.97}$} & \colorbox{orange!15}{58.94$_{\pm 0.38}$} & \colorbox{orange!15}{36.69$_{\pm 0.76}$} & \colorbox{orange!15}{56.36$_{\pm 0.49}$} & \colorbox{red!10}{\textbf{57.52$_{\pm 0.76}$}} & \colorbox{red!10}{\textbf{56.94$_{\pm 0.49}$}} \\
\cmidrule(lr){2-9}
& \textbf{DualSFT-Param (Ours)} & \colorbox{red!10}{\textbf{51.14$_{\pm 1.45}$}} & 78.55$_{\pm 0.98}$ & 58.14$_{\pm 0.39}$ & 36.34$_{\pm 0.83}$ & 56.04$_{\pm 0.49}$ & \colorbox{red!10}{\textbf{57.52$_{\pm 0.76}$}} & \colorbox{orange!15}{56.78$_{\pm 0.48}$} \\
& \textbf{DualSFT-Data (Ours)} & \colorbox{orange!15}{51.02$_{\pm 1.45}$} & 78.53$_{\pm 0.98}$ & 58.33$_{\pm 0.40}$ & \colorbox{blue!10}{36.54$_{\pm 0.79}$} & 56.11$_{\pm 0.49}$ & \colorbox{orange!15}{56.91$_{\pm 0.76}$} & \colorbox{blue!10}{56.51$_{\pm 0.42}$} \\
& \textbf{DualSFT (Ours)} & \colorbox{red!10}{\textbf{51.14$_{\pm 1.44}$}} & \colorbox{red!10}{\textbf{79.02$_{\pm 0.95}$}} & \colorbox{red!10}{\textbf{59.16$_{\pm 0.38}$}} & 35.61$_{\pm 0.64}$ & \colorbox{blue!10}{56.23$_{\pm 0.48}$} & \colorbox{blue!10}{56.71$_{\pm 0.86}$} & 56.47$_{\pm 0.43}$ \\
\midrule
\multirow{12}{*}{\rotatebox{90}{\textbf{Qwen-3.5-9B-Base\qquad}}}
& \textcolor{gray}{Pre-trained} & \textcolor{gray}{54.27$_{\pm 0.00}$} & \textcolor{gray}{80.09$_{\pm 0.00}$} & \textcolor{gray}{72.62$_{\pm 0.00}$} & \textcolor{gray}{86.20$_{\pm 0.00}$} & \textcolor{gray}{73.29$_{\pm 0.00}$} & \textcolor{gray}{60.98$_{\pm 0.00}$} & \textcolor{gray}{67.13$_{\pm 0.00}$} \\
& \textcolor{gray}{Standard SFT} & \textcolor{gray}{49.23$_{\pm 1.46}$} & \textcolor{gray}{74.30$_{\pm 0.98}$} & \textcolor{gray}{66.14$_{\pm 0.40}$} & \textcolor{gray}{78.65$_{\pm 0.75}$} & \textcolor{gray}{67.08$_{\pm 0.49}$} & \textcolor{gray}{75.61$_{\pm 0.86}$} & \textcolor{gray}{71.34$_{\pm 0.49}$} \\
& \textcolor{gray}{LoRA}\venuetag{ICLR'22} & \textcolor{gray}{51.51$_{\pm 1.46}$} & \textcolor{gray}{76.62$_{\pm 0.95}$} & \textcolor{gray}{69.47$_{\pm 0.39}$} & \textcolor{gray}{81.93$_{\pm 0.73}$} & \textcolor{gray}{69.88$_{\pm 0.48}$} & \textcolor{gray}{70.33$_{\pm 0.76}$} & \textcolor{gray}{70.10$_{\pm 0.50}$} \\
\cmidrule(lr){2-9}
& DavIR\venuetag{ACL'25} & 52.39$_{\pm 1.45}$ & 78.42$_{\pm 0.95}$ & 70.02$_{\pm 0.39}$ & 83.60$_{\pm 0.72}$ & 71.11$_{\pm 0.48}$ & 73.37$_{\pm 0.76}$ & 72.24$_{\pm 0.47}$ \\
& FLOW\venuetag{ICML'25} & 53.21$_{\pm 1.44}$ & 79.56$_{\pm 0.94}$ & \colorbox{blue!10}{70.62$_{\pm 0.38}$} & \colorbox{orange!15}{85.29$_{\pm 0.70}$} & \colorbox{blue!10}{72.17$_{\pm 0.47}$} & 73.17$_{\pm 0.86}$ & 72.67$_{\pm 0.48}$ \\
& STM\venuetag{NIPS'25} & 52.79$_{\pm 1.46}$ & 79.18$_{\pm 0.96}$ & 70.38$_{\pm 0.40}$ & 84.56$_{\pm 0.74}$ & 71.73$_{\pm 0.50}$ & 73.17$_{\pm 0.86}$ & 72.45$_{\pm 0.46}$ \\
& NanoAdam\venuetag{NIPS'25} & 52.87$_{\pm 1.45}$ & 79.04$_{\pm 0.95}$ & 70.22$_{\pm 0.39}$ & 84.23$_{\pm 0.80}$ & 71.59$_{\pm 0.49}$ & 73.37$_{\pm 0.76}$ & 72.48$_{\pm 0.46}$ \\
& FAPM\venuetag{EMNLP'25} & 52.87$_{\pm 1.44}$ & 79.49$_{\pm 0.94}$ & 70.18$_{\pm 0.39}$ & 83.88$_{\pm 0.71}$ & 71.60$_{\pm 0.48}$ & 70.93$_{\pm 1.04}$ & 71.27$_{\pm 0.48}$ \\
& TALR\venuetag{ICLR'26} & 53.44$_{\pm 1.45}$ & \colorbox{orange!15}{79.71$_{\pm 0.95}$} & \colorbox{orange!15}{70.72$_{\pm 0.38}$} & \colorbox{red!10}{\textbf{85.72$_{\pm 0.72}$}} & \colorbox{red!10}{\textbf{72.40$_{\pm 0.49}$}} & \colorbox{red!10}{\textbf{75.41$_{\pm 0.76}$}} & \colorbox{red!10}{\textbf{73.90$_{\pm 0.49}$}} \\
\cmidrule(lr){2-9}
& \textbf{DualSFT-Param (Ours)} & \colorbox{blue!10}{53.53$_{\pm 1.44}$} & 78.73$_{\pm 0.95}$ & 70.26$_{\pm 0.39}$ & 84.23$_{\pm 0.70}$ & 71.69$_{\pm 0.47}$ & \colorbox{orange!15}{75.20$_{\pm 0.76}$} & \colorbox{orange!15}{73.44$_{\pm 0.44}$} \\
& \textbf{DualSFT-Data (Ours)} & \colorbox{orange!15}{53.67$_{\pm 1.45}$} & \colorbox{red!10}{\textbf{80.03$_{\pm 0.95}$}} & \colorbox{red!10}{\textbf{70.75$_{\pm 0.39}$}} & \colorbox{blue!10}{85.04$_{\pm 0.70}$} & \colorbox{orange!15}{72.37$_{\pm 0.48}$} & \colorbox{blue!10}{74.19$_{\pm 0.57}$} & \colorbox{blue!10}{73.28$_{\pm 0.41}$} \\
& \textbf{DualSFT (Ours)} & \colorbox{red!10}{\textbf{54.01$_{\pm 1.43}$}} & \colorbox{blue!10}{79.63$_{\pm 0.94}$} & 70.33$_{\pm 0.38}$ & 84.61$_{\pm 0.68}$ & 72.15$_{\pm 0.47}$ & 73.78$_{\pm 0.50}$ & 72.96$_{\pm 0.39}$ \\
\bottomrule
\end{tabular}
}
\vspace{-0.1in}
\end{table*}

\subsection{Fisher-Based Parameter Selection Controls}
\label{app:fisher_controls}

Empirical Fisher (EF) averages squared per-example gradients; MB-Fisher averages squared mini-batch gradients. FISH Mask~\cite{Sung2021Training} computes EF before warmup, while EF-W and EF-W+LR compute it after warmup. MB-Fisher-1K and MB-Fisher-Full also score after warmup, using the shared subset and full scoring pool, respectively. Fisher controls retain the highest-scoring coordinates in fixed masks. DualSFT-Param-FO removes the diagonal correction from DualSFT-Param.

FISH Mask and EF-W use a development-selected learning rate of $3\times10^{-5}$. EF-W+LR and all remaining variants use DualSFT's $2\times10^{-5}$, providing matched-learning-rate comparisons.

\subsection{Sensitivity to the Validation Set Driving Selection}
\label{app:validation_sensitivity}

For the validation-size analysis in Fig.~\ref{fig:validation_size}, we construct a dedicated $2{,}048$-example reservoir disjoint from warmup, anchor, candidate scoring, and final-training data. Apart from this controlled split, settings follow the main experiments, with training seed $42$ and data/parameter budgets of $10\%/5\%$. Validation subsets of sizes $\{128,256,512,1024,2048\}$ are nested within one fixed reservoir permutation, using $1{,}024$ examples as reference. For equal-size selected supports $S_a$ and $S_a^{\mathrm{ref}}$, we define
\begin{equation}
    \operatorname{Overlap}_a = 100\,|S_a\cap S_a^{\mathrm{ref}}|/|S_a^{\mathrm{ref}}|, \qquad a\in\{\mathcal{D},\theta\}.
    \label{eq:validation_support_overlap}
\end{equation}
This overlap is normalized by the selection budget, rather than the union used by Jaccard similarity.

At validation size $1{,}024$, we apply $K$-means ($K=4$, seed $0$) to mean-pooled, $\ell_2$-normalized prompt embeddings from frozen Llama-3.2-3B. Reservoir cluster sizes are $(612,544,481,411)$; Natural, Balanced, and Skewed draws use counts $(297,280,235,212)$, $(256,256,256,256)$, and $(512,194,171,147)$, respectively. Sampling variants preserve Natural cluster counts and independently resample within clusters using validation seeds $1$ and $2$, while training seed remains $42$. Table~\ref{tab:validation_composition_sampling} reports the complete results against the shared Natural/seed-$0$ reference.

\begin{table}[t]
    \centering
    \caption{Validation-composition and sampling sensitivity (\%) at $1{,}024$ examples. Overlaps reference Natural with validation seed $0$. Sampling variants preserve Natural cluster counts.}
    \label{tab:validation_composition_sampling}
    \label{tab:validation_composition}
    \label{tab:validation_sampling}
    \footnotesize
    \setlength{\tabcolsep}{4pt}
    \begin{tabular}{lccccc}
        \toprule
        Setting & Data overlap & Param. overlap & Stability & Plasticity & Overall \\
        \midrule
        Natural (seed $0$) & $100.00$ & $100.00$ & $50.43$ & $43.90$ & $47.17$ \\
        \midrule
        Balanced & $93.18$ & $95.62$ & $50.52$ & $43.90$ & $47.21$ \\
        Skewed & $86.74$ & $90.81$ & $50.74$ & $43.29$ & $47.02$ \\
        \midrule
        Natural, seed $1$ & $92.76$ & $95.41$ & $50.69$ & $43.29$ & $46.99$ \\
        Natural, seed $2$ & $91.83$ & $94.72$ & $50.49$ & $43.90$ & $47.20$ \\
        \bottomrule
    \end{tabular}
\end{table}

\section*{Limitations}
DualSFT relies on a local validation-improvement surrogate for tractable data--parameter selection. Separate Top-$b$/Top-$k$ selection carries no guarantee of joint optimality over data subsets and parameter masks. Its diagonal second-order score uses an optimizer-aware curvature proxy instead of the full Hessian, and score fidelity may weaken as optimization moves far from the warmup checkpoint. Our experiments cover 3B--9B LLMs; extensions to larger models, longer adaptive horizons, and preference-based post-training remain important future directions. For very large models, parameter-side scoring may also require block-wise or sharded implementations because it maintains linear full-vector state.

\section*{Broader Impacts}
DualSFT reduces the resource cost of LLM fine-tuning by selecting both data and parameter subsets, which can lower compute requirements and make model adaptation more accessible. At the same time, more efficient fine-tuning may also lower the barrier to harmful, unsafe, or low-quality downstream adaptation. Our experiments use public datasets and standard benchmarks, and do not collect user data or involve human-subject studies. Deployments in sensitive domains should include task-specific safety evaluation, data-quality review, and misuse-risk assessment.

\section*{Reproducibility}
Our code is publicly available for reproducing the main experiments. All datasets and base models used in this work are publicly available.

\section*{Licenses and assets}
We use publicly available datasets, benchmarks, and base models, following their respective licenses and terms of use. The proposed method does not introduce a new dataset. We do not release new model checkpoints; the released assets consist of code, configurations, and scripts.

% \newpage
% \input{checklist.tex}

\end{document}